\documentclass[11pt,letterpaper]{article}
\usepackage[margin=1in]{geometry}
\usepackage{setspace}
\newcommand{\keywords}[1]{\par\medskip\noindent{\small\textit{Keywords:} #1\par}\medskip}
\makeatletter
\renewcommand{\@maketitle}{%
  \begin{center}
    {\LARGE\bfseries \@title\par}
    \vskip 1.25em
    {\normalsize\begin{tabular}[t]{c}\@author\end{tabular}\par}
  \end{center}
  \vskip 1em
}
\makeatother
\usepackage[utf8]{inputenc}
\usepackage[T1]{fontenc}
\usepackage{lmodern}
\AtBeginDocument{\DeclareFontShape{T1}{lmr}{bx}{sc}{<->ssub*lmr/m/sc}{}}

\usepackage{textcomp,graphicx,booktabs,makecell,arydshln,adjustbox,threeparttable,multirow,diagbox,tabularx,array,siunitx,xcolor}
\usepackage[normalem]{ulem}
\usepackage{amsmath,amsfonts,amssymb,amsthm,mathtools,bm,mathrsfs,latexsym,nicefrac}
\usepackage[nopatch={footnote}]{microtype}
\usepackage{algorithm,algorithmicx,enumerate,enumitem,float,rotating,rotfloat,appendix,fix-cm,comment}
\usepackage[font=footnotesize,labelfont=bf]{caption}
\usepackage[round,authoryear]{natbib}
\usepackage{url}
\usepackage[hidelinks]{hyperref}
\usepackage{doi}
\hypersetup{pdftitle={Diagonalized Attention for Individualized Regression: Latent-Row Localization and Prediction},pdfauthor={Borui Peng, Liwei Lin, Feifei Wang, Long Feng}}

\date{}
\DeclareFontFamily{U}{rsfs}{\skewchar\font127}
\DeclareFontShape{U}{rsfs}{m}{n}{<-6> rsfs5 <6-8> rsfs7 <8-> rsfs10}{}

\newcommand{\bel}{\begin{eqnarray}\label}
\newcommand{\eel}{\end{eqnarray}}
\newcommand{\bes}{\begin{eqnarray*}}
\newcommand{\ees}{\end{eqnarray*}}
\newcommand{\bei}{\begin{itemize}}
\newcommand{\beiftnt}{\begin{itemize}\footnotesize}
\newcommand{\eei}{\end{itemize}}

\def\benu{\begin{enumerate}}
\def\eenu{\end{enumerate}}

\DeclareMathOperator*{\argmax}{arg\,max}

\def\R{{\real}}

\def\E{{\mathbb{E}}}

\def\complex{\mathop{{\rm I}\kern-.58em\hbox{\rm C}}\nolimits}

\def\diag{\hbox{\rm diag}}

\def\mathbold{\boldsymbol} 

\def\bA{\mathbold{A}}

\def\bB{\mathbold{B}}

\def\bC{\mathbold{C}}

\def\bD{\mathbold{D}}

\def\bH{\mathbold{H}}

\def\bI{\mathbold{I}}

\def\bK{\mathbold{K}}

\def\bL{\mathbold{L}}

\def\bM{\mathbold{M}}

\def\calN{{\cal N}}

\def\bP{\mathbold{P}}
\def\calP{{\cal P}}

\def\bQ{\mathbold{Q}}

\def\bR{\mathbold{R}}

\def\calS{{\cal S}}

\def\bT{\mathbold{T}}

\def\bu{\mathbold{u}}

\def\bU{\mathbold{U}}

\def\bv{\mathbold{v}}

\def\bV{\mathbold{V}}\def\hbV{{\widehat{\bV}}}

\def\bW{\mathbold{W}}\def\hbW{{\widehat{\bW}}}

\def\bx{\mathbold{x}}

\def\bX{\mathbold{X}}

\def\bz{\mathbold{z}}

\def\bZ{\mathbold{Z}}

\def\balpha{\mathbold{\alpha}}

\def\bbeta{\mathbold{\beta}}

\def\bgamma{\mathbold{\gamma}}

\def\bGamma{\mathbold{\Gamma}}

\def\bDelta{\mathbold{\Delta}}

\def\eps{\epsilon}

\def\btheta{\mathbold{\theta}}

\def\hbtheta{{\widehat{\btheta}}}

\def\bTheta{\mathbold{\Theta}}

\def\bkappa{\mathbold{\kappa}}

\def\bmu{\mathbold{\mu}}

\def\bSigma{\mathbold{\Sigma}}

\def\bpsi{\mathbold{\psi}}

\def\bOmega{\mathbold{\Omega}}

\def\E{\mathbb{E}}

\def\R{\mathbb{R}}

\def\T{\top}

\let\hat\widehat
\let\tilde\widetilde

\useunder{\uline}{\ul}{}
\newcolumntype{Y}{>{\centering\arraybackslash}X}

\newtheorem{proposition}{Proposition}

\newcommand{\Prob}{\mathbb{P}}
\newcommand{\softmax}{\operatorname{softmax}}

\def\bA{\bm{A}}
\def\bB{\bm{B}}
\def\bC{\bm{C}}
\def\bI{\bm{I}}
\def\bP{\bm{P}}
\def\bU{\bm{U}}
\def\bV{\bm{V}}
\def\bW{\bm{W}}
\def\bX{\bm{X}}
\def\balpha{\bm{\alpha}}
\def\bbeta{\bm{\beta}}
\def\bmu{\bm{\mu}}

\def\bSigma{\bm{\Sigma}}
\def\bbE{\mathbb{E}}
\def\bbR{\mathbb{R}}

\def\calN{\mathcal{N}}
\def\calS{\mathcal{S}}

\theoremstyle{plain}
\newtheorem{theorem}{Theorem}[section]
\newtheorem{lemma}[theorem]{Lemma}

\newtheorem{example}{Example}

\useunder{\uline}{\ul}{}
\theoremstyle{remark}
\newtheorem{remark}{Remark}[section]

\theoremstyle{definition}

\newtheorem{assumption}[theorem]{Assumption}

\title{Diagonalized Attention for Individualized Regression: Latent-Row Localization and Prediction}
\author{Borui Peng$^{1,\dagger}$ \quad Liwei Lin$^{2,\dagger}$ \quad Feifei Wang$^{2}$ \quad Long Feng$^{1,\ast}$\\[6pt]
\normalfont\small $^{1}$School of Computing \& Data Science, University of Hong Kong\\
\normalfont\small $^{2}$School of Statistics, Renmin University of China\\[4pt]
\normalfont\small $^{\dagger}$The first two authors contribute equally.\\
\normalfont\small $^{\ast}$Correspondence to: \href{mailto:lfeng@hku.hk}{lfeng@hku.hk}}
\begin{document}
\maketitle
\begin{abstract}
Modern text and image representations are often matrix-valued, with rows corresponding to tokens, patches, or other local feature vectors. Predictive information is often sparse but sample-specific, making classical sparse regression methods with a common support poorly suited to this heterogeneity. This paper formalizes an individualized sparse regression framework for matrix-valued covariates in which each observation has its own rows of interest, while the associated regression effects are shared across the population. To estimate this model, we introduce a diagonalized attention mechanism that uses query--key scores to localize sample-specific signal rows and a value matrix for downstream regression. The proposed method has a parameter dimension independent of sample size and can identify rows of interest for new observations without their responses. We establish existence theorems showing that, under suitable score-separation and concentration conditions, single-head and multi-head diagonalized attention models recover the latent rows with high probability, yielding prediction risk bounds. Our theory therefore provides a statistical explanation of how attention-based scoring localizes sample-specific signals in heterogeneous matrix-valued data. Simulations demonstrate strong prediction and localization in regression and misspecified classification across varying sample sizes, dimensions, and signal cardinalities. Real sentiment analyses show improved classification accuracy and interpretable token selection.

\end{abstract}
\section{Introduction}\label{sec:intro}
Modern large language and vision models have transformed the way unstructured data are represented and analyzed. Text, images, and other complex objects are now routinely encoded as collections of feature vectors rather than as a single fixed vector of covariates. This shift creates new statistical challenges. 
The relevant information in such data is often sparse, but its location is highly sample-dependent: two observations may share the same type of predictive signal 
while expressing it through different tokens, patches, or local features. Consequently, a single population-level coefficient vector or a fixed sparsity pattern can be too rigid for these applications. This motivates an individualized sparse regression framework in which each sample is allowed to have its own regions of interest, while the regression effects associated with such regions are learned across samples.

The difficulty is especially visible in language applications. After tokenization and embedding, a sentence or comment is naturally represented by a matrix, with each row corresponding to the feature vector of one token. The number of rows varies with sentence length, and the tokens that determine the response may appear at different positions in different samples. In sentiment analysis, for instance, a positive or negative label may be driven by only a few sentiment-bearing words, but these words do not occur in aligned locations across comments. Similar issues arise in image analysis, where an image can be partitioned into local patches and the informative object or texture may occupy only a small, sample-specific part of the frame. Classical sparse regression methods, such as lasso-type or group-sparse estimators, are designed to select a common set of coordinates or rows across observations. They are therefore not well suited to matrix-valued data with varying row dimensions and unknown, sample-specific signal regions.

There is a growing statistical literature on individualized modeling. Early representative work includes varying coefficient models, where regression coefficients are modeled as functions of an auxiliary variable \citep{hastie1993varying,fan2003adaptive}. Another line of work assumes that similar samples should have similar regression coefficients and uses pairwise regularization induced by predefined similarity structures \citep{xu2015formula,yamada2017localized,lengerich2018personalized}. These approaches provide important tools for heterogeneity modeling, but they typically require an observed indexing variable, a neighborhood structure, or another form of external prior information. More recent subgroup analysis methods improve estimation by encouraging samples to share coefficients within latent groups. For example, \citet{tang2021individualized} proposed a multidirectional separation penalty for individualized variable selection and subgroup identification, and \citet{he2023center} developed a center-augmented $\ell_2$-type regularization framework with improved statistical and computational properties. However, these subgroup-based methods still require outcomes from new samples for individualized estimation and inference, making them unsuitable for standard prediction settings where only test covariates are observed. Related individualized modeling frameworks have also been studied from other perspectives; see, for example, \citet{tang2020individualized,miao2025reinforcement,rim2025individualized,zhang2026individualized}.

In parallel, attention mechanisms, as a core component of the Transformer architecture, have achieved remarkable empirical success in both language and vision models \citep{vaswani2017attention,devlin2019bert,dosovitskiy2020image}. In a standard self-attention layer, the query and key matrices define data-adaptive scores among input tokens or patches, while the value matrix determines the features that are aggregated for downstream prediction. This architecture is naturally attractive for heterogeneous data because it does not require the important row or patch to appear at a fixed location. Despite its success, the statistical mechanisms through which attention can localize sample-specific signals remain only partially understood. Recent theoretical work has begun to address this issue: \citet{marion2025attention} studied a simplified nonlinear self-attention layer for individualized single-location regression; \citet{wang2024transformers} and \citet{barnfield2026high} investigated the ability of attention to extract sparse informative words; and \citet{chen2024transformers} and \citet{zhang2025transformer} established guarantees for group-sparsity selection. In addition, \citet{yang2024attention} connected an individualized image regression framework with attention mechanisms. These studies suggest that attention can serve not only as a neural-network module, but also as a statistical device for individualized signal localization.

Motivated by these observations, this paper formalizes an individualized sparse regression model for matrix-valued covariates in which the signal rows are unknown and sample-specific. We propose a diagonalized attention mechanism that modifies the standard query-key scoring step by retaining the diagonal scores associated with individual rows. The resulting attention weights act as continuous estimators of sample-specific row indicators, while the value matrix plays the role of shared regression coefficients for the localized rows. This construction separates localization from regression: the query--key scores identify the rows of interest (ROIs), while the value component relates the localized features to the response. 
Additionally, our approach uses a number of parameters independent of the sample size and enables individualized prediction and ROI identification for new observations based solely on their covariates.

Our theoretical analysis proceeds from single-head to multi-head diagonalized attention. For the single-head case, corresponding to one signal row per sample, we establish the existence result constructively by showing that an appropriately chosen query--key matrix produces attention weights that concentrate on the true ROI with high probability. We then extend the construction to the multi-head setting, where each head is designed to locate one of multiple signal rows. Building on these localization guarantees, we derive an excess prediction risk bound under an individualized linear regression model. Our theory clarifies how diagonalized query--key scores exploit the score gap between active and inactive rows, and how the softmax operation translates this gap into accurate localization. It thus provides a statistical explanation of how attention-based scoring can identify sample-specific signals in heterogeneous matrix-valued data.

We further evaluate the proposed method through simulations and real data analysis. In Gaussian regression settings, diagonalized attention achieves strong prediction accuracy and ROI localization, even under comparison settings that favor competing methods by allowing them to select more candidate rows. In a Gaussian-mixture classification setting, the proposed method remains competitive despite model misspecification. Finally, in real-world sentiment classification datasets, diagonalized attention identifies sentiment-bearing tokens that are interpretable and consistent with the observed labels, while also improving classification performance over standard attention and other individualized regression competitors.

The rest of the paper is organized as follows. Section~\ref{section: methodology} introduces the individualized regression framework and motivating applications in language and image data. Section~\ref{section: single-head attn and diag} reviews the standard attention mechanism and introduces the proposed diagonalized attention mechanism. Section~\ref{section: theory} establishes localization guarantees and derives a prediction risk bound. Sections~\ref{section: simulation} and~\ref{section: real} present simulation studies and real data analysis, respectively.

\textbf{Notation.} 
We use bold uppercase letters for matrices and bold
lowercase letters for vectors. For a matrix $\bA$, $\bA_{j,\cdot}$ and
$\bA_{\cdot,k}$ denote its $j$-th row and $k$-th column, respectively;
$\lambda_{\min}(\bA)$ and
$\lambda_{\max}(\bA)$ denote its smallest and largest eigenvalues if it is symmetric; $\|\bA\|_F=\sqrt{\sum_{i,j} A_{i,j}^2}$, $\|\bA\|_{\mathrm{op}}=\sup_{\|\bx\|_2=1}\|\bA\bx\|_2=\sqrt{\lambda_{\max}(\bA^\top\bA)}$, and $\|\bA\|_\infty=\max_{j,k}|A_{j,k}|$ denote its Frobenius, operator, and maximum-entry norms, respectively; and
$\bA^\dagger$ denotes its  Moore--Penrose pseudo-inverse.  For
symmetric matrices $\bA$ and $\bB$ of the same dimension, $\bA\succ\bB$ and
$\bA\succeq\bB$ mean that $\bA-\bB$ is positive definite and positive
semi-definite, respectively. For nonnegative quantities \(a\) and \(b\), \(a\lesssim b\) means \(a\leq Cb\) for some constant \(C>0\), and \(a\gtrsim b\) means \(b\lesssim a\). The notation \(O_{\Prob}(\cdot)\) has its usual stochastic meaning, while \(o(\cdot)\) has its usual deterministic asymptotic meaning. Finally, for any positive integer \(n\), we write \([n]=\{1,\ldots,n\}\).

\section{Individualized regression}\label{section: methodology}
We begin by presenting the general formulation of individualized regression in Section \ref{section:general}. The motivating applications in language models and image analysis are presented in  Section \ref{section:application}.

\subsection{The general formulation of individualized regression} \label {section:general}
Consider $n$ observations, each consisting of a matrix-valued covariate $\bX^{(i)}\in\bbR^{p^{(i)}\times d}$ and a scalar response $y^{(i)}$, for $i\in[n]$. The covariate matrix $\bX^{(i)}$ contains $p^{(i)}$ feature vectors of dimension $d$, denoted as $\bX^{(i)}
=
\left[
\left(\bX_{1,\cdot}^{(i)}\right)^\top,
\ldots,
\left(\bX_{p^{(i)},\cdot}^{(i)}\right)^\top
\right]^\top$, where each row $\bX_{j,\cdot}^{(i)}\in\bbR^{1\times d}$ corresponds to the $j$-th feature vector for sample $i$ and $p^{(i)}$ is allowed to vary across samples. We assume that the response $y^{(i)}$ follows a generalized linear model with link function $g(\cdot)$:
\begin{equation}
	\label{model1}
g\left(\bbE(y^{(i)})\right) = \langle\bX^{(i)}, \bC^{(i)} \rangle,
\end{equation}
where $\bC^{(i)}\in\bbR^{p^{(i)}\times d}$ denotes sample-specific coefficients that model individualized effects.
The notation $\langle \bA, \bB \rangle = \mathrm{tr}(\bA^\top \bB)$ denotes inner product.  
Model~(\ref{model1}) can be easily extended to include additional covariates  $\bz^{(i)}$ whose coefficients are shared across the population. In the present work, however, we focus exclusively on the matrix-valued input $\bX^{(i)}$ and omit $\bz^{(i)}$ for clarity.

Sparse matrix regression has been extensively studied in the literature \citep{yuan2007dimension,bunea2011optimal,negahban2011estimation,chen2012sparse}. 
By imposing various penalty functions, one can identify important entries, rows, or columns in the input matrix.  
In this work, we also impose a sparsity assumption on the coefficients. Unlike classical
sparse regression models, however, we allow the signal locations to vary across
samples. Specifically, we assume that each sample-specific coefficient matrix
$\bC^{(i)}$ contains $s$ signal rows, with $s\leq \min\{p^{(1)},\ldots,p^{(n)}\}$, and takes the following form:
\begin{equation} \label{eq:structure_decomp_R}
\bC^{(i)} = {\bA^{(i)}}\bB^\top, 
\end{equation}
where $\bB=[\bbeta_1,\ldots,\bbeta_s]\in\bbR^{d\times s}$ and $\bA^{(i)}=[\balpha_1^{(i)},\ldots,\balpha_s^{(i)}]$ are unknown and satisfy
\begin{equation}
 \bA^{(i)} = [\bm{\alpha}_1^{(i)}, \ldots, \bm{\alpha}_s^{(i)}] \in \{0, 1\}^{p^{(i)} \times s}, \ \ 
\left(\bA^{(i)}\right)^\T \bA^{(i)} =\bI_{s}. 
\end{equation}
With this decomposition, the binary matrix $\bA^{(i)}$ serves as an individualized binary selection matrix that identifies the active rows for subject~$i$, and $\bB$ represents the regression coefficient matrix associated with the active rows. Denote the active rows for subject~$i$ as $\calS^{(i)}$, i.e.,
\begin{equation}
    \calS^{(i)}=\left\{j:\ \bA_{j,\cdot}^{(i)}\neq \bm{0}\right\}.
\end{equation}
We refer to $\calS^{(i)}$ as the Rows of Interest (ROI) for subject $i$ in subsequent analysis. This is closely related to the problem of detecting regions of interest in the existing literature; see Section \ref{section:application} below for further application details. 
Clearly, when $p^{(i)}$ and $\calS^{(i)}$ are fixed for all samples, this setting reduces to the classical row-sparse matrix-valued regression. The constraint $\left(\bA^{(i)}\right)^\T \bA^{(i)} =\bI_{s}$ serves as an identifiability condition, and it also implies  
$\text{Card}(\calS^{(i)})=s$ for $i\in[n]$. 
We note that Models~(\ref{model1}) and~(\ref{eq:structure_decomp_R}) can be rewritten as 
\begin{equation}\label{eq:individualized_framework}
 g \left( \mathbb{E}(y^{(i)}) \right)  
 = \left\langle \big(\bX^{(i)}\big)^\T \bA^{(i)}, \ \bB \right\rangle.
\end{equation}
This formulation further clarifies that $\bA^{(i)}$ identifies the ROI positions for sample~$i$, while $\bB$ contains the corresponding regression coefficients.
Although here we focus on identifying important rows, our framework can be readily
adapted to detect important columns or regions; see Example~2 in
Section~\ref{section:application} for more details.

Under the individualized sparse regression framework, it is essential to clearly define the active rows for each observation. Without such a definition, any row could be regarded as active, making ROI identification infeasible. To ensure feasible ROI identification, we impose the following distributional assumption on the input. Specifically, for all the observations $i\in[n]$, assume
\begin{equation}\label{eq:dist_assumption}
\boldsymbol{X}_{j,\cdot}^{(i)}
\mid
A_{j,k}^{(i)}=1
\ \sim \ 
f_k(\boldsymbol{x}),\quad  j=1,\ldots, p^{(i)},\quad k\in[s]
\end{equation}
holds for certain unknown functions $f_k(\cdot)$.
Here, the common distributional assumption is imposed only on the $s$ active rows of each observation. The distributions of the inactive rows are allowed to vary across samples. Moreover, we do not impose any independence assumption among the rows within a sample; in fact, both active and inactive rows may be mutually dependent. This setting is therefore highly challenging, and classical sparse regression methods, such as lasso-type or group-sparse estimators, are not directly applicable.
Additionally, under the parameterization in (\ref{eq:individualized_framework}), 
treating each $\bA^{(i)}$ as an unconstrained matrix causes the total number of parameters to grow linearly with $n$. This is infeasible in practice and undermines statistical consistency. In the next section, inspired by the standard attention mechanism, we show how a diagonalized attention structure can be used to address this individualized regression problem.

\subsection{Applications of individualized regression}\label{section:application}

The individualized regression framework discussed in Section \ref{section: methodology} is widely applicable. In this section, we illustrate two representative examples in language models and image analysis.

\begin{example}\label{eg:language}
(Sentiment Analysis in Language Models) Consider a sentiment analysis problem that classifies comments as positive or negative. A common strategy is to use a pretrained language model to tokenize each comment into a sequence of tokens and represent each token as a feature vector in a \(d\)-dimensional continuous space. Each comment is then transformed into an embedding matrix by stacking the feature vectors in the original token order. 
Since the sentiment polarity of a comment is typically driven by only a few sentiment-related tokens, this naturally motivates a sparse regression formulation.
However, natural language complexity gives rise to heterogeneity across comments in two ways. First, comment lengths vary, leading to different numbers of tokens. Second, sentiment-bearing tokens may appear at arbitrary positions and are not aligned across comments.
Thus, order-preserving representations yield input matrices with different numbers of rows and informative tokens in unpredictable locations. Classical statistical models often struggle with such individualized tasks, whereas our individualized regression framework~\eqref{eq:individualized_framework} naturally accommodates these structures and identifies informative sentiment-related tokens.

Let $\bX^{(i)} \in \R^{p^{(i)} \times d}$ denote the embedding matrix for comment $i$, where $p^{(i)}$ is the number of tokens in comment $i$ and $d$ is the embedding dimension of the feature vector associated with each token. 
To account for different comment lengths and unknown locations of informative tokens, we consider the following individualized sparse logistic regression model:
$\mathrm{logit}\bigl(\E (y^{(i)})\bigr) = \langle \bX^{(i)}, \bC^{(i)} \rangle$,
where the coefficient matrix $\bC^{(i)}\in \R^{p^{(i)} \times d}$ is assumed to admit the structural decomposition in \eqref{eq:structure_decomp_R}.
We will study this sentiment analysis problem in Section \ref{section: real} and show that our method can successfully identify informative tokens without any additional preprocessing.
\end{example}

\begin{example}
(Sparse Image Regression)
     Consider an image classification task in which the objective is to distinguish between two classes: Cat and Dog. The input consists of corresponding images for each class. In many practical scenarios, the species-discriminative features occupy only a small region within a larger image, thereby creating a sparse signal setting. Unlike applications with well-aligned input vectors across samples, such as genomic data analysis, images are often unaligned, with the animal’s location varying substantially across frames. Classification accuracy can therefore be improved by focusing on the small patch containing the relevant object, a setting naturally modeled by our individualized regression framework for identifying subject-specific regions of interest.
     
Let image \(i\) have dimension \(D_1^{(i)}\times D_2^{(i)}\) and be partitioned into \(p_1^{(i)}\times p_2^{(i)}\) patches of size \(d_1\times d_2\), where \(D_1^{(i)}=p_1^{(i)}d_1,D_2^{(i)}=p_2^{(i)}d_2\). When only a small number of patches contain class-discriminative information, we consider an individualized sparse logistic regression model $\mathrm{logit}(\E y^{(i)})=\langle \tilde{\bX}^{(i)},\tilde{\bC}^{(i)}\rangle$ in which the coefficients $\tilde{\bC}^{(i)}$ are specified as follows:
\begin{equation}\label{skpd}
    \tilde{\bC}^{(i)} = \sum_{k=1}^s \bU_k^{(i)}\otimes \bV_k.
\end{equation}
Here, $\otimes$ is the Kronecker product, $\bU_k^{(i)}\in\{0,1\}^{p_1^{(i)}\times p_2^{(i)}}$ denotes the patch-indicator matrix for image $i$, indicating the locations of signal-containing patches, and $\bV_k \in \mathbb{R}^{d_1\times d_2}$ represents the corresponding coefficient matrix. This patch-wise localization problem can be reformulated within our models~(\ref{model1}) and~(\ref{eq:structure_decomp_R}) via the linear mapping $\mathcal{R}: \mathbb{R}^{D_1^{(i)} \times D_2^{(i)}}\to \mathbb{R}^{p_1^{(i)} p_2^{(i)}\times d_1 d_2}$,
\[
\mathcal{R}\bigl(\widetilde{\bC}^{(i)}\bigr)
=
\left[
\operatorname{vec}\bigl(\widetilde{\bC}_{1,1}^{(i)}\bigr),
\ldots,
\operatorname{vec}\bigl(\widetilde{\bC}_{p_1^{(i)},1}^{(i)}\bigr),
\ldots,
\operatorname{vec}\bigl(\widetilde{\bC}_{1,p_2^{(i)}}^{(i)}\bigr),
\ldots,
\operatorname{vec}\bigl(\widetilde{\bC}_{p_1^{(i)},p_2^{(i)}}^{(i)}\bigr)
\right]^\top .
\]
where $\tilde{\bC}^{(i)}_{j,l}$ denotes the $(j,l)$-th block of $\tilde{\bC}^{(i)}$ and  $\operatorname{vec}(\cdot)$ denotes column-wise
vectorization.
With this mapping, the logit model and (\ref{skpd}) can be written as  $\mathrm{logit}(\E y^{(i)})=\langle \bX^{(i)},\bC^{(i)}\rangle$, where
\[
\bX^{(i)} = \mathcal{R}\big( \tilde{\bX}^{(i)} \big), 
\quad
\bC^{(i)} = \mathcal{R}\big( \tilde{\bC}^{(i)} \big), 
\quad
\balpha_k^{(i)} = \textnormal{vec}\big( \bU_k^{(i)} \big), 
\quad 
\bbeta_k = \textnormal{vec}\big( \bV_k \big).
\]

Under this linear mapping, we can reframe the task of detecting {\bf Regions of Interest}  in $ \tilde{\bX}^{(i)}$ to our problem of detecting {\bf Rows of Interest} in $\bX^{(i)}$. We refer to \citet{wu2023sparse} for a detailed discussion of the Kronecker product decomposition in image data analysis.
\end{example}

\section{Standard attention and diagonalized attention}\label{section: single-head attn and diag}

\subsection{Standard attention}
\label{sec:method_attn}


We first review the standard single-head self-attention mechanism. Given
learnable matrices $\bW_Q$, $\bW_K$, and $\bW_V$, it maps the input features to
the representation
\begin{equation}\label{eq:standard_attn_full}
\mathcal{A}(\bX^{(i)})
=
\mathcal{A}(\bX^{(i)};\bW_Q,\bW_K,\bW_V)
=
\softmax\left(
\frac{(\bX^{(i)}\bW_Q)(\bX^{(i)}\bW_K)^\top}{\sqrt{d}}
\right)
\bX^{(i)}\bW_V .
\end{equation}
Here, $\bW_Q$, $\bW_K$, and $\bW_V$ are referred to as the query, key, and value matrices, respectively. 
For simplicity, we assume that these matrices are square with dimension $d \times d$, although they may also be chosen to be non-square matrices of dimension $d \times d'$, where $d' \neq d$. 
The softmax operator can be applied either row-wise or entry-wise, depending on the context. 
When applied row-wise to a matrix \(\bm{M} \in \mathbb{R}^{p \times p}\), its \((j,k)\)-th entry is defined by
\begin{equation*}
\left[\softmax(\bm{M})\right]_{j,k}
=
\frac{\exp(M_{j,k})}{\sum_{l=1}^p \exp(M_{j,l})}.
\end{equation*}
Under this definition, each row of the resulting attention matrix is nonnegative and sums to one. 
Consequently, each output row may be interpreted as a data-adaptive combination or aggregation of the value vectors \(\bX^{(i)}\bW_V\), where the aggregation weights are determined by learned pairwise similarities among the row features induced by the projection matrices \(\bW_Q\) and \(\bW_K\).

Now, let’s take a closer look at the matrix inside the softmax operation. 
Denote $\mathbf{\Psi}^{(i)}= d^{-1/2}(\bX^{(i)}\bW_Q)(\bX^{(i)}\bW_K)^\top$. Then its $(j,k)$-th entry is $\Psi_{j,k}^{(i)} = d^{-1/2}\langle \bX^{(i)}_{j,\cdot}\bW_Q,\bX^{(i)}_{k,\cdot}\bW_K\rangle$, which measures the compatibility between the $j$-th and $k$-th rows after projection into the query and key spaces, respectively. 
Hence, $\mathbf{\Psi}^{(i)}$ can be interpreted as an uncentered cross-Gram (similarity) matrix induced by the two linear projections. 
The softmax is then applied row-wise to $\mathbf{\Psi}^{(i)}$, so each row is normalized into a probability distribution over keys. The scaling factor $1/\sqrt{d}$ controls the variance of dot-product scores and helps prevent softmax weights from becoming overly concentrated when $d$ is large.
In fact, the core of attention lies in this query--key scoring step, which defines a similarity or compatibility function between feature representations. 
More broadly, for vectors in $\mathbb{R}^d$, one may define a general scoring function $\psi:\mathbb{R}^d\times\mathbb{R}^d\to\mathbb{R}$ to quantify alignment or dependence. 
Different choices of \(\psi\) yield different attention mechanisms, including additive attention based on multilayer perceptrons \citep{bahdanau2014neural} and kernelized or linearized variants \citep{katharopoulos2020transformers}.

For the final prediction, the attention layer is typically followed by additional layers, such as fully connected layers. We consider the simplest setting with a single fully connected layer, whose output weight matrix is denoted by $\bW_O$. 
Then, when the outcomes are continuous scalars, the final prediction takes the form
\[
\mathcal{O}(\bX^{(i)})=\mathcal{O}(\bX^{(i)};\bW_Q,\bW_K,\bW_V,\bW_O)
=
\left\langle
\mathcal{A}(\bX^{(i)}),\bW_O
\right\rangle .
\]
If we let $\bW=\bW_Q\bW_K^\top$ and $\bV=\bW_O\bW_V^\top$, the final output can be equivalently written as
\begin{equation}\label{single1}
\mathcal{O}(\bX^{(i)};\bW,\bV)
=
\left\langle
\softmax\left(
\frac{\bX^{(i)}\bW(\bX^{(i)})^\top}{\sqrt{d}}
\right)
\bX^{(i)},
\bV
\right\rangle .
\end{equation}
When considering classification problems, an additional link function may be incorporated; we omit the details here.

Extending the single-head operator defined above, multi-head attention concatenates the outputs from all heads. Specifically, consider an $R$-head attention and let
$\mathcal{A}(\bX^{(i)};\bOmega_r)=\mathcal{A}(\bX^{(i)};\bW_{Q,r},\bW_{K,r},\bW_{V,r})$ for $r\in[R]$. Then we define
\begin{equation}\label{eq:standard_multi_head}
\mathcal{MA}(\bX^{(i)};\bOmega)
=
\big[ \mathcal{A}(\bX^{(i)};\bOmega_1), \dots, 
\mathcal{A}(\bX^{(i)};\bOmega_R) \big].
\end{equation}
Here, $\bOmega=\{\bOmega_r\}_{r=1}^R$ collects all learnable parameters, and $\mathcal{MA}$ denotes multi-head attention, which produces an $R$-head attention representation. As in the single-head mechanism, we consider a simple setting with one fully connected layer for each head, and denote the corresponding output weight matrix by $\bW_{O,r}$ for head $r$. Further, let $\bW_r=\bW_{Q,r}\bW_{K,r}^\T$ and $\bV_r=\bW_{O,r}\bW_{V,r}^\top$. Then the final output of the multi-head attention mechanism takes the form
\begin{equation}\label{multi1}
\mathcal{O}(\bX^{(i)};\{\bW_r,\bV_r\}_{r=1}^R)
=\sum_{r=1}^R
\left\langle
\softmax\left(
\frac{\bX^{(i)}\bW_r(\bX^{(i)})^\top}{\sqrt{d}}
\right)
\bX^{(i)},
\bV_r
\right\rangle .
\end{equation}

\subsection{Diagonalized attention for individualized regression}
\label{sec:method_diag_attn}
For the individualized sparse regression setting discussed in Section~\ref{section:general}, 
we propose a diagonalized version of the attention mechanism. 
We start with the case $s=1$, where each observation contains exactly one signal row, and then generalize to the multi-head attention for general $s$.

Recall that the signal row varies across observations, making its identification a crucial step. In this paper, we propose to accomplish this using a diagonalized attention mechanism. First, recall that
$\bW = \bW_Q \bW_K^\top$,
and the attention score matrix \(\mathbf{\Psi}^{(i)}\) can be written as $\mathbf{\Psi}^{(i)} = d^{-1/2}\bX^{(i)} \bW (\bX^{(i)})^\top$.
Then we consider the diagonal vector of the attention matrix, i.e.,
\begin{equation}\label{eq:diag_score}
	\bm{\psi}^{(i)} = \diag(\mathbf{\Psi}^{(i)})
	= d^{-1/2}
	\left[
	\bX_{1,\cdot}^{(i)} \bW (\bX_{1,\cdot}^{(i)})^\top, 
	\dots, 
	\bX_{p^{(i)},\cdot}^{(i)} \bW (\bX_{p^{(i)},\cdot}^{(i)})^\top
	\right]^\top
	\in \mathbb{R}^{p^{(i)}}.
\end{equation}
Note that \(\bm{\psi}^{(i)}\) is a function of the unknown matrix \(\bW\) and the observation \(\bX^{(i)}\). To emphasize this dependence, we may also write
$\bm{\psi}^{(i)} = \bm{\psi}(\bX^{(i)};\bW)$.
If we further apply the softmax transformation to \(\bm{\psi}^{(i)}\), we obtain
\begin{equation}\label{eq:softmax_diag_score}
	\bkappa^{(i)} = \bkappa(\bX^{(i)};\bW) = \softmax(\bm{\psi}^{(i)}) \in \bbR^{p^{(i)}},
\end{equation}
which defines a probability distribution over the \(p^{(i)}\) entries.
Finally, we propose
\begin{equation}\label{eq:final_estimator}
	z^{(i)} = z(\bX^{(i)};\bW,\bv)
	= \left\langle \bv,  (\bX^{(i)})^\top \bkappa^{(i)} \right\rangle
\end{equation}
as our final estimate for \(g(\E(y^{(i)}))\) in the generalized linear model~\eqref{model1}. 
Here, $\bv$ is an unknown coefficient vector that plays the role of 
$\bV=\bW_O\bW_V^\T$ in the standard attention mechanism in \eqref{single1}. 
Consequently, $z^{(i)}$ is a function of $\bX^{(i)}$ and the unknown 
parameters $\bW \in \bbR^{d\times d}$ and $\bv \in \bbR^d$. 

The key feature of this construction is the separation of localization and regression. 
The weight vector \(\bkappa^{(i)}=\bkappa(\bX^{(i)};\bW)\) is designed to locate the signal row of observation \(i\) and thus serves as an estimator of the latent location vector \(\balpha^{(i)}\).  In particular, the proposed diagonalized scoring design in (\ref{eq:diag_score})
 enables effective localization of the signal row under the model in 
(\ref{eq:dist_assumption}). 
Conditional on this localization step, the parameter \(\bv\) plays the role of the regression coefficient \(\bbeta\) by linking the aggregated feature \((\bX^{(i)})^\top \bkappa^{(i)}\) to the response. 
Note that the softmax transformation in \eqref{eq:softmax_diag_score} is crucial, since it amplifies relative differences among the diagonal scores and encourages the resulting weights to place most of their mass on the informative row. 
This concentration property is central to our theoretical analysis, and will be established formally in Section~\ref{section: theory}.

Now consider the case in which each sample contains \(s\) active rows,
where \(s \ge 2\). Under this setting, an \(R\)-head diagonalized attention mechanism with $R=s$ is needed, in which each attention head is responsible for
detecting a distinct feature row.
Following the diagonalized scoring design in \eqref{eq:diag_score}, let
\(\{\boldsymbol{\kappa}_r^{(i)}\}_{r=1}^R\) denote the diagonal attention weight vectors
of the \(i\)-th sample across \(R\) heads after the softmax transformation. Specifically, for the \(r\)-th head, we define
\begin{equation}\label{eq:softmax_diag_score_multi}
\boldsymbol{\kappa}_r^{(i)}
= \boldsymbol{\kappa}\bigl(\bX^{(i)};\bW_r\bigr)
= \operatorname{softmax} \left(
\boldsymbol{\psi}\bigl(\bX^{(i)};\bW_r\bigr)
\right)
\in \mathbb{R}^{p^{(i)}} .
\end{equation}
Here, \(\bW_r=\bW_{Q,r}\bW_{K,r}^\T\) denotes the weight
matrix associated with the \(r\)-th head. Since the softmax transformation
normalizes each score vector, \(\boldsymbol{\kappa}_r^{(i)}\) naturally
represents a probability distribution over the \(p^{(i)}\) entries. 
We then
concatenate the \(R\) probability distributions and define the aggregated
attention matrix as
\begin{equation}\label{eq:K_matrix_estimation}
\bK^{(i)}
= \bK\bigl(\bX^{(i)};\{\bW_r\}_{r=1}^{R}\bigr)
= \left[
\boldsymbol{\kappa}_1^{(i)},\ldots,\boldsymbol{\kappa}_R^{(i)}
\right]
\in \mathbb{R}^{p^{(i)} \times R}.
\end{equation}
As in the
single-head attention mechanism, the matrix
\(
\bK^{(i)}
\)
is designed to identify the \(s\) signal rows of the \(i\)-th observation and
serves as a continuous estimator of the latent indicator matrix
\(\bA^{(i)}\) in \eqref{eq:individualized_framework}.
The final output is obtained by taking the inner product between the
value matrix \(\bV=[\bv_1,\ldots,\bv_R]\in\bbR^{d\times R}\) and the
attention-pooled feature matrix \((\bX^{(i)})^\top\bK^{(i)}\). Specifically,
\begin{equation}\label{eq:final_estimator_multi}
z^{(i)}
=
z\bigl(\bX^{(i)};\{\bW_r,\bv_r\}_{r=1}^{R}\bigr)
=
\left\langle
\bV,
(\bX^{(i)})^\top \bK^{(i)}
\right\rangle
=
\sum_{r=1}^{R}
\left\langle
\bv_r,
(\bX^{(i)})^\top \bkappa_r^{(i)}
\right\rangle .
\end{equation}
Here, the parameters $\bv_r$ act as regression coefficients for the attention-pooled feature map $(\bX^{(i)})^\top \bkappa_r^{(i)}$. The resulting $R$ features and their corresponding coefficients are then combined to produce the final prediction $z^{(i)}$.

With the formulation of diagonalized attention, we consider the following problem
\begin{equation}\label{obj}
\min_{(\bV, \{\bW_r\}_{r=1}^R)\in \bTheta} \mathcal{L}_n(\bV, \{\bW_r\}_{r=1}^R) = \frac{1}{n} \sum_{i=1}^n \ell\left(y^{(i)}, z(\bX^{(i)};\{\bW_r,\bv_r\}_{r=1}^{R})\right),
\end{equation}
where \(\ell(\cdot, \cdot)\) is the canonical loss associated with the chosen GLM family, such as squared error for Gaussian responses or cross-entropy for Bernoulli responses, and \(\bTheta\) denotes the parameter space. In the absence of prior information, we may take
\(
\bTheta=\left\{\left(\bV,\{\bW_r\}_{r=1}^R\right):\bV\in\mathbb{R}^{d\times R},\ \bW_r\in\mathbb{R}^{d\times d}\right\}.
\) 
When structural preferences are available, \(\bTheta\) may instead be restricted to a suitable constrained space, such as a bounded parameter set.
We solve the optimization problem in~\eqref{obj} using a standard gradient-based algorithm to obtain the estimate $\hbtheta=(\hbV,\{\hbW_r\}_{r=1}^R)$, although other optimization methods may also be applied. 
Given estimated $\{\hbW_r\}_{r=1}^R$, the active rows for observation $i$ can be identified by finding the maximum entry in each of the $R$ columns of the attention matrix, i.e., 
\begin{equation*}
    \hat{\calS}^{(i)} = \left\{\hat{j}_r: \quad r\in[R],\quad \ \hat{j}_r=\argmax_{j\in[p^{(i)}]}\hat{\bK}_{j,r}^{(i)},\quad  \hat{\bK}^{(i)}=\bK\left(\bX^{(i)};\{\hat{\bW}_r\}_{r=1}^R\right)\right\}.
\end{equation*}
Ideally, the number of heads $R$ should be set equal to the number of active rows $s$. However, since $s$ is typically unknown in practice, $R$ can be treated as a tuning hyperparameter and selected using a validation-set approach or cross-validation. See supplementary material for further details.

\section{Diagonalized attention for individualized regression: latent-row localization and prediction}\label{section: theory}

In this section, we show that the proposed diagonalized attention mechanism effectively approximates individualized regression. We first consider single-head attention with one important row per observation and then extend the analysis to multi-head attention. Prediction accuracy is analyzed in Section~\ref{prediction}. Computational issues are beyond the scope of this paper; our theory focuses on constructive localization and global risk analysis.

\subsection{Single-head localization under score separation}\label{sec4-1}

As discussed before, the key feature of our design is that the weight vector
$\bkappa^{(i)}$ identifies the signal row. For observation $i$, let $j_{i}^{\star}$ denote the active row satisfying $\alpha_{j_{i}^{\star}}^{(i)}=1$. A row with $\alpha_{j}^{(i)}=0$ is called inactive.
Our strategy of identifying $j_i^{\star}$ relies on the existence of a matrix $\bW_0$ that separates the active row from the inactive rows, as formalized in the following assumption.

\begin{assumption}[Score separability] \label{assum:score-gap}
In the single-active-row setting, assume that there exists a matrix $\bW_0$ and a scalar $\delta_0^{(i)}$ such that, for each observation $i$,
\begin{equation}\label{assum:score-gap-eq}
\delta_0^{(i)} := \E\!\left(
\bX_{j_i^{\star},\cdot}^{(i)}
\bW_0
(\bX_{j_i^{\star},\cdot}^{(i)})^\top
\right)
-
\max_{\substack{
1\le j\le p^{(i)}, \ j\neq j_i^{\star}
}}
\E\!\left(
\bX_{j,\cdot}^{(i)}
\bW_0
(\bX_{j,\cdot}^{(i)})^\top
\right)
> 0.
\end{equation}
Further let $\delta_0 := \min_{i\in[n]}\delta_0^{(i)}$.
\end{assumption}

Assumption~\ref{assum:score-gap} is the central identifiability condition. 
It requires the existence of a matrix $\bW_0$ such that the active row has a larger expected score than the strongest inactive competitor. The inactive rows need not share a common distribution or a common expected score; their heterogeneity is summarized by the maximum in~\eqref{assum:score-gap-eq}. This is a population-level condition and does not require Gaussianity or any other specific distributional assumption. Only the relevant first and second moments are used to define the expected quadratic scores. The next theorem gives a concrete sufficient construction. It is deliberately stated under common active and inactive first two moments so that the role of covariance separation can be seen explicitly.

\begin{theorem}
[A covariance-based sufficient construction]\label{prop1}
Suppose that, across all observations, all active rows have mean $\bmu_1$ and covariance matrix $\bSigma_1$, while all inactive rows have mean $\bmu_0$ and covariance matrix $\bSigma_0$. 
Let $\bDelta_{\Sigma}=\bSigma_1-\bSigma_0$ and suppose that $\bDelta_{\Sigma}\succ0$. Define
\(
\bZ=
\bDelta_{\Sigma}^{-1/2}\left[
\bmu_1,\
\bmu_0
\right]
\)
and 
\(
\boldsymbol{P}_{\boldsymbol{Z}}^{\perp}
=
\boldsymbol{I}_d
-
\boldsymbol{Z}
\boldsymbol{Z}^{\dagger}
\), and $\operatorname{rank}(\bZ)\ll d$.
For any parameter $\tau>0$, if we let
\begin{equation}\label{W0}
\bW_0
=
\frac{1}{\tau}\bDelta_{\Sigma}^{-1/2}
\bP_{\bZ}^{\perp}
\bDelta_{\Sigma}^{-1/2},
\end{equation}
then Assumption \ref{assum:score-gap} is satisfied with 
$\delta_0
=
\frac{
d-\operatorname{rank}(\bZ)
}{\tau}
>0.$
\end{theorem}

\begin{remark}
The common active and inactive first two moments assumption in Theorem \ref{prop1} is used only to provide a transparent closed-form construction. With heterogeneous inactive rows, Assumption~\ref{assum:score-gap} remains the appropriate general condition: one must control the largest inactive expected quadratic score. Note that a covariance envelope by itself is not sufficient when inactive means differ, as the mean term also contributes to the score. One may instead impose a uniform condition on the full second-moment matrices, or verify the score gap directly. If the relevant covariance contrast is positive definite only on an informative subspace, the same construction can be restricted to that subspace, provided that the projected subspace remains nontrivial. See Section~\ref{app:sec:covariance-constructions} of the supplementary material for more details. 
\end{remark}

The population condition alone does not guarantee sample-level localization. We therefore impose the following uniform concentration condition on the diagonal scores.

\begin{assumption}[The concentration condition] \label{assum:concentration}
Assume that for every $\eta\in(0,1)$ there exists a deterministic radius $r_{\eta,p_{\max}}>0$ where
$p_{\max}=\max_{1\le i\le n}p^{(i)}$ such that
\begin{equation}\label{assum:score-concentration-eq}
\Prob\left(
\max_{1\le j\le p^{(i)}}
\left|\bX_{j,\cdot}^{(i)}\bW_0(\bX_{j,\cdot}^{(i)})^\top-\E(\bX_{j,\cdot}^{(i)}\bW_0(\bX_{j,\cdot}^{(i)})^\top) \right|
\le r_{\eta,p_{\max}}
\right)
\ge 1-\eta .
\end{equation}
\end{assumption}

This assumption controls the stochastic fluctuations of all row-wise scores simultaneously. Gaussian or sub-Gaussian row distributions provide convenient sufficient conditions through standard quadratic-form concentration inequalities, but the localization theorem itself only requires the concentration statement in Assumption~\ref{assum:concentration}.
Combining the population score gap with this concentration condition yields the following guarantee.

\begin{theorem}[Localization guarantee, $s=1$]
\label{thm:selection_single_s1}
For each \(i\in[n]\), define the attention score
\(
\bm{\psi}_0^{(i)}
=
\left[
\bX_{1,\cdot}^{(i)}\bW_0(\bX_{1,\cdot}^{(i)})^\top,\ldots,
\bX_{p^{(i)},\cdot}^{(i)}\bW_0(\bX_{p^{(i)},\cdot}^{(i)})^\top
\right]^\top
\)
and
\(
\bkappa_0^{(i)}
=
\softmax(\bm{\psi}_0^{(i)}).
\)
Suppose that Assumptions~\ref{assum:score-gap} and~\ref{assum:concentration} hold with constants $\delta_0$ and $r_{\eta,p_{\max}}$, respectively. For any $\varepsilon, \eta\in(0,1)$, if 
\begin{equation}\label{eq:score-gap-condition-single}
\delta_0
\ge\max \left\{4r_{\eta,p_{\max}},
2\log\left(\frac{p_{\max}-1}{\varepsilon}\right)
\right\}
\end{equation}
then for each observation $i$,
\begin{equation}\label{eq:bound_kappa_single_s1}
\Prob\left(
\bigl\|\bkappa_0^{(i)}-\balpha^{(i)}\bigr\|_{\infty}
\le
\varepsilon
\right)
\ge
1-\eta .
\end{equation}
\end{theorem}

Theorem~\ref{thm:selection_single_s1} makes the mechanism explicit: the query-key matrix creates a quadratic-score gap, concentration transfers this gap to the observed scores, and softmax normalization turns it into high-probability localization. The logarithmic term reflects the number of inactive competitors, while \(r_{\eta,p_{\max}}\) captures the stochastic difficulty of distinguishing their scores.

\begin{remark}[A sufficient condition under sub-Gaussian rows]
For the covariance-based construction in Theorem~\ref{prop1}, standard
quadratic-form concentration shows that, under sub-Gaussian row distributions
with uniformly bounded sub-Gaussian norms, the fluctuation requirement in
Eq.~\eqref{eq:score-gap-condition-single} is satisfied whenever
$d-\operatorname{rank}(\boldsymbol{Z})
\gtrsim
\frac{
\lambda_{\max}^{4}(\boldsymbol{\Sigma}_1)
}{
\lambda_{\min}^{2}(\boldsymbol{\Delta}_{\Sigma})
}
\log \left(
\frac{p_{\max}}{\eta}
\right)$.
Under the informative-subspace
construction described above, $\boldsymbol{W}_0$, or equivalently
$\boldsymbol{W}_Q$ and $\boldsymbol{W}_K$, can be chosen to be low-rank,
reducing model complexity while preserving the positive quadratic-score gap
on the informative subspace.
See Sections~\ref{app:subsec:single-construction} and \ref{app:subsec:informative-subspaces} of the supplementary material.
\end{remark}

\begin{remark}[The temperature parameter $\tau$]
In standard attention, \(1/\sqrt{d}\) in~\eqref{eq:standard_attn_full} controls the scale of dot-product scores. In diagonalized construction, this role is played by the temperature \(\tau\), which rescales the diagonal scores before softmax. A smaller \(\tau\) sharpens the weights and helps satisfy the logarithmic threshold in Theorem~\ref{thm:selection_single_s1}, but does not improve the ratio between the population gap and stochastic fluctuation. Specifically, the theorem requires \( \delta_0 \propto \tau^{-1}\), together with the corresponding concentration condition. Although the argument imposes no theoretical lower bound on \(\tau\), overly small values may cause numerical instability~\citep{zhai2023stabilizing}.
\end{remark}

\subsection{Multi-head localization under class-wise score separation}
\label{sec4-2}

We next consider the case in which each observation contains $s\geq 2$ active rows. We set the number of attention heads to $R=s$, with head $k$ intended to identify the row belonging to the $k$-th signal class. For each $i\in[n]$ and $k\in [s]$, let $j_{i,k}^{\star}$ denote the unique row satisfying $A_{j_{i,k}^{\star},k}^{(i)}=1$. A row with $\bA_{j,\cdot}^{(i)}=\boldsymbol{0}$ is called inactive for observation $i$.

The multi-head identification problem differs from the single-head problem because a given head must distinguish its target class not only from inactive rows, but also from the rows belonging to the other signal classes. 
We formalize the required separation through the following class-wise score-gap condition.

\begin{assumption}[Class-wise score separability]\label{assum:multi-score-gap}
Suppose that there exist matrices
$\{\bW_{0,k}\}_{k=1}^{s}$ and positive constants $\{\delta_k\}_{k=1}^{s}$ such that, for every observation $i$ and every head $k$,

\begin{equation}\label{assum:multi-score-gap-eq}
\begin{aligned}
\delta_k^{(0,i)}
&:=
\E\!\left\{
\bX_{j_{i,k}^{\star},\cdot}^{(i)}
\bW_{0,k}
\bigl(\bX_{j_{i,k}^{\star},\cdot}^{(i)}\bigr)^\top
\right\}-
\max_{\substack{1\leq j\leq p^{(i)}\\
j\notin\{j_{i,1}^{\star},\ldots, j_{i,s}^{\star}\}
}}
\E\!\left\{
\bX_{j,\cdot}^{(i)}\bW_{0,k}
\bigl(\bX_{j,\cdot}^{(i)}\bigr)^\top
\middle| \bA_{j,\cdot}^{(i)}=\boldsymbol{0}
\right\}
>0.\\
\delta_k^{(1,i)}
&:=
\E\!\left\{
\bX_{j_{i,k}^{\star},\cdot}^{(i)}
\bW_{0,k}
\bigl(\bX_{j_{i,k}^{\star},\cdot}^{(i)}\bigr)^\top
\right\}-
\max_{\ell\in\{1,\ldots, s\}\setminus\{k\}}
\E\!\left\{
\bX_{j_{i,\ell}^{\star},\cdot}^{(i)}\bW_{0,k}
\bigl(\bX_{j_{i,\ell}^{\star},\cdot}^{(i)}\bigr)^\top
\right\}
>0.
\end{aligned}
\end{equation}
Let $\delta_k^{(0)} := \min_{i\in[n]}\delta_k^{(0,i)}$ and $\delta_k^{(1)} := \min_{i\in[n]}\delta_k^{(1,i)}$, \(\delta_k:=\min\{\delta_k^{(0)}, \delta_k^{(1)}\},
\) and \(\delta_{\min}:=\min_{ k\in [s]}\delta_k.
\)
\end{assumption}

Assumption~\ref{assum:multi-score-gap} is the multi-head analogue of the single-head score-gap condition. 
For each $k\in [s]$, the matrix $\boldsymbol{W}_{0,k}$ ensures that the row from the \(k\)-th signal class has an expected score at least \(\delta_k^{(0)}\) larger than that of the strongest inactive competitor and at least \(\delta_k^{(1)}\) larger than that of every other signal class.
It allows the inactive rows to have heterogeneous distributions: only the strongest inactive competitor enters the gap. 
The following theorem gives a concrete sufficient construction under common first and second moments for the inactive rows and for each signal class. It also illustrates why multi-head localization requires a subspace in which the target class is more prominent than the remaining signal classes.

\begin{theorem}[A covariance-based sufficient multi-head construction]\label{prop2}
Suppose that, across all observations, all inactive rows have mean $\bmu_0$ and covariance matrix $\bSigma_0$, while rows in signal class $k$ have mean $\bmu_k$ and covariance matrix $\bSigma_k$, for $k\in[s]$. For each $k$, suppose that
\(
\bDelta_k=\bSigma_k-\bSigma_0\succ0
\) and let 
\(
\bZ_k=\bDelta_k^{-1/2}
[\bmu_0,\bmu_1,\ldots,\bmu_s],
\)
\(
\bP_{\bZ_k}^{\perp}
=\bI_d-\bZ_k\bZ_k^\dagger,
\)
where
$\operatorname{rank}(\bZ_k)\le s+1 \ll d$. Define
\(
d_k=d-\operatorname{rank}(\bZ_k).
\)
For $k\in[s]$, let
\[
\bD_k=\sum_{\ell\in\{1,\ldots,s\}\setminus\{k\}}\bDelta_\ell,
\qquad
\widetilde{\bD}_k
=
\bP_{\bZ_k}^{\perp}
\bDelta_k^{-1/2}\bD_k\bDelta_k^{-1/2}
\bP_{\bZ_k}^{\perp}.
\]
Let $d^\ast=\lfloor d/s\rfloor$ and assume that $d^\ast\leq d_k$ for every $k$. Denote the positive eigenvalues of $\widetilde{\bD}_k$ by
\(
\gamma_{k,1}\geq\gamma_{k,2}\geq\ldots\geq\gamma_{k,d_k}>0,
\)
and define the average of its $d^\ast$ smallest eigenvalues by
\(
\bar\gamma_k^\ast
=
\frac{1}{d^\ast}
\sum_{q=d_k-d^\ast+1}^{d_k}\gamma_{k,q},
\)
\(
\bar\gamma_{\max}^\ast=\max_{1\leq k\leq s}\bar\gamma_k^\ast.
\)
Suppose that $\bar\gamma_{\max}^\ast<1$. Let
$\bGamma_k\in\mathbb R^{d\times d^\ast}$ contain the corresponding orthonormal eigenvectors and, for any temperature parameter $\tau>0$, set
\begin{equation}\label{eq:multi-W0}
\bW_{0,k}
=
\frac{1}{\tau}
\bDelta_k^{-1/2}
\bGamma_k\bGamma_k^\top
\bDelta_k^{-1/2}.
\end{equation}
Then, the target class in head $k$ has an exact expected score gap
\(
\delta_{k}^{(0)}
=
\frac{d^\ast}{\tau}
\)
over the inactive rows. Moreover, for every competing signal class $\ell\neq k$, its expected score gap satisfies
\(
\delta_{k}^{(1)}
\geq
\frac{d^\ast(1-\bar\gamma_k^\ast)}{\tau}.
\)
\end{theorem}

\begin{remark}[Interpretation and extensions of Theorem~\ref{prop2}]
Theorem~\ref{prop2} constructs each head-specific query--key matrix by selecting a $d^\ast$-dimensional subspace in which interference from the other signal classes is weak. The condition $\bar\gamma_{\max}^\ast<1$ guarantees that, in this subspace, every target class retains a positive expected score advantage over every competing signal class; see Section~\ref{app:sec:overlap-discussion} of the supplementary material for further discussion.
As in the single-head case, heterogeneous inactive rows can be handled through the general score-gap condition in Assumption~\ref{assum:multi-score-gap}. Extensions to covariance contrasts that are positive definite only on informative subspaces are analogous; see Section~\ref{app:subsec:informative-subspaces} of the supplementary material for more details.
The theorem is a sufficient construction, rather than a necessary characterization, of multi-head separability. Its role is to make the source of the class-wise gap explicit; the theorem below requires only Assumption~\ref{assum:multi-score-gap}. 
\end{remark}

If we further assume Assumption \ref{assum:concentration} holds uniformly over each head $k \in [s]$ with a deterministic radius $r_{\eta,p_{\max},s}>0$,
we have the following theorem on multi-head attention.



\begin{theorem}[Localization guarantee, $s>1$]
\label{thm:selection_multi_s1}
For each observation $i$ and each head $k$, define
\(
\boldsymbol{\psi}_{k,0}^{(i)}
=
\left[
\bX_{1,\cdot}^{(i)}\bW_{0,k}
\bigl(\bX_{1,\cdot}^{(i)}\bigr)^\top,\ldots,
\bX_{p^{(i)},\cdot}^{(i)}\bW_{0,k}
\bigl(\bX_{p^{(i)},\cdot}^{(i)}\bigr)^\top
\right]^\top\)
and
\(
\boldsymbol{\kappa}_{k,0}^{(i)}
=
\operatorname{softmax}\!\left(\boldsymbol{\psi}_{k,0}^{(i)}\right).
\)
Let
\( 
\bK_0^{(i)}
=
\left[
\boldsymbol{\kappa}_{1,0}^{(i)},\ldots,
\boldsymbol{\kappa}_{s,0}^{(i)}
\right].
\)
Assume that Assumption \ref{assum:concentration} holds uniformly over each head $1\le k \le s$ with a deterministic radius $r_{\eta,p_{\max},s}>0$.
Suppose that Assumption~\ref{assum:multi-score-gap}  holds with 
\begin{equation}\label{eq:score-gap-condition-multi}
\delta_{\min}
\geq
\max\left\{
4r_{\eta,p_{\max},s},
2\log\left(\frac{p_{\max}-1}{\varepsilon}\right)
\right\}.
\end{equation}
Then, for every observation $i$,
\begin{equation}\label{eq:bound-K-multi}
\Prob\!\left(
\left\|
\bK_0^{(i)}-\bA^{(i)}
\right\|_{\infty}
\leq\varepsilon
\right)
\geq 1-\eta.
\end{equation}
\end{theorem}

Theorem~\ref{thm:selection_multi_s1} has the same structure as the single-head result. For each head, the target class must beat its strongest competitor by a gap large enough to dominate both score fluctuations and the softmax threshold. The joint concentration condition allows the guarantees for all heads to hold simultaneously, yielding recovery of the full latent indicator matrix.

\begin{remark}[A concrete sufficient condition under sub-Gaussian rows]
Under the construction in Theorem~\ref{prop2}, assuming sub-Gaussian row distributions with uniformly bounded norms, standard quadratic-form concentration gives, up to constants depending on these norms,
\[
r_{\eta,p_{\max},s}
\lesssim
\frac{1}{\tau}
\frac{
\max_{0\leq \ell\leq s}\lambda_{\max}(\bSigma_\ell)
}{
\min_{1\leq k\leq s}\lambda_{\min}(\bDelta_k)
}
\sqrt{
\frac{d}{s}\log\left(\frac{s p_{\max}}{\eta}\right)
}.
\]
This bound, together with the lower bound on $\delta_{\min}$ in Theorem~\ref{prop2}, gives a direct sufficient dimensional condition for multi-head high-probability localization. See Section~\ref{app:subsec:multi-construction} of the supplementary material for more details.
\end{remark}

\subsection{Prediction risk analysis}\label{prediction}

We briefly recall how the preceding localization result translates into
prediction risk. Consider the individualized linear model
in~\eqref{eq:individualized_framework},
\(
y^{(i)}
=
\langle \bX^{(i)},\bC^{(i)}\rangle+\xi^{(i)}\) with \( 
\xi^{(i)}\sim\mathcal N(0,\sigma^2), 
\)
and estimate it by solving the diagonalized-attention problem~\eqref{obj} with \(R=s\). 
Let
\(
\btheta=\left(\bV,\{\bW_k\}_{k=1}^s\right)
\)
collect the trainable parameters, whose total dimension is \(D=sd+sd^2\). Recall that \(z(\bX^{(i)})\), defined
in~\eqref{eq:final_estimator_multi}, is the diagonalized-attention prediction.
The empirical and population risks are
\(
\mathcal L_n(\btheta)
=
\frac{1}{n}\sum_{i=1}^n
\left[y^{(i)}-z(\bX^{(i)})\right]^2
\) and 
\(
\mathcal L(\btheta)=\E\left[\mathcal L_n(\btheta)\right].
\)
Let \(\mathcal L^\ast\) denote the oracle population risk when true \(\bB\) and
\(\bA^{(i)}\) are known. Since
\(\xi^{(i)}\sim\mathcal N(0,\sigma^2)\),
\(
\mathcal L^\ast
=
\E\left[
y^{(i)}
-
\left\langle
\bB,(\bX^{(i)})^\top\bA^{(i)}
\right\rangle
\right]^2
=
\sigma^2.
\)
Our goal is to bound the excess prediction risk
\(
\mathcal L(\hbtheta)-\mathcal L^\ast,
\)
where
\(\hbtheta=(\hbV,\{\hbW_k\}_{k=1}^s)\)
solves problem~\eqref{obj}.

We impose the following empirical-process regularity condition:
\begin{equation}\label{eq:uniform-loss-condition}
\sup_{\btheta\in\bTheta}
\left|
\mathcal L_n(\btheta)-\mathcal L(\btheta)
\right|
=
O_{\Prob}\left(
\sqrt{\frac{D\log n}{n}}
\right).
\end{equation}
This condition can be verified for suitable parameter classes under
sub-Gaussian design and noise assumptions, uniform envelope control, and
standard entropy bounds for the induced squared-loss class. Precise sufficient
conditions are provided in
Section~\ref{app:subsec:empirical-process} of the supplementary material.

\begin{theorem}[Excess prediction risk bound]\label{thm:erm_excess}
Suppose that the uniform empirical-process condition~\eqref{eq:uniform-loss-condition} holds and that the true coefficients satisfy
$\max_{1\le k\le s}\|\bbeta_k\|_2\le C_{\beta}\sqrt{d}$ for certain constant $C_{\beta}$. Further assume the standard row-wise regularity condition detailed in Section~\ref{app:assum:rowwise-regularity}, together with the conditions of Theorem~\ref{thm:selection_multi_s1}. If the score-gap
condition~\eqref{eq:score-gap-condition-multi} holds with
$(\varepsilon,\eta)=(\varepsilon_n,\eta_n)$, then
\begin{equation}\label{eq:scale-explicit-excess}
\mathcal L(\widehat\btheta)-\mathcal L^\ast
=
O_{\Prob}\left[
\sqrt{\frac{D\log n}{n}}
+
s^2 d
\left\{
(d+\log p_{\max})\varepsilon_n^2
+
\eta_n\left(
d+\log\frac{p_{\max}}{\eta_n}
\right)
\right\}
\right].
\end{equation}
Moreover, when $(\varepsilon_n,\eta_n)$ satisfies \(
\varepsilon_n \asymp d^{-1/2}s^{-3/4} \left( \frac{\log n}{n} \right)^{1/4}, 
\)
\(
\eta_n \asymp d^{-1}s^{-3/2} \left( \frac{\log n}{n} \right)^{1/2},
\)
the final excess prediction risk reduces to 
\begin{equation}\label{th-4.8-2}
\mathcal L(\widehat\btheta)-\mathcal L^\ast
=
O_{\Prob}\left(
\sqrt{\frac{D\log n}{n}}
\right).
\end{equation}
\end{theorem}

The first term in (\ref{eq:scale-explicit-excess}) reflects the empirical risk error associated with the effective
parameter dimension, whereas the remaining terms quantify the error due to
imperfect localization. If additional structure in $\bB$ or $\bW_0$, such as
low rank, is known, incorporating the corresponding constraints into the
optimization may further reduce the empirical risk error; see
Section~\ref{app:subsec:structured-risk} of the supplementary material. When the localization error is appropriately controlled by suitable choices of
$(\eps_n,\eta_n)$, balancing it with the empirical risk error yields the final
prediction risk bound in~\eqref{th-4.8-2}.


\begin{remark}[Dimensional implications]
For the excess prediction risk in Theorem \ref{thm:erm_excess} to vanish, we require
\(
D\log n=o(n).
\)
The score-gap condition~\eqref{eq:score-gap-condition-multi} further imposes
growth restrictions on the maximum number of rows \(p_{\max}\) and the row
dimension \(d\). Nevertheless, our theory accommodates input matrices of fairly
large dimensions. For example, the conditions of the theorem can hold when 
\[
d=o\left(\sqrt{\frac{n}{s\log n}}\right),
\qquad
p_{\max}
=O\left(
\frac{1}{s^{5/2}}
\sqrt{\frac{\log n}{n}}
\exp\left(c\frac{d}{s}\right)
\right),
\]
for a sufficiently small constant $c>0$.
The corresponding compatibility conditions are provided in
Section~\ref{app:subsec:growth-conditions} of the supplementary material.
\end{remark}

\section{Simulation studies}
\label{section: simulation}
In this section, we conduct simulations to evaluate the proposed diagonalized attention method in terms of prediction accuracy and localization of sample-specific active rows. We consider two regimes: a Gaussian regression setting consistent with model~(\ref{model1}) and a Gaussian-mixture classification setting whose conditional mean falls outside the GLM in~(\ref{model1}), providing a deliberately misspecified test of the localization mechanism. For comparability of ROI identification across methods, we set \(p^{(1)}=\ldots=p^{(n)}\coloneq p\).

\subsection{Gaussian-distributed active rows for regression}\label{sec:simu_reg}
We first consider linear regression settings in which active and inactive rows differ in their Gaussian distributions. The response is generated from the individualized sparse model
\begin{equation}
	\label{eq:DGP-reg}
	y^{(i)} = \sum_{j=1}^s \left\langle\bX_{S_j^{(i)},\cdot}^{(i)},\ \bbeta_j^\top\right\rangle + \varepsilon^{(i)},\quad i\in[n],
\end{equation}
where $\varepsilon^{(i)}\sim\calN(0,1)$, and $\calS^{(i)}\coloneq\{S_j^{(i)}\}_{j=1}^s$ is the set of active rows for sample $i$. The coefficient vector $\bbeta_j\in\bbR^d$ is shared across samples for the $j$-th active row and has entries independently generated from $\calN(1,1)$. For each sample, $s$ distinct active indices are randomly drawn from $\{1,\ldots,p\}$. Conditional on these indices, the active row $\bX_{S_j^{(i)},\cdot}^{(i)}$ is generated from $\calN_d(\mu_j\bmu,\sigma_j^2\bSigma)$, while all inactive rows are generated from $\calN_d(\bmu,\bSigma)$. The entries of $\bmu\in\bbR^d$ are independently drawn from $\calN(0,0.1^2)$, and $\bSigma$ is a correlation matrix to be specified later. 

We compare the proposed diagonalized attention mechanism (denoted as DiagAtt) with four competitors: group-lasso regression (\citealp[GroupLasso]{yuan2006model}), which ignores heterogeneity in signal locations, the multidirectional separation penalty (\citealp[MDSP]{tang2021individualized}), the standard multi-head attention (denoted as StdAtt), and a two-step clustering-based procedure using Gaussian mixture modeling (denoted as Two-Step). In this simulation setting, active rows across all observations are assumed to follow one common distribution, while non-active rows follow another. A straightforward two-step approach is therefore to treat the full collection of \(\sum_{i=1}^n p^{(i)}\) rows from the \(n\) samples as independent observations, apply a standard clustering algorithm such as a Gaussian mixture model to identify the active rows, and then use only the identified rows in the subsequent regression or classification analysis.
Additionally, GroupLasso treats each row as a group to encourage row-wise sparsity, while MDSP uses a row-wise multidirectional separation penalty with two latent subgroups to achieve row localization and individualized effects.
For both StdAtt and DiagAtt, we consider single-head and multi-head attention while keeping their architectures identical. 
Additional implementation details of different methods are given in Section~\ref{supp_sec:detail} of the supplementary material.

To evaluate prediction and active-row localization, we independently generate a test set of size \(n_{\mathrm{test}}=n/4\), where \(n\) is the training sample size. We report the test Root Mean Squared Error (RMSE) and Correct ROI Localization Rate (CRLR), i.e., 
\[ \text{RMSE}=\frac{\|\bm{Y}^{\mathrm{test}}-\widehat{\bm{Y}}^{\mathrm{test}}\|_2}{\sqrt{n_{\mathrm{test}}}}, \ \ 
\mathrm{CRLR}
= \frac{1}{n_{\mathrm{test}}}
\sum_{i=1}^{n_{\mathrm{test}}}
\frac{\mathrm{Card}\bigl(\calS^{(i)}\cap\widehat{\calS}^{(i)}\bigr)}
{\mathrm{Card}\bigl(\calS^{(i)}\bigr)}\times 100\%,
\]
where $\widehat{\calS}^{(i)}$ denotes the predicted active row set for sample $i$. We report two sets of experiments: (i) varying the sample size and dimension and (ii) varying the number of active rows per sample. All results are averaged over 100 independent replications.

\textbf{\textsc{Varying sample size and dimension.}}
We begin with the single-signal setting $s=1$ and $\mu_1=\sigma_1^2=2$. Table~\ref{tab:simu-reg} reports prediction and ROI localization results over different matrix dimensions with a sample size of $n=9000$. Additional results for $n=18000$ are given in Section~\ref{supp_sec:add_simu_reg} of the supplementary material. DiagAtt gives the lowest RMSE in all configurations and also achieves the highest CRLR. The improvement over StdAtt is particularly informative, since both methods use attention-type representations but differ in scoring structures. The results suggest that discarding the off-diagonal query-key interactions and concentrating on row-wise diagonal scores leads to more reliable localization in this individualized sparse setting. 
For the comparison between single-head and multi-head attention, since the number of active rows is fixed at \(s=1\) in this setting, single-head attention is sufficient and demonstrates better performance. The case with \(s>1\) is presented later.

\begin{table}[!htbp]
    \centering
    \fontsize{10.5}{12}\selectfont
    \renewcommand{\arraystretch}{1.25}
    \caption{Performance of different methods across varying dimensions under the Gaussian regression setting with $n=9000$. The best and second-best results are marked by \textbf{bold} and \underline{underline}, respectively.}
    \label{tab:simu-reg}
    \begin{adjustbox}{max width=\linewidth}
\begin{tabular}{c|c|c|c|c|c}
        \hline
        Metric                          & Method           & $(p,d)=(100,70)$          & $(p,d)=(100,90)$          & $(p,d)=(120,70)$          & $(p,d)=(120,90)$          \\ \hline
        \multirow{7}{*}{RMSE}   & GroupLasso       & 15.181 (0.239)            & 16.122 (0.224)            & 15.262 (0.197)            & 16.145 (0.226)            \\
                                      & MDSP             & 8.707 (0.107)             & 8.138 (0.106)             & {\ul 7.997 (0.111)}       & 7.544 (0.105)             \\
                                      & Two-Step         & 15.483 (0.278)            & 16.285 (0.522)            & 15.672 (0.222)            & 16.309 (0.329)            \\
                                      & 1-head StdAtt    & 14.913 (0.214)            & 15.511 (0.232)            & 15.180 (0.190)            & 15.778 (0.222)            \\
                                      & $R$-head StdAtt  & 15.360 (0.229)            & 16.034 (0.255)            & 15.871 (0.209)            & 16.637 (0.244)            \\
                                      & 1-head DiagAtt   & \textbf{7.294 (0.259)}    & \textbf{6.000 (0.249)}    & \textbf{7.716 (0.267)}    & \textbf{6.716 (0.287)}    \\
                                      & $R$-head DiagAtt & {\ul 7.779 (0.269)}       & {\ul 6.122 (0.263)}       & 8.210 (0.278)             & {\ul 6.883 (0.298)}       \\ \hline
        \multirow{7}{*}{CRLR} & GroupLasso       & 9.97\% (0.68\%)           & 10.07\% (0.58\%)          & 8.30\% (0.52\%)           & 8.36\% (0.52\%)           \\
                                      & MDSP             & 31.00\% (1.25\%)          & 33.92\% (1.07\%)          & 26.46\% (1.35\%)          & 27.97\% (0.89\%)          \\
                                      & Two-Step         & 24.72\% (14.33\%)         & 26.59\% (23.32\%)         & 11.98\% (9.68\%)          & 25.26\% (16.17\%)         \\
                                      & 1-head StdAtt    & 68.64\% (1.09\%)          & 77.48\% (1.02\%)          & 64.83\% (1.16\%)          & 72.98\% (0.97\%)          \\
                                      & $R$-head StdAtt  & 69.76\% (1.12\%)          & 77.58\% (1.17\%)          & 66.81\% (1.02\%)          & 74.90\% (1.06\%)          \\
                                      & 1-head DiagAtt   & {\ul 79.13\% (0.96\%)}    & {\ul 87.43\% (0.73\%)}    & {\ul 76.95\% (0.91\%)}    & {\ul 84.15\% (0.89\%)}    \\
                                      & $R$-head DiagAtt & \textbf{79.55\% (0.95\%)} & \textbf{87.94\% (0.77\%)} & \textbf{77.57\% (0.92\%)} & \textbf{84.91\% (0.89\%)} \\ \hline
    \end{tabular}
\end{adjustbox}
\end{table}

The patterns across $(p,d)$ are consistent with the theory. Decreasing $p$ reduces the number of candidate noise rows competing with the true ROI, and increasing $d$ provides more row-wise distributional information for separating active rows from background rows. These effects are visible for DiagAtt in both RMSE and CRLR. In contrast, GroupLasso performs poorly because it imposes fixed row-level structure across samples. MDSP attains relatively competitive prediction error but much weaker ROI localization, reflecting a mismatch between subgroup-level coefficient heterogeneity and the randomly located ROI structure used here. The two-step procedure is unstable, associated with the class-imbalance and within-sample dependence concerns.

\textbf{\textsc{Varying Sparsity Level.}}
We now consider the setting in which each sample contains multiple active rows. Specifically, we examine the case $s=2$ with $(\mu_1,\sigma_1^2)=(1,3)$ and $(\mu_2,\sigma_2^2)=(4,2)$, as well as the case $s=3$ with $(\mu_1,\sigma_1^2)=(1,3)$, $(\mu_2,\sigma_2^2)=(4,2)$, and $(\mu_3,\sigma_3^2)=(-4,2)$. The sample size is increased to $n=27000$, while the remaining parameters are fixed at $\rho=0$ and $(p,d)=(100,70)$. Here, we consider the number of attention heads $R\in\{1,2,3,4,5\}$.

Table~\ref{tab:simu-cardinality} shows that increasing the number of active rows substantially increases the difficulty of the problem for most methods. Nevertheless, the multi-head DiagAtt maintains a clear advantage over other competitors in terms of both RMSE and CRLR, whereas the single-head version suffers a marked performance deterioration because a single head can identify at most one active row among the $s$ active rows. This result further suggests that the multi-head construction can allocate different heads to different active rows, supporting the theoretical role of multi-head diagonalized attention in extending single-ROI localization to multiple individualized signal rows.

\begin{table}[!htbp]
    \centering
    \fontsize{10.5}{12}\selectfont
	\renewcommand{\arraystretch}{1.25}
    \caption{Performance of different methods across varying numbers of active rows per sample under the Gaussian regression setting. The best and second-best results are marked by \textbf{bold} and \underline{underline}, respectively.}
    \label{tab:simu-cardinality}
    \begin{tabular}{c|cc|cc}
        \hline
        \multirow{2}{*}{Method} & \multicolumn{2}{c|}{$s=2$}                                              & \multicolumn{2}{c}{$s=3$}                                               \\ \cline{2-5} 
                                & \multicolumn{1}{c|}{RMSE}                   & CRLR                      & \multicolumn{1}{c|}{RMSE}                   & CRLR                      \\ \hline
        GroupLasso              & \multicolumn{1}{c|}{27.431 (0.241)}         & 9.97\% (0.28\%)           & \multicolumn{1}{c|}{33.401 (0.288)}         & 9.98\% (0.26\%)           \\
        MDSP                    & \multicolumn{1}{c|}{11.995 (0.239)}         & 71.58\% (0.77\%)          & \multicolumn{1}{c|}{{\ul 12.099 (0.254)}}   & 75.87\% (0.62\%)          \\
        Two-Step                & \multicolumn{1}{c|}{23.220 (2.516)}         & 41.27\% (10.02\%)         & \multicolumn{1}{c|}{23.759 (3.305)}         & 54.98\% (13.19\%)         \\
        1-head StdAtt           & \multicolumn{1}{c|}{22.304 (0.215)}         & 85.68\% (0.42\%)          & \multicolumn{1}{c|}{23.239 (0.243)}         & {\ul 88.45\% (0.30\%)}    \\
        $R$-head StdAtt         & \multicolumn{1}{c|}{19.215 (0.201)}         & {\ul 87.95\% (2.34\%)}    & \multicolumn{1}{c|}{21.029 (0.222)}         & 81.78\% (4.51\%)          \\
        1-head DiagAtt          & \multicolumn{1}{c|}{{\ul 11.954 (0.138)}}   & 48.36\% (0.12\%)          & \multicolumn{1}{c|}{17.745 (0.205)}         & 32.03\% (0.09\%)          \\
        $R$-head DiagAtt        & \multicolumn{1}{c|}{\textbf{5.012 (0.208)}} & \textbf{95.98\% (0.17\%)} & \multicolumn{1}{c|}{\textbf{8.874 (0.185)}} & \textbf{89.30\% (0.39\%)} \\ \hline
    \end{tabular}
\end{table}

\subsection{Gaussian-mixture-distributed active rows for classification}\label{sec:simu_cla}
We next consider a binary classification setting that departs from the preceding linear regression model. The inactive rows follow the same background Gaussian distribution, while the active row follows a class-dependent Gaussian mixture. We restrict \(s=1\) and generate the data as follows:
\begin{equation}
	\label{eq:data-generation-cla}
	y^{(i)}\sim\mathrm{Ber}(\pi)\quad \mathrm{and}\quad \bX_{j,\cdot}^{(i)}\sim
	\begin{cases}
		\calN_d(\bmu,\bSigma),\quad &\mathrm{if}\ j\notin \calS^{(i)};\\
		\calN_d(\mu_P\bmu,\sigma_P^2\bSigma),\quad &\mathrm{if}\ j\in \calS^{(i)}\ \mathrm{and}\ y^{(i)}=1;\\
		\calN_d(\mu_N\bmu,\sigma_N^2\bSigma),\quad &\mathrm{if}\ j\in \calS^{(i)}\ \mathrm{and}\ y^{(i)}=0.
	\end{cases}
\end{equation}
Here, $y^{(i)}$ belongs to either the positive or negative class, and $\mathrm{Ber}(\pi)$ denotes a Bernoulli distribution with success probability $\pi$. This setting mimics real-world classification problems arising in sentiment analysis and sparse image regression. However, under this design, the responses are not generated from an explicit linear conditional mean model. Therefore, this experiment evaluates whether accurate individualized localization remains beneficial when the regression model is misspecified.

We set $\pi=0.5$ for class balance and use $(\mu_P,\sigma_P^2)=(3,2)$ and $(\mu_N,\sigma_N^2)=(-3,2)$ to create a clear separation between positive and negative active rows. The active set $\calS^{(i)}$, mean vector $\bmu$, and covariance matrix $\bSigma$ are generated as in Section~\ref{sec:simu_reg}, and all the rows are independent within each sample in this experiment. Figure~\ref{fig:simu-mixture-main} summarizes classification and ROI localization results for $n\in\{9000,18000\}$ and $(p,d)=(100,70)$; additional results are reported in Section~\ref{supp_sec:add_simu_cla} of the supplementary material.

DiagAtt remains competitive across all configurations, both in classification accuracy and ROI localization. This indicates that the diagonalized attention weights can still locate class-informative rows even when the final linear prediction model is not correctly specified. 
When the class-informative rows are correctly identified by $\bkappa^{(i)}$ in DiagAtt, a standard logistic regression is sufficient to produce accurate classification results.
The experiment demonstrates the effectiveness of the attention-based method beyond linear data-generating models.

\begin{figure}[!htbp]
	\centering
	\includegraphics[width=1\textwidth]{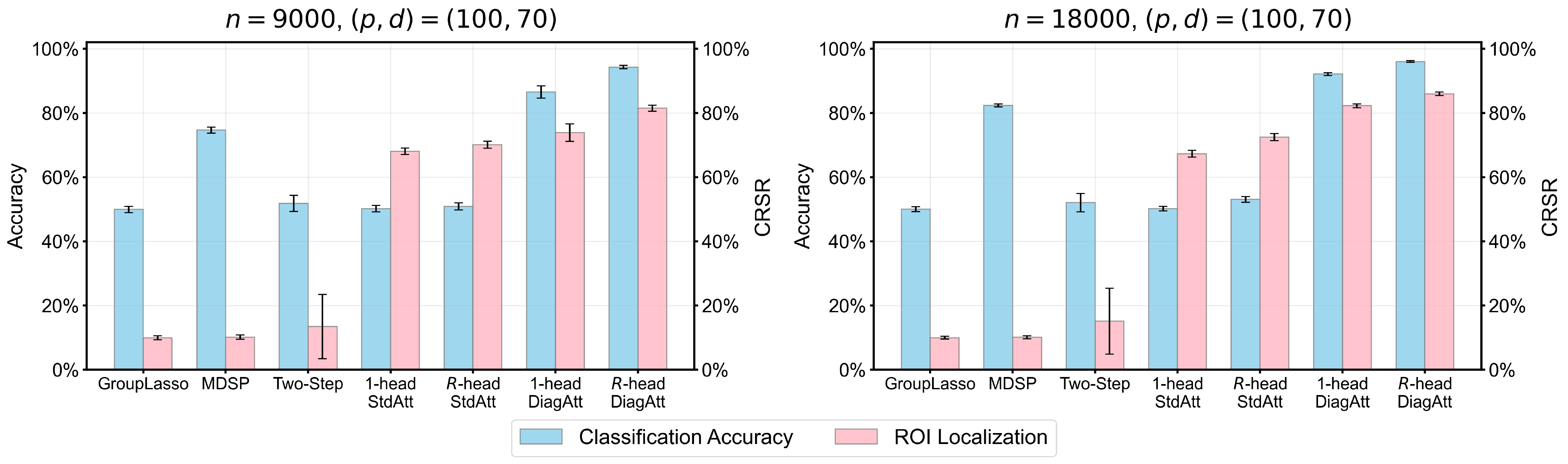}
	\caption{Performance of different methods under the Gaussian-mixture classification setting with $n\in\{9000,18000\}$ and $(p,d)=(100,70)$. Each subplot displays classification accuracy (left axis, measured by Accuracy) and ROI localization (right axis, measured by CRLR), with bar heights and error bars representing the average and standard deviation of each metric, respectively.}
	\label{fig:simu-mixture-main}
\end{figure}


\section{Real data analysis}\label{section: real}
In this section, we study sentiment analysis, a real-world language application, using the diagonalized attention mechanism. Sentiment analysis is a fundamental task in natural language processing that leverages semantic information in texts to determine sentiment orientation, typically categorized as positive or negative, and has broad applications in public opinion monitoring and decision-making. 
We consider the \textit{Sentiment Labelled Sentences} dataset from the UC Irvine Machine Learning Repository\footnote{\url{https://archive.ics.uci.edu/dataset/331/sentiment+labelled+sentences}}.  This dataset consists of 3,000 short comments on movies, products, shops, and restaurants, each labeled as positive or negative. The labels are balanced between positive and negative classes, with a positive class ratio of 0.5.

As described in Example~\ref{eg:language}, we first tokenize each comment into a sequence of tokens and represent each token as a feature vector of dimension $d=1024$. This representation is obtained using a pretrained language model from Hugging Face\footnote{\url{https://huggingface.co/}}, a popular open-source AI platform. Each comment is then transformed into an embedding matrix $\bX^{(i)}$ by stacking the corresponding $p^{(i)}$ feature vectors in the original token order. Figure~\ref{fig:realUCI}(\hyperref[fig:realUCI]{a}) shows that the token counts \(p^{(i)}\) vary substantially across comments and are right-skewed, with only a small fraction of comments containing many tokens. To apply the competitors, all \(\bX^{(i)}\) must be zero-padded to a common matrix dimension, which may introduce redundancy and unnecessary computational burden.

\begin{figure}[!htbp]
	\centering
	\includegraphics[width=\textwidth]{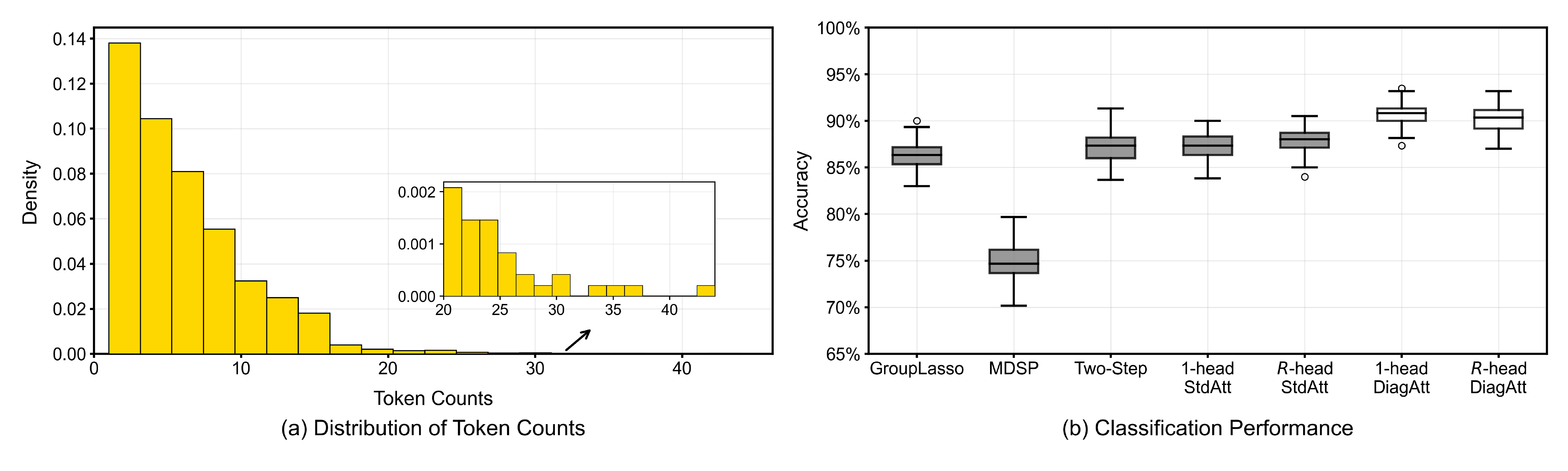}
	\caption{(a) Distribution of token counts across comments in the UCI online comment dataset, with an inset showing a zoomed-in view of the right-tail region; (b) Classification performance of different methods on the UCI online comment dataset.}
	\label{fig:realUCI}
\end{figure}

We randomly split the data into training (70\%), validation (10\%), and test (20\%) sets for model training, hyperparameter tuning, and model evaluation, respectively. Specifically, the hyperparameters are selected by mean validation cross-entropy. The training and validation sets are then combined to refit the final model, which is evaluated on the test set. This partitioning procedure is repeated 100 times, and results are averaged over the repetitions.

Figure~\ref{fig:realUCI}(\hyperref[fig:realUCI]{b}) compares the classification accuracy of different methods. DiagAtt continues to dominate in classification performance over all the competitors, with the single-head attention outperforming the $R$-head version. Notably, MDSP fails to deliver expected classification accuracy despite access to labeled data during testing, implying a significant mismatch between individualized patterns in sentiment analysis and its subgroup-based modeling assumption. The other competitors attain acceptable but not competitive classification performance.

Next, we examine ROI localization. After training, we apply the procedures in Section~\ref{supp_subsec:localization} to the full dataset. The competitors select up to \(\widetilde{R}=5\) tokens per comment, while DiagAtt identifies at most \(R=3\) tokens per comment. Selected tokens are assigned sentiment polarity using their comment labels and then aggregated across all comments and 100 repetitions. Table~\ref{tab:UCI-representative-tokens} reports the top 10 representative tokens for each polarity under each method, ranked by selection frequency.

\begin{table}[!htbp]
    \centering
    \fontsize{9.5}{12}\selectfont
    \renewcommand{\arraystretch}{1.1}
    \caption{Top 10 representative tokens for each sentiment polarity under different methods on the UCI online comment dataset. Here, tokens are ranked by selection frequency, and the correctly identified sentiment-bearing tokens are marked by "$\ast$". Note: Since ground-truth ROIs are unavailable, "$\ast$" denote tokens manually judged by the authors to convey the corresponding sentiment polarity based on their conventional usage.}
    \label{tab:UCI-representative-tokens}
    \begin{adjustbox}{max width=\linewidth}
\begin{tabular}{c|c|c|c|c|c|c|c}
        \hline
        Label                      & GroupLasso  & MDSP        & Two-Step    & 1-head StdAtt & $R$-head StdAtt  & 1-head DiagAtt    & $R$-head DiagAtt    \\ \hline
        \multirow{10}{*}{Positive} & great$\ast$ & great$\ast$ & great$\ast$ & great$\ast$    & great$\ast$     & great$\ast$        & great$\ast$        \\
                                   & good$\ast$  & good$\ast$  & good$\ast$  & good$\ast$     & good$\ast$      & good$\ast$         & good$\ast$         \\
                                   & hat         & wit$\ast$   & wit$\ast$   & well$\ast$     & well$\ast$      & best$\ast$         & wit$\ast$          \\
                                   & wit$\ast$   & one         & hat         & film           & best$\ast$      & excellent$\ast$    & hat                \\
                                   & film        & hat         & film        & best$\ast$     & really          & well$\ast$         & excellent$\ast$    \\
                                   & movie       & film        & phone       & phone          & excellent$\ast$ & love$\ast$         & well$\ast$         \\
                                   & phone       & movie       & one         & one            & love$\ast$      & nice$\ast$         & best$\ast$         \\
                                   & one         & phone       & movie       & movie          & like$\ast$      & amazing$\ast$      & love$\ast$         \\
                                   & food        & like$\ast$  & best$\ast$  & really         & nice$\ast$      & really             & service            \\
                                   & best$\ast$  & well$\ast$  & food        & place          & works           & wonderful$\ast$    & nice$\ast$         \\ \hline
        \multirow{10}{*}{Negative} & not$\ast$   & not$\ast$   & not$\ast$   & not$\ast$      & not$\ast$       & not$\ast$          & not$\ast$          \\
                                   & hat$\ast$   & hat$\ast$   & hat$\ast$   & bad$\ast$      & bad$\ast$       & bad$\ast$          & hat$\ast$          \\
                                   & wit         & wit         & bad$\ast$   & movie          & good            & good               & bad$\ast$          \\
                                   & bad$\ast$   & movie       & wit         & no$\ast$       & thing           & worst$\ast$        & wit                \\
                                   & movie       & food        & movie       & hat$\ast$      & no$\ast$        & terrible$\ast$     & no$\ast$           \\
                                   & no$\ast$    & phone       & no$\ast$    & one            & worst$\ast$     & no$\ast$           & good               \\
                                   & one         & one         & one         & time           & work            & poor$\ast$         & worst$\ast$        \\
                                   & phone       & time        & phone       & like           & terrible$\ast$  & disappointed$\ast$ & disappointed$\ast$ \\
                                   & food        & film        & food        & good           & place           & like               & terrible$\ast$     \\
                                   & time        & even        & time        & food           & like            & awful$\ast$        & service            \\ \hline
\end{tabular}
\end{adjustbox}
\end{table}

Table~\ref{tab:UCI-representative-tokens} shows that most competitors select many non-informative tokens for each sentiment polarity, along with only a few sentiment-bearing terms, which may explain their limited classification accuracy. In contrast, the single-head DiagAtt identifies representative tokens largely consistent with the corresponding sentiment polarity. Although it occasionally includes non-informative or seemingly mismatched tokens, such as ``really'' for positive sentiment and ``good'' or ``like'' for negative sentiment, these choices are often interpretable: ``really'' can amplify positive sentiment, while phrases such as ``not good'' or ``not like'' convey negative sentiment.
The multi-head DiagAtt and StdAtt achieve comparable ROI localization results, but both select some incorrect or non-informative tokens, reducing classification accuracy. StdAtt, with its higher-dimensional parameter space, exhibits poorer prediction accuracy and more frequent misidentification under the limited sample size. Overall, these results suggest that diagonalized attention is well-suited to capture individualized patterns in sentiment analysis. Additional analysis using another example from \citet{marion2025attention} is provided in Section~\ref{supp_sec:add_emp} of the supplementary material, further supporting these findings.


\clearpage
\pdfbookmark[0]{Supplementary material}{supplementary-material}
\begin{center}
{\Large\bfseries Supplementary material for ``Diagonalized Attention for Individualized Regression: Latent-Row Localization and Prediction''}\par
\end{center}
\setcounter{equation}{0}
\renewcommand{\theequation}{S\arabic{equation}}
\setcounter{figure}{0}
\renewcommand{\thefigure}{S\arabic{figure}}
\setcounter{table}{0}
\renewcommand{\thetable}{S\arabic{table}}
\setcounter{section}{0}
\renewcommand{\thesection}{S\arabic{section}}
\renewcommand{\theHequation}{supp.\arabic{equation}}
\renewcommand{\theHfigure}{supp.\arabic{figure}}
\renewcommand{\theHtable}{supp.\arabic{table}}
\renewcommand{\theHsection}{supp.\arabic{section}}
\setcounter{theorem}{0}
\setcounter{remark}{0}
\renewcommand{\theremark}{\arabic{remark}}
\makeatletter
\@removefromreset{remark}{section}
\let\c@proposition\c@theorem
\let\c@definition\c@theorem
\renewcommand{\theproposition}{\thetheorem}
\renewcommand{\thedefinition}{\thetheorem}
\@addtoreset{example}{section}
\renewcommand{\theexample}{\thesection.\arabic{example}}
\makeatother
\noindent This supplementary material is organized into seven sections. Section~\ref{app:sec:selection} establishes localization guarantee by deriving a deterministic gap-to-softmax bound and applying it to the single-head and multi-head settings. 
Section~\ref{app:sec:covariance-constructions} verifies covariance-based sufficient constructions and presents extensions
to heterogeneous backgrounds, sub-Gaussian concentration, and informative subspaces. 
Section~\ref{app:sec:excess-risk} proves the excess-risk theorem
through an estimation--approximation decomposition and further discusses structured query--key matrices and dimension-growth conditions.
Section~\ref{app:sec:overlap-discussion} provides additional analysis of the spectral-overlap assumption through analytic and shared-and-class-specific constructions, examines the effect of removing the whitened means, and concludes with a numerical illustration.
Section~\ref{supp_sec:add_simu} provides additional simulation results for both Gaussian regression and Gaussian-mixture classification settings. Section~\ref{supp_sec:add_emp} introduces another real sentiment analysis example. Finally, Section~\ref{supp_sec:detail} describes supplementary implementation details, including model architectures, hyperparameter tuning procedures, and ROI localization strategies. All notation in the supplementary material is consistent with that in the main text unless otherwise specified.

\section{Proofs for localization guarantee results}
\label{app:sec:selection}

This section proves two localization consistency results under the
quadratic-score gap and concentration assumptions of the main text.  The covariance-based
constructions in Theorems~\ref{prop1} and~\ref{prop2} are sufficient
mechanisms for producing these gaps, but they are not needed by the
localization arguments themselves.  We therefore first prove the abstract
gap-to-localization implication and verify the covariance constructions
afterward.

Throughout the localization proofs, score expectations are conditional on the
realized latent selector \(\balpha^{(i)}\) in the single-head case and on
the realized latent assignment matrix \(\bA^{(i)}\) in the multi-head case.
The expectations in the main-text assumptions are understood in this sense:
when the active locations are treated as fixed, these are simply the ordinary
expectations displayed there.
No independence among rows, or among scores produced by different heads,
is required.  The temperature parameter is already incorporated into the
constructed query--key matrices.  Thus the attention weights are written as
\(\softmax(\bm{\psi})\), with no additional rescaling of \(\bm{\psi}\) by
\(\tau\).

\subsection{A deterministic gap-to-softmax bound}
\label{app:subsec:softmax-lemma}

\begin{lemma}[Deterministic gap-to-softmax bound]
\label{app:lem:deterministic-softmax}
Let \(p\ge2\), let
\(\boldsymbol z,\boldsymbol m\in\bbR^p\), and fix \(j^\ast\in[p]\).
Suppose that, for some \(\delta>0\) and \(r\ge0\),
\[
    m_{j^\ast}-\max_{j\ne j^\ast}m_j\ge\delta,
    \qquad
    \|\boldsymbol z-\boldsymbol m\|_\infty\le r.
\]
If \(g:=\delta-2r>0\), then
\begin{equation}
\label{app:eq:deterministic-softmax-bound}
    \left\|
        \softmax(\boldsymbol z)-\mathbf e_{j^\ast}
    \right\|_\infty
    \le
    (p-1)e^{-g}.
\end{equation}
In particular, if
\[
    \delta\ge4r
    \qquad\text{and}\qquad
    \delta\ge
    2\log\left(\frac{p-1}{\varepsilon}\right),
\]
then the left-hand side of
\eqref{app:eq:deterministic-softmax-bound} is at most \(\varepsilon\).
For \(p=1\), the conclusion holds trivially with zero error.
\end{lemma}

\begin{proof}
For every \(j\ne j^\ast\),
\[
\begin{aligned}
    z_{j^\ast}-z_j
    &=
    (m_{j^\ast}-m_j)
    +(z_{j^\ast}-m_{j^\ast})
    -(z_j-m_j)\\
    &\ge\delta-2r=g.
\end{aligned}
\]
Writing \(\bkappa=\softmax(\boldsymbol z)\), we have
\[
\begin{aligned}
    1-\kappa_{j^\ast}
    &=
    \frac{
        \sum_{j\ne j^\ast}\exp(z_j-z_{j^\ast})
    }{
        1+\sum_{j\ne j^\ast}\exp(z_j-z_{j^\ast})
    }\\
    &\le
    \sum_{j\ne j^\ast}\exp(z_j-z_{j^\ast})
    \le(p-1)e^{-g}.
\end{aligned}
\]
Moreover, every \(j\ne j^\ast\) satisfies
\[
    \kappa_j
    =
    \frac{\exp(z_j-z_{j^\ast})}{
        1+\sum_{\ell\ne j^\ast}\exp(z_\ell-z_{j^\ast})
    }
    \le e^{-g}
    \le(p-1)e^{-g}.
\]
This proves \eqref{app:eq:deterministic-softmax-bound}.  If
\(\delta\ge4r\), then \(g\ge\delta/2\), and the remaining assertion follows
from \((p-1)e^{-\delta/2}\le\varepsilon\).
\end{proof}

\subsection{The single-head localization theorem}
\label{app:subsec:single}

\begin{proof}[Proof of Theorem~\ref{thm:selection_single_s1}]
Fix an observation \(i\), write \(p_i=p^{(i)}\), and use the notation
\(j_i^\star\) from the main text for the unique index such that
\(\alpha_{j_i^\star}^{(i)}=1\).  Thus
\(\balpha^{(i)}=\mathbf e_{j_i^\star}\).  For \(j\in[p_i]\), put
\[
    \psi_{0,j}^{(i)}
    =
    \bX_{j,\cdot}^{(i)}
    \bW_0
    \bigl(\bX_{j,\cdot}^{(i)}\bigr)^\top
\]
and define the conditional center
\begin{equation}
\label{app:eq:single-conditional-center}
    m_{0,j}^{(i)}
    :=
    \mathbb{E}\!\left[
        \psi_{0,j}^{(i)}
        \,\middle|\,
        \balpha^{(i)}
    \right].
\end{equation}
Under Assumption~\ref{assum:score-gap}, with the score-gap relation
understood as a uniform lower bound,
\begin{equation}
\label{app:eq:single-conditional-gap}
    m_{0,j_i^\star}^{(i)}
    -
    \max_{j\ne j_i^\star}m_{0,j}^{(i)}
    \ge\delta_0.
\end{equation}

Let
\begin{equation}
\label{app:eq:single-concentration-event}
    \mathcal E_i^{(1)}
    :=
    \left\{
        \max_{1\le j\le p_i}
        \left|
            \psi_{0,j}^{(i)}-m_{0,j}^{(i)}
        \right|
        \le r_{\eta,p_{\max}}
    \right\}.
\end{equation}
Assumption~\ref{assum:concentration}, centered conditionally on
\(\balpha^{(i)}\), gives
\(\Prob(\mathcal E_i^{(1)}\mid\balpha^{(i)})\ge1-\eta\), uniformly in the
realized selector.  On \(\mathcal E_i^{(1)}\), for every
\(j\ne j_i^\star\),
\begin{align}
\label{app:eq:single-realized-gap-generic}
    \psi_{0,j_i^\star}^{(i)}-\psi_{0,j}^{(i)}
    &\ge
    m_{0,j_i^\star}^{(i)}-m_{0,j}^{(i)}
    -2r_{\eta,p_{\max}}\nonumber\\
    &\ge
    \delta_0-2r_{\eta,p_{\max}}.
\end{align}
The condition \(\delta_0\ge4r_{\eta,p_{\max}}\) makes the last expression
at least \(\delta_0/2\).  Lemma~\ref{app:lem:deterministic-softmax}
therefore yields
\[
\begin{aligned}
    \left\|
        \bkappa_0^{(i)}-\balpha^{(i)}
    \right\|_\infty
    &\le
    (p_i-1)\exp\left(-\frac{\delta_0}{2}\right)\\
    &\le
    (p_{\max}-1)\exp\left(-\frac{\delta_0}{2}\right)
    \le\varepsilon.
\end{aligned}
\]
The last inequality is precisely the softmax-threshold part of
\eqref{eq:score-gap-condition-single}.  Integrating the conditional
probability over \(\balpha^{(i)}\) proves
\[
    \Prob\left(
        \left\|
            \bkappa_0^{(i)}-\balpha^{(i)}
        \right\|_\infty
        \le\varepsilon
    \right)
    \ge1-\eta.
\]
Only the conditional score gap and concentration event were used; in
particular, the inactive rows need not share a distribution and need not
be independent.
\end{proof}

\subsection{The multi-head localization theorem}
\label{app:subsec:multi}

\begin{proof}[Proof of Theorem~\ref{thm:selection_multi_s1}]
Fix an observation \(i\), write \(p_i=p^{(i)}\), and condition on the
complete latent assignment matrix \(\bA^{(i)}\).  For each head
\(k\in[s]\), use the main-text notation \(j_{i,k}^\star\) for the unique
index satisfying \(A_{j_{i,k}^\star,k}^{(i)}=1\), and put
\[
    \balpha_k^{(i)}
    :=
    \bA_{\cdot,k}^{(i)}
    =
    \mathbf e_{j_{i,k}^\star}.
\]
For \(k\in[s]\) and \(j\in[p_i]\), define
\[
    \psi_{k,0,j}^{(i)}
    :=
    \bX_{j,\cdot}^{(i)}
    \bW_{0,k}
    \bigl(\bX_{j,\cdot}^{(i)}\bigr)^\top,
    \qquad
    m_{k,0,j}^{(i)}
    :=
    \mathbb{E}\!\left[
        \psi_{k,0,j}^{(i)}
        \,\middle|\,
        \bA^{(i)}
    \right].
\]
Every row \(j\ne j_{i,k}^\star\) is either inactive or is the row associated
with a signal class \(\ell\ne k\).  By the definitions of
\(\delta_k^{(0)}\), \(\delta_k^{(1)}\), \(\delta_k\), and
\(\delta_{\min}\) in Assumption~\ref{assum:multi-score-gap},
\begin{equation}
\label{app:eq:multi-conditional-gap}
    m_{k,0,j_{i,k}^\star}^{(i)}
    -
    m_{k,0,j}^{(i)}
    \ge\delta_k
    \ge\delta_{\min}
    \quad
    \text{for every }k\in[s]\text{ and }j\ne j_{i,k}^\star.
\end{equation}

Introduce the joint head--row concentration event
\begin{equation}
\label{app:eq:multi-joint-concentration-event}
    \mathcal E_i^{(s)}
    :=
    \left\{
        \max_{\substack{1\le k\le s\\1\le j\le p_i}}
        \left|
            \psi_{k,0,j}^{(i)}
            -
            m_{k,0,j}^{(i)}
        \right|
        \le r_{\eta,p_{\max},s}
    \right\}.
\end{equation}
The uniform-over-heads version of Assumption~\ref{assum:concentration}
imposed in Theorem~\ref{thm:selection_multi_s1} is understood jointly over
all head--row pairs; it gives
\(\Prob(\mathcal E_i^{(s)}\mid\bA^{(i)})\ge1-\eta\), uniformly in the
realized assignment matrix.  On this single event, simultaneously for
all \(k\in[s]\) and \(j\ne j_{i,k}^\star\),
\[
    \psi_{k,0,j_{i,k}^\star}^{(i)}
    -
    \psi_{k,0,j}^{(i)}
    \ge
    \delta_{\min}-2r_{\eta,p_{\max},s}
    \ge\frac{\delta_{\min}}{2}.
\]
The last inequality uses the fluctuation part of
\eqref{eq:score-gap-condition-multi}.
Applying Lemma~\ref{app:lem:deterministic-softmax} to every head gives
\[
\begin{aligned}
    \left\|
        \bkappa_{k,0}^{(i)}-\balpha_k^{(i)}
    \right\|_\infty
    &\le
    (p_i-1)\exp\left(-\frac{\delta_{\min}}{2}\right)\\
    &\le
    (p_{\max}-1)\exp\left(-\frac{\delta_{\min}}{2}\right)
    \le\varepsilon.
\end{aligned}
\]
The final inequality uses the softmax-threshold part of
\eqref{eq:score-gap-condition-multi}.
Since the columns of \(\bK_0^{(i)}-\bA^{(i)}\) are precisely these
head-specific errors,
\[
    \left\|
        \bK_0^{(i)}-\bA^{(i)}
    \right\|_\infty
    =
    \max_{1\le k\le s}
    \left\|
        \bkappa_{k,0}^{(i)}-\balpha_k^{(i)}
    \right\|_\infty
    \le\varepsilon
\]
on \(\mathcal E_i^{(s)}\).  Integrating over \(\bA^{(i)}\) proves
\[
    \Prob\left(
        \left\|
            \bK_0^{(i)}-\bA^{(i)}
        \right\|_\infty
        \le\varepsilon
    \right)
    \ge1-\eta.
\]
No additional union bound over heads is needed, because the assumed event
\(\mathcal E_i^{(s)}\) already controls all head--row pairs jointly.
\end{proof}

\section{Proofs for covariance-based sufficient constructions}
\label{app:sec:covariance-constructions}

The results in this section verify that the moment-based matrices in
Theorems~\ref{prop1} and~\ref{prop2} satisfy the abstract score-gap
assumptions.  These constructions use common background moments (or the
full-second-moment extensions below), whereas the localization theorems in the
preceding section do not.

\subsection{The single-head covariance construction}
\label{app:subsec:single-construction}

\begin{proof}[Proof of Theorem~\ref{prop1}]
For a random vector \(\bx\in\bbR^d\) with mean \(\bmu\), covariance
\(\bSigma\), and finite second moment, the quadratic-score identity
\begin{equation}
\label{app:eq:quadratic-score-identity}
    \mathbb E(\bx^\top\bW\bx)
    =
    \operatorname{tr}(\bW\bSigma)
    +
    \bmu^\top\bW\bmu
\end{equation}
holds for every deterministic matrix \(\bW\).  Thus Gaussianity is not
needed for the theorem.

Let
\[
    \bDelta_{\Sigma}:=\bSigma_1-\bSigma_0,
    \qquad
    q_0:=d-\operatorname{rank}(\bZ).
\]
Because
\(\bDelta_{\Sigma}^{-1/2}\bmu_0\) and
\(\bDelta_{\Sigma}^{-1/2}\bmu_1\) belong to
\(\operatorname{col}(\bZ)\),
\[
    \bP_{\bZ}^{\perp}
    \bDelta_{\Sigma}^{-1/2}\bmu_a
    =\mathbf0,
    \qquad a\in\{0,1\}.
\]
The matrix in \eqref{W0} consequently satisfies
\begin{equation}
\label{app:eq:single-mean-annihilation}
    \bW_0\bmu_0=\bW_0\bmu_1=\mathbf0.
\end{equation}
Applying \eqref{app:eq:quadratic-score-identity} to the active and
inactive rows and subtracting gives
\[
\begin{aligned}
    &\mathbb E\!\left(
        \bx^\top\bW_0\bx
        \,\middle|\,
        \alpha=1
    \right)
    -
    \mathbb E\!\left(
        \bx^\top\bW_0\bx
        \,\middle|\,
        \alpha=0
    \right)\\
    &\qquad=
    \operatorname{tr}
    \bigl(\bW_0\bDelta_{\Sigma}\bigr)\\
    &\qquad=
    \frac1{\tau}
    \operatorname{tr}\!\left(
        \bDelta_{\Sigma}^{-1/2}
        \bP_{\bZ}^{\perp}
        \bDelta_{\Sigma}^{-1/2}
        \bDelta_{\Sigma}
    \right)\\
    &\qquad=
    \frac1{\tau}\operatorname{tr}(\bP_{\bZ}^{\perp})
    =
    \frac{d-\operatorname{rank}(\bZ)}{\tau}.
\end{aligned}
\]
The last quantity is strictly positive precisely when
\(\operatorname{rank}(\bZ)<d\).  In particular,
\(\operatorname{rank}(\bZ)\le2\), so this is automatic when
\(d>2\).  All inactive rows have the same first
two moments in the theorem, so the same calculation applies to their
maximum expected score and proves the stated lower bound.
\end{proof}

\subsubsection{Heterogeneous inactive rows}

The expected quadratic score depends on the full second moment, not on the
covariance alone.  To make this explicit, write
\[
    \bM_1:=\mathbb E_1(\bx\bx^\top)
    =\bSigma_1+\bmu_1\bmu_1^\top,
    \qquad
    \bM_0(f):=\mathbb E_f(\bx\bx^\top)
    =\bSigma_0(f)+\bmu_0(f)\bmu_0(f)^\top.
\]
Suppose that the inactive second moments admit a positive-semidefinite
upper envelope \(\overline{\bM}_0\) such that
\[
    \bM_0(f)\preceq\overline{\bM}_0
    \quad\text{for every }f\in\mathcal F_0,
    \qquad
    \bDelta_M:=\bM_1-\overline{\bM}_0\succ0.
\]
Then the direct construction
\begin{equation}
\label{app:eq:heterogeneous-second-moment-weight}
    \bW_{0,\mathrm{het}}
    :=\frac1\tau\bDelta_M^{-1}
\end{equation}
satisfies, uniformly over \(f\in\mathcal F_0\),
\begin{align}
\label{app:eq:heterogeneous-second-moment-gap}
    &\mathbb E_1(\bx^\top\bW_{0,\mathrm{het}}\bx)
    -\mathbb E_f(\bx^\top\bW_{0,\mathrm{het}}\bx)\nonumber\\
    &\quad=
    \operatorname{tr}\!\left[
        \bW_{0,\mathrm{het}}
        \{\bM_1-\bM_0(f)\}
    \right]\nonumber\\
    &\quad=
    \frac d\tau
    +
    \operatorname{tr}\!\left[
        \bW_{0,\mathrm{het}}
        \{\overline{\bM}_0-\bM_0(f)\}
    \right]
    \ge\frac d\tau.
\end{align}
The last inequality follows because both matrices in the final trace are
positive semidefinite.  This proves the uniform gap against the strongest
inactive competitor suggested in the main-text remark.

A covariance envelope alone does not give this conclusion when the means
vary, because the term \(\bmu_0(f)^\top\bW\bmu_0(f)\) is then uncontrolled.
It does suffice when the active and inactive means are common and annihilated
as in Theorem~\ref{prop1}.  More generally, if all transformed inactive
means lie in a fixed low-dimensional span, that span can be included in the
matrix \(\bZ\); the corresponding projector annihilates every such mean, and
the covariance-envelope calculation then applies on the remaining
subspace.  These are sufficient mechanisms for the abstract heterogeneous
gap, not requirements of the localization theorem.

\subsubsection{A sub-Gaussian concentration radius}

For completeness, we give the concentration calculation underlying the
single-head sub-Gaussian remark.  Put
\[
    q_0=d-\operatorname{rank}(\bZ),
    \qquad
    \lambda_\Delta
    =
    \lambda_{\min}(\bDelta_{\Sigma}),
\]
and suppose that, conditional on its class
\(a\in\{0,1\}\), every centered active or inactive row satisfies a
Hanson--Wright inequality with a uniform sub-Gaussian scale \(K\).
Because \(\bW_0\) annihilates both means, the noncentral cross terms
vanish.  Moreover,
\[
    \operatorname{rank}(\bW_0)=q_0,
    \qquad
    \|\bW_0\|_{\mathrm{op}}
    \le\frac1{\tau\lambda_\Delta},
    \qquad
    \|\bW_0\|_F
    \le\frac{\sqrt{q_0}}{\tau\lambda_\Delta}.
\]
Consequently, for universal \(c,C>0\),
\begin{equation}
\label{app:eq:single-hanson-wright}
    \Prob\left(
        \left|
            \bx^\top\bW_0\bx
            -
            \mathbb E(
                \bx^\top\bW_0\bx
                \mid a
            )
        \right|>t
        \,\middle|\,a
    \right)
    \le
    2\exp\left[
        -c\min\left\{
            \frac{\tau^2\lambda_\Delta^2t^2}{K^4q_0},
            \frac{\tau\lambda_\Delta t}{K^2}
        \right\}
    \right].
\end{equation}
No independence between different rows is needed.  A union bound over at
most \(p_{\max}\) rows shows that, with
\(L_\eta=\log(2p_{\max}/\eta)\), one may take
\begin{equation}
\label{app:eq:single-subgaussian-radius}
    r_{\eta,p_{\max}}
    \le
    C\frac{K^2}{\tau\lambda_\Delta}
    \left\{
        \sqrt{q_0L_\eta}+L_\eta
    \right\}.
\end{equation}
Since the constructed gap is \(q_0/\tau\), the common factor
\(1/\tau\) cancels when verifying
\(\delta_0\ge4r_{\eta,p_{\max}}\).  In the regime
\(L_\eta\lesssim q_0\), a sufficient condition is
\[
    q_0
    \gtrsim
    \frac{K^4}{\lambda_\Delta^2}
    \log\left(\frac{p_{\max}}{\eta}\right).
\]
Under the scale convention
\(K\lesssim\lambda_{\max}(\bSigma_1)\), this yields the displayed
order in the main-text remark.  The softmax-threshold condition is
enforced by choosing \(\tau\) so that
\[
    \frac{q_0}{\tau}
    \ge
    2\log\left(\frac{p_{\max}-1}{\varepsilon}\right).
\]

\subsection{The multi-head covariance construction}
\label{app:subsec:multi-construction}

\begin{proof}[Proof of Theorem~\ref{prop2}]
Fix \(k\in[s]\) and consider the matrix \(\bW_{0,k}\) in
\eqref{eq:multi-W0}.  Put
\[
    q_k:=\operatorname{rank}(\bZ_k),
    \qquad
    d_k=d-q_k,
    \qquad
    \bP_k^\perp:=\bP_{\bZ_k}^{\perp}.
\]
The claimed strictly positive gaps implicitly require
\(d^\ast=\lfloor d/s\rfloor\ge1\), which is automatic when \(d\ge s\).
The matrix
\[
    \widetilde{\bD}_k
    =
    \bP_k^\perp
    \bDelta_k^{-1/2}
    \bD_k
    \bDelta_k^{-1/2}
    \bP_k^\perp
\]
vanishes on \(\operatorname{col}(\bZ_k)\).  On the
\(d_k\)-dimensional range of \(\bP_k^\perp\), it is positive definite:
for every nonzero \(\bv\) in this range,
\[
    \bv^\top\widetilde{\bD}_k\bv
    =
    (\bDelta_k^{-1/2}\bv)^\top
    \bD_k
    (\bDelta_k^{-1/2}\bv)
    >0,
\]
because \(\bD_k\succ0\).  Hence
\(\widetilde{\bD}_k\) has exactly \(d_k\) positive eigenvalues.

The columns of \(\bGamma_k\) are the eigenvectors associated with its
\(d^\ast\) smallest positive eigenvalues.  Thus
\[
    \bP_k^\perp\bGamma_k=\bGamma_k,
    \qquad
    \bGamma_k^\top\bGamma_k=\bI_{d^\ast},
\]
and, using the ordering in Theorem~\ref{prop2},
\begin{equation}
\label{app:eq:selected-tail-trace}
    \operatorname{tr}
    (\bGamma_k^\top\widetilde{\bD}_k\bGamma_k)
    =
    \sum_{q=d_k-d^\ast+1}^{d_k}\gamma_{k,q}
    =
    d^\ast\bar\gamma_k^\ast.
\end{equation}
Every transformed mean
\(\bDelta_k^{-1/2}\bmu_\ell\), \(0\le\ell\le s\), belongs to
\(\operatorname{col}(\bZ_k)\).  Therefore
\begin{equation}
\label{app:eq:multi-mean-annihilation}
    \bW_{0,k}\bmu_\ell=\mathbf0,
    \qquad \ell=0,\ldots,s.
\end{equation}
By \eqref{app:eq:quadratic-score-identity}, all class-score comparisons
therefore reduce to covariance trace differences.

For the background class,
\begin{equation}
\label{app:eq:multi-background-gap}
\begin{aligned}
    \operatorname{tr}
    \{\bW_{0,k}(\bSigma_k-\bSigma_0)\}
    &=
    \frac1{\tau}
    \operatorname{tr}
    (\bGamma_k^\top\bGamma_k)\\
    &=
    \frac{d^\ast}{\tau}
    =\delta_k^{(0)}.
\end{aligned}
\end{equation}
For a competing signal class \(\ell\ne k\), positivity of every
\(\bDelta_\ell\) and
\(\bDelta_\ell\preceq\bD_k=\sum_{h\ne k}\bDelta_h\) gives
\[
\begin{aligned}
    &\operatorname{tr}
    \left(
        \bGamma_k^\top
        \bDelta_k^{-1/2}
        \bDelta_\ell
        \bDelta_k^{-1/2}
        \bGamma_k
    \right)\\
    &\qquad\le
    \operatorname{tr}
    (\bGamma_k^\top\widetilde{\bD}_k\bGamma_k)
    =
    d^\ast\bar\gamma_k^\ast.
\end{aligned}
\]
It follows that
\begin{equation}
\label{app:eq:multi-signal-gap}
\begin{aligned}
    \operatorname{tr}
    \{\bW_{0,k}(\bSigma_k-\bSigma_\ell)\}
    &\ge
    \frac{
        d^\ast(1-\bar\gamma_k^\ast)
    }{\tau}\\
    &\ge
    \frac{
        d^\ast(1-\bar\gamma_{\max}^\ast)
    }{\tau}
    >0.
\end{aligned}
\end{equation}
Because \(\delta_k^{(1)}\) is the minimum of the target-versus-signal
gaps, \eqref{app:eq:multi-signal-gap} implies
\begin{equation}
\label{app:eq:multi-delta-min-lower-bound}
    \delta_k^{(1)}
    \ge\frac{d^\ast(1-\bar\gamma_k^\ast)}{\tau},
    \qquad
    \delta_{\min}
    \ge\frac{d^\ast(1-\bar\gamma_{\max}^\ast)}{\tau}.
\end{equation}
Here we used \(0\le\bar\gamma_k^\ast<1\), so the lower bound for
\(\delta_k^{(1)}\) does not exceed the exact inactive-row gap
\(\delta_k^{(0)}=d^\ast/\tau\).  Equations
\eqref{app:eq:multi-background-gap}--%
\eqref{app:eq:multi-delta-min-lower-bound} prove Theorem~\ref{prop2}.
\end{proof}

\subsubsection{Heterogeneous backgrounds in the multi-head construction}

As in the single-head case, a covariance envelope alone does not control
heterogeneous mean terms.  The following stronger structure is one useful
sufficient extension.  Let
\(\bSigma_0(f)\preceq\overline{\bSigma}_0\) for all background
distributions and define
\[
    \bDelta_k=\bSigma_k-\overline{\bSigma}_0\succ0.
\]
Construct \(\bZ_k\) using the transformed signal means together with a
basis for the transformed background-mean span, and then construct
\(\bD_k,\widetilde{\bD}_k,\bGamma_k,\bW_{0,k}\) as in
Theorem~\ref{prop2}.  For every background distribution \(f\),
\[
\begin{aligned}
    &\mathbb E_k(\bx^\top\bW_{0,k}\bx)
    -
    \mathbb E_f(\bx^\top\bW_{0,k}\bx)\\
    &\quad=
    \operatorname{tr}\{\bW_{0,k}
        (\bSigma_k-\overline{\bSigma}_0)\}
    +
    \operatorname{tr}\{\bW_{0,k}
        (\overline{\bSigma}_0-\bSigma_0(f))\}\\
    &\quad\ge\frac{d^\ast}{\tau}.
\end{aligned}
\]
The signal-to-signal proof is unchanged because
\(\bSigma_k-\bSigma_\ell=\bDelta_k-\bDelta_\ell\).  Thus the same
spectral condition supplies the lower bound in
\eqref{app:eq:multi-signal-gap}.

\subsubsection{A multi-head sub-Gaussian concentration radius}

Assume now that all centered rows in classes
\(\ell=0,\ldots,s\) obey a Hanson--Wright inequality with a common
sub-Gaussian scale \(K\), and put
\[
    \lambda_\Delta
    =
    \min_{1\le k\le s}\lambda_{\min}(\bDelta_k).
\]
For every head,
\[
    \operatorname{rank}(\bW_{0,k})=d^\ast,
    \qquad
    \|\bW_{0,k}\|_{\mathrm{op}}
    \le\frac1{\tau\lambda_\Delta},
    \qquad
    \|\bW_{0,k}\|_F
    \le\frac{\sqrt{d^\ast}}{\tau\lambda_\Delta}.
\]
A union bound over at most \(sp_{\max}\) head--row pairs therefore gives,
with \(L_{\eta,s}=\log(2sp_{\max}/\eta)\),
\begin{equation}
\label{app:eq:multi-subgaussian-radius}
    r_{\eta,p_{\max},s}
    \le
    C\frac{K^2}{\tau\lambda_\Delta}
    \left\{
        \sqrt{d^\ast L_{\eta,s}}+L_{\eta,s}
    \right\}.
\end{equation}
In particular, conditional on the realized assignment matrix,
\begin{align}
\label{app:eq:multi-conditional-hanson-wright}
&\Prob\Bigg(
    \max_{\substack{1\le k\le s\\1\le j\le p^{(i)}}}
    \Bigg|
        \bX_{j,\cdot}^{(i)}
        \bW_{0,k}
        \bigl(\bX_{j,\cdot}^{(i)}\bigr)^\top
        \nonumber\\[-2mm]
&\hspace{35mm}
        -
        \mathbb E\!\left[
            \bX_{j,\cdot}^{(i)}
            \bW_{0,k}
            \bigl(\bX_{j,\cdot}^{(i)}\bigr)^\top
            \,\middle|\,
            \bA^{(i)}
        \right]
    \Bigg|
    \le r_{\eta,p_{\max},s}
    \,\Bigm|\,
    \bA^{(i)}
\Bigg)
\ge1-\eta.
\end{align}
When the standardized sub-Gaussian norm is uniformly bounded,
\(K^2\) may be replaced, up to a universal constant, by
\(\max_{0\le\ell\le s}\lambda_{\max}(\bSigma_\ell)\).  If
\(L_{\eta,s}\lesssim d^\ast\), the square-root term dominates and
\eqref{app:eq:multi-subgaussian-radius} has the square-root order displayed
in the main-text remark, together with the explicit factor \(1/\tau\)
coming from the construction.  Both the constructed gap and the radius contain
the same factor \(1/\tau\), so their comparison is invariant to the
temperature scaling.

More explicitly, set
\(g_\gamma:=1-\bar\gamma_{\max}^\ast>0\).  Combining
\eqref{app:eq:multi-delta-min-lower-bound} and
\eqref{app:eq:multi-subgaussian-radius}, the fluctuation requirement
\(\delta_{\min}\ge4r_{\eta,p_{\max},s}\) is ensured by
\begin{equation}
\label{app:eq:multi-subgaussian-dimension}
    d^\ast
    \gtrsim
    \max\left\{
        \frac{K^4}{\lambda_\Delta^2g_\gamma^2},
        \frac{K^2}{\lambda_\Delta g_\gamma}
    \right\}
    \log\left(\frac{2sp_{\max}}{\eta}\right).
\end{equation}
The remaining deterministic softmax requirement is
\[
    \frac{d^\ast g_\gamma}{\tau}
    \ge
    2\log\left(\frac{p_{\max}-1}{\varepsilon}\right).
\]
Thus \(\tau\) appears through the constructed matrices and their effective
gaps, never as a second division inside the softmax.

\subsection{Informative-subspace constructions}
\label{app:subsec:informative-subspaces}

\paragraph{The single-head construction.}
The full-rank condition on \(\bDelta_{\Sigma}\) can be relaxed when its
positive contrast is confined to an informative subspace
\(\mathcal I\subseteq\bbR^d\).  Let
\(\boldsymbol U_{\mathcal I}\in\bbR^{d\times m}\), with
\(m=\dim(\mathcal I)\), have orthonormal columns spanning
\(\mathcal I\), and suppose
\[
    \bDelta_{\mathcal I}
    :=
    \boldsymbol U_{\mathcal I}^{\top}
    \bDelta_{\Sigma}
    \boldsymbol U_{\mathcal I}
    \succ0.
\]
Define the compressed means and whitened mean matrix by
\[
    \bmu_{a,\mathcal I}
    :=
    \boldsymbol U_{\mathcal I}^{\top}\bmu_a,
    \qquad
    \bZ_{\mathcal I}
    :=
    \bDelta_{\mathcal I}^{-1/2}
    [\bmu_{1,\mathcal I},\bmu_{0,\mathcal I}],
\]
and set
\[
    \bP_{\bZ_{\mathcal I}}^\perp
    =
    \bI_m-\bZ_{\mathcal I}\bZ_{\mathcal I}^{\dagger}.
\]
Whitening and projecting inside \(\mathcal I\), followed by lifting back
to \(\bbR^d\), gives
\begin{equation}
\label{app:eq:single-informative-weight}
    \bW_{0,\mathcal I}
    :=
    \frac1{\tau}
    \boldsymbol U_{\mathcal I}
    \bDelta_{\mathcal I}^{-1/2}
    \bP_{\bZ_{\mathcal I}}^\perp
    \bDelta_{\mathcal I}^{-1/2}
    \boldsymbol U_{\mathcal I}^{\top}.
\end{equation}
Writing
\[
    q_{\mathcal I}
    :=
    m-\operatorname{rank}(\bZ_{\mathcal I}),
\]
the same trace calculation as in Theorem~\ref{prop1} gives
\begin{equation}
\label{app:eq:single-informative-properties}
\begin{aligned}
    \bW_{0,\mathcal I}\bmu_0
    &=
    \bW_{0,\mathcal I}\bmu_1
    =
    \mathbf0,\\
    \operatorname{tr}
    (\bW_{0,\mathcal I}\bDelta_{\Sigma})
    &=
    \frac{q_{\mathcal I}}{\tau},\\
    \operatorname{rank}(\bW_{0,\mathcal I})
    &=
    q_{\mathcal I},\\
    \|\bW_{0,\mathcal I}\|_{\mathrm{op}}
    &\le
    \frac1{
        \tau\lambda_{\min}(\bDelta_{\mathcal I})
    }.
\end{aligned}
\end{equation}
Thus no definiteness is needed outside \(\mathcal I\), and a nonzero
gap requires only
\(m>\operatorname{rank}(\bZ_{\mathcal I})\).  The value \(m-2\) is the
special case in which the two compressed whitened means are linearly
independent.

The conditional Hanson--Wright and union-bound calculation in
\eqref{app:eq:single-hanson-wright}--%
\eqref{app:eq:single-subgaussian-radius} carries over after replacing
\(q_0\), \(\lambda_\Delta\), and the covariance scale by their compressed
counterparts.  In particular, the deterministic softmax requirement is
\[
    \frac{q_{\mathcal I}}{\tau}
    \ge
    2\log\left(\frac{p_{\max}-1}{\varepsilon}\right).
\]
Equivalently, \(\tau\) may be chosen of order
\(q_{\mathcal I}/\log\{(p_{\max}-1)/\varepsilon\}\).  This is a
condition on the effective gap already encoded in
\(\bW_{0,\mathcal I}\); the softmax itself remains unscaled.

Because \(\bW_{0,\mathcal I}\) is positive semidefinite of rank
\(q_{\mathcal I}\), it admits a factorization
\[
    \bW_{0,\mathcal I}
    =
    \bL_{\mathcal I}\bL_{\mathcal I}^{\top},
    \qquad
    \bL_{\mathcal I}\in\bbR^{d\times q_{\mathcal I}}.
\]
Thus the same construction can be represented by low-rank query and key
matrices without changing its score gap.

\paragraph{Head-specific informative subspaces.}
The restriction can be carried out separately for every head.  Let
\(\mathcal I_k\subseteq\bbR^d\) have orthonormal basis
\(\boldsymbol U_k\), and suppose
\[
    \bDelta_{k,\mathcal I_k}
    :=
    \boldsymbol U_k^\top\bDelta_k\boldsymbol U_k
    \succ0.
\]
For \(\ell=0,\ldots,s\), put
\(\bmu_{\ell,\mathcal I_k}=\boldsymbol U_k^\top\bmu_\ell\), and define
\[
\begin{aligned}
    \bZ_{k,\mathcal I_k}
    &:=
    \bDelta_{k,\mathcal I_k}^{-1/2}
    [
        \bmu_{0,\mathcal I_k},
        \ldots,
        \bmu_{s,\mathcal I_k}
    ],\\
    \bP_{k,\mathcal I_k}^{\perp}
    &:=
    \bI_{\dim(\mathcal I_k)}
    -
    \bZ_{k,\mathcal I_k}
    \bZ_{k,\mathcal I_k}^{\dagger},\\
    \bD_{k,\mathcal I_k}
    &:=
    \boldsymbol U_k^\top\bD_k\boldsymbol U_k,\\
    \widetilde{\bD}_{k,\mathcal I_k}
    &:=
    \bP_{k,\mathcal I_k}^{\perp}
    \bDelta_{k,\mathcal I_k}^{-1/2}
    \bD_{k,\mathcal I_k}
    \bDelta_{k,\mathcal I_k}^{-1/2}
    \bP_{k,\mathcal I_k}^{\perp}.
\end{aligned}
\]
Choose
\[
    d_k^\ast
    \le
    \dim(\mathcal I_k)
    -
    \operatorname{rank}(\bZ_{k,\mathcal I_k})
\]
and let \(\bGamma_{k,\mathcal I_k}\) collect the corresponding
\(d_k^\ast\) smallest-interference eigenvectors.  The lifted matrix is
\begin{equation}
\label{app:eq:multi-informative-weight}
    \bW_{0,k}^{(\mathcal I_k)}
    :=
    \frac1{\tau}
    \boldsymbol U_k
    \bDelta_{k,\mathcal I_k}^{-1/2}
    \bGamma_{k,\mathcal I_k}
    \bGamma_{k,\mathcal I_k}^{\top}
    \bDelta_{k,\mathcal I_k}^{-1/2}
    \boldsymbol U_k^\top.
\end{equation}
It satisfies
\begin{equation}
\label{app:eq:multi-informative-properties}
\begin{aligned}
    \bW_{0,k}^{(\mathcal I_k)}\bmu_\ell
    &=
    \mathbf0,
    \qquad \ell=0,\ldots,s,\\
    \operatorname{rank}
    \bigl(\bW_{0,k}^{(\mathcal I_k)}\bigr)
    &=
    d_k^\ast,\\
    \operatorname{tr}
    \bigl(
        \bW_{0,k}^{(\mathcal I_k)}
        \bDelta_k
    \bigr)
    &=
    \frac{d_k^\ast}{\tau},\\
    \|\bW_{0,k}^{(\mathcal I_k)}\|_{\mathrm{op}}
    &\le
    \frac1{
        \tau
        \lambda_{\min}(\bDelta_{k,\mathcal I_k})
    }.
\end{aligned}
\end{equation}
If
\(\bar\gamma_{k,\mathcal I_k}^{\ast}\) is the average of the
selected interference eigenvalues, then, for every \(\ell\ne k\),
\begin{equation}
\label{app:eq:multi-informative-signal-gap}
    \operatorname{tr}\left\{
        \bW_{0,k}^{(\mathcal I_k)}
        (\bSigma_k-\bSigma_\ell)
    \right\}
    \ge
    \frac{
        d_k^\ast
        (1-\bar\gamma_{k,\mathcal I_k}^{\ast})
    }{\tau}.
\end{equation}
Hence the multi-head construction and selection proof carry over after
replacing the full-space effective dimensions and spectral quantities by
their head-specific compressed versions.  This produces a low-rank
feasible comparator; reducing the estimation complexity additionally
requires restricting the optimization class itself, as discussed after
the excess-risk proof.

\section{Proof of the excess-risk theorem}
\label{app:sec:excess-risk}

We prove Theorem~\ref{thm:erm_excess} for the empirical risk minimizer
\(\widehat\btheta=(\hbV,\hbW)\) defined in the main text.  The separating matrices
\(\{\bW_{0,k}\}_{k=1}^s\) are used only to exhibit a feasible comparator;
they are not fixed in the jointly optimized estimator
\((\hbV,\hbW)\).  Thus the excess-risk argument depends on the generic
localization theorem, while Theorem~\ref{prop2} gives one concrete way
to construct the comparator.

\subsection{The joint class and a feasible comparator}

Write \(\bW=(\bW_1,\ldots,\bW_s)\) and
\(\btheta=(\bV,\bW)\in\bTheta\), as in the main text, and let
\[
    \mathfrak{F}_{\bTheta}
    :=
    \left\{
        \bX\longmapsto z(\bX;\bV,\bW):
        \btheta=(\bV,\bW)\in\bTheta
    \right\}
\]
be the induced predictor class.  To shorten notation, write
\[
    f_{\btheta}(\bX):=f_{\bV,\bW}(\bX):=z(\bX;\bV,\bW),
    \qquad
    \widehat f:=f_{\widehat\btheta}=f_{\hbV,\hbW}.
\]
The empirical and population risks are the quantities
\(\mathcal L_n\) and \(\mathcal L\) defined in the main text; when
convenient, we write
\[
    \mathcal L_n(f_{\bV,\bW})
    :=
    \mathcal L_n(\bV,\bW),
    \qquad
    \mathcal L(f_{\bV,\bW})
    :=
    \mathcal L(\bV,\bW).
\]

For reference, under the covariance construction in Theorem~\ref{prop2}, each
effective matrix entering the score satisfies
\begin{equation}
\label{app:eq:comparator-op-bound}
    \|\bW_{0,k}\|_{\mathrm{op}}
    \le
    \frac1{\tau}
    \|\bDelta_k^{-1/2}\|_{\mathrm{op}}^2
    \|\bGamma_k\bGamma_k^\top\|_{\mathrm{op}}
    =
    \frac1{
        \tau\lambda_{\min}(\bDelta_k)
    }.
\end{equation}
This display records the scale of the constructed comparator.  It is not a
norm constraint in the high-level theorem, but it becomes a feasibility
requirement for the bounded sufficient class considered below.  More
generally, write
\[
    \bW_0:=(\bW_{0,1},\ldots,\bW_{0,s}),
    \qquad
    \btheta^\dagger:=(\bB,\bW_0)\in\bTheta,
\]
where membership is the comparator assumption in
Theorem~\ref{thm:erm_excess}.
Define the corresponding comparator
\begin{equation}
\label{app:eq:bridge-predictor}
    f^\dagger(\bX^{(i)})
    :=
    f_{\btheta^\dagger}(\bX^{(i)})
    =
    \left\langle
        \bB,
        (\bX^{(i)})^\top\bK_0^{(i)}
    \right\rangle.
\end{equation}
Thus \(\mathcal L(\btheta^\dagger)=\mathcal L(f^\dagger)\).

\subsection{Estimation--approximation decomposition}

Adding and subtracting the comparator risk gives
\begin{align}
\label{app:eq:risk-decomposition}
    \mathcal L(\widehat\btheta)-\mathcal L^\ast
    &=
    \underbrace{
        \mathcal L(\widehat\btheta)-\mathcal L(\btheta^\dagger)
    }_{\mathrm{(I)}\ \text{estimation error}}
    \nonumber\\
    &\quad+
    \underbrace{
        \mathcal L(\btheta^\dagger)-\mathcal L^\ast
    }_{\mathrm{(II)}\ \text{approximation error}}.
\end{align}
Term (II) is controlled by the localization theorem, while term (I) follows
from empirical optimality and the empirical-process condition assumed in
the main text.

\subsection{Approximation error}

We first state the rowwise conditions used to control the prediction error
caused by imperfect localization.  For a scalar random variable \(U\), let
\[
    \|U\|_{\psi_\alpha}
    :=
    \inf\{c>0:\mathbb E\exp(|U|^\alpha/c^\alpha)\le2\}
\]
denote its Orlicz norm.  Write
\[
    \bmu_j^{(i)}
    :=
    \mathbb E\!\left(
        \bX_{j,\cdot}^{(i)}\mid\bA^{(i)}
    \right).
\]
\begin{assumption}[Rowwise regularity]
\label{app:assum:rowwise-regularity}
Conditionally on the selection matrices, suppose that there exist constants
\(K_X,C_\mu>0\), independent of \(n,d,p_{\max}\), and \(s\), such that
\begin{align}
\label{app:eq:rowwise-subgaussian-condition}
    \max_{1\le i\le n}\max_{1\le j\le p^{(i)}}
    \sup_{\|\bu\|_2=1}
    \left\|
        \bu^\top
        \bigl(\bX_{j,\cdot}^{(i)}-\bmu_j^{(i)}\bigr)^\top
    \right\|_{\psi_2}
    &\le K_X,\\
\label{app:eq:row-mean-diameter-condition}
    \max_{1\le i\le n}
    \max_{1\le j,\ell\le p^{(i)}}
    \|\bmu_j^{(i)}-\bmu_\ell^{(i)}\|_2
    &\le C_\mu\sqrt d,
\end{align}
where the Orlicz norm in
\eqref{app:eq:rowwise-subgaussian-condition} is evaluated under the
conditional law given the selection matrices, uniformly over their possible
realizations.  The rows within an observation need not be independent.
Condition~\eqref{app:eq:row-mean-diameter-condition} is stated in terms of
mean differences, rather than absolute mean sizes, because the prediction
error below is invariant to a common translation of all rows.
\end{assumption}

We additionally impose the coefficient-scale condition that, for a constant
\(C_\beta>0\) independent of \(n,d,p_{\max}\), and \(s\),
\begin{equation}
\label{app:eq:coefficient-column-bound}
    \max_{1\le k\le s}\|\bbeta_k\|_2\le C_\beta\sqrt d.
\end{equation}

For each observation, define its row diameter by
\[
    R_i
    :=
    \max_{1\le j,\ell\le p^{(i)}}
    \|\bX_{j,\cdot}^{(i)}-\bX_{\ell,\cdot}^{(i)}\|_2.
\]

\begin{lemma}[Row-diameter bounds]
\label{app:lem:row-diameter}
Under Assumption~\ref{app:assum:rowwise-regularity}, there is a constant \(C>0\),
depending only on \(K_X\) and \(C_\mu\), such that, uniformly in \(i\),
\begin{equation}
\label{app:eq:row-diameter-second-moment}
    \mathbb E R_i^2
    \le
    C\{d+\log p_{\max}\}.
\end{equation}
Moreover, for every event \(\mathcal E_i\) satisfying
\(\Prob(\mathcal E_i)\le\eta\), where \(\eta\in(0,1)\),
\begin{equation}
\label{app:eq:row-diameter-tail-event}
    \mathbb E\!\left(R_i^2\boldsymbol 1_{\mathcal E_i}\right)
    \le
    C\eta
    \left\{
        d+\log\left(\frac{p_{\max}}{\eta}\right)
    \right\}.
\end{equation}
\end{lemma}

\begin{proof}
Put \(\bZ_j^{(i)}=\bX_{j,\cdot}^{(i)}-\bmu_j^{(i)}\).
A standard \(1/2\)-net argument on the unit sphere, followed by a union
bound over the rows, gives constants \(c,C_0>0\), depending only on
\(K_X\), such that
\[
    \Prob\left\{
        \max_{1\le j\le p^{(i)}}\|\bZ_j^{(i)}\|_2^2
        >C_0\bigl(d+\log p_{\max}+t\bigr)
    \right\}
    \le C_0e^{-ct},
    \qquad t\ge0.
\]
This union-bound argument does not require independence among the rows.
By the triangle inequality and
\eqref{app:eq:row-mean-diameter-condition},
\[
    R_i
    \le
    2\max_{1\le j\le p^{(i)}}\|\bZ_j^{(i)}\|_2
    +C_\mu\sqrt d.
\]
Integrating the resulting tail bound proves
\eqref{app:eq:row-diameter-second-moment}.

To prove the second assertion, let
\(u_\eta=C_1\{d+\log(p_{\max}/\eta)\}\), with \(C_1\) sufficiently
large.  Then
\begin{align*}
    \mathbb E\left(R_i^2\boldsymbol 1_{\mathcal E_i}\right)
    &\le
    u_\eta\Prob(\mathcal E_i)
    +\mathbb E\left(R_i^2\boldsymbol 1_{\{R_i^2>u_\eta\}}\right)\\
    &\le
    C\eta
    \left\{
        d+\log\left(\frac{p_{\max}}{\eta}\right)
    \right\},
\end{align*}
where the last step follows by integrating the same exponential tail.
\end{proof}

For \(\varepsilon\in(0,1)\), define the localization event
\[
    \mathcal G_\varepsilon^{(i)}
    :=
    \left\{
        \|\bK_0^{(i)}-\bA^{(i)}\|_\infty
        \le\varepsilon
    \right\}.
\]
\begin{proposition}
\label{app:prop:approximation}
Suppose the conditions of
Theorem~\ref{thm:selection_multi_s1} hold for
\(\varepsilon,\eta\in(0,1)\), and suppose
Assumption~\ref{app:assum:rowwise-regularity} and
Condition~\eqref{app:eq:coefficient-column-bound} hold.  Then
\begin{equation}
\label{app:eq:approximation-bound}
    \mathcal L(f^\dagger)-\mathcal L^\ast
    =
    O\!\left[
        s^2d
        \left\{
            (d+\log p_{\max})\varepsilon^2
            +
            \eta\left(
                d+\log\frac{p_{\max}}{\eta}
            \right)
        \right\}
    \right].
\end{equation}
\end{proposition}

\begin{proof}
By the identity-link individualized regression model described in
Section~\ref{prediction}, the definition
\(\mathcal L=\mathbb E(\mathcal L_n)\), and the conditional mean-zero
property of the noise,
\begin{align}
    \mathcal L(f^\dagger)-\mathcal L^\ast
    &=
    \frac1n\sum_{i=1}^n
    \mathbb{E}\left[
        \left\langle
            \bB,
            (\bX^{(i)})^\top
            (\bA^{(i)}-\bK_0^{(i)})
        \right\rangle^2
    \right].
\end{align}
For head \(k\), let \(j_{i,k}^\star\) be its target row.  Since
\(\boldsymbol\kappa_{k,0}^{(i)}\) belongs to the probability simplex,
\begin{align}
\label{app:eq:translation-invariant-error}
    &(\bX^{(i)})^\top
    \left(
        \boldsymbol\kappa_{k,0}^{(i)}
        -\mathbf e_{j_{i,k}^\star}
    \right)\nonumber\\
    &\qquad=
    \sum_{j\ne j_{i,k}^\star}
    \kappa_{j,k,0}^{(i)}
    \left(
        \bX_{j,\cdot}^{(i)}
        -\bX_{j_{i,k}^\star,\cdot}^{(i)}
    \right)^\top.
\end{align}
On \(\mathcal G_\varepsilon^{(i)}\),
\[
    \sum_{j\ne j_{i,k}^\star}\kappa_{j,k,0}^{(i)}
    =1-\kappa_{j_{i,k}^\star,k,0}^{(i)}
    \le\varepsilon,
\]
whereas the same sum is always at most one.  Consequently,
Condition~\eqref{app:eq:coefficient-column-bound} and
\eqref{app:eq:translation-invariant-error} give
\begin{align*}
    \left|
        \left\langle
            \bB,
            (\bX^{(i)})^\top(\bA^{(i)}-\bK_0^{(i)})
        \right\rangle
    \right|
    &\le C_\beta s\sqrt d\,\varepsilon R_i,
    &&\text{on }\mathcal G_\varepsilon^{(i)},\\
    \left|
        \left\langle
            \bB,
            (\bX^{(i)})^\top(\bA^{(i)}-\bK_0^{(i)})
        \right\rangle
    \right|
    &\le C_\beta s\sqrt d\,R_i,
    &&\text{always}.
\end{align*}
Theorem~\ref{thm:selection_multi_s1} gives
\(\Prob\{(\mathcal G_\varepsilon^{(i)})^c\}\le\eta\).  Splitting the
expectation in the preceding risk identity over
\(\mathcal G_\varepsilon^{(i)}\) and its complement, and applying
Lemma~\ref{app:lem:row-diameter}, therefore gives
\begin{align*}
    \mathcal L(f^\dagger)-\mathcal L^\ast
    &\le
    \frac{C_\beta^2s^2d}{n}
    \sum_{i=1}^n
    \left[
        \varepsilon^2\mathbb E R_i^2
        +\mathbb E\left\{
            R_i^2\boldsymbol 1_{(\mathcal G_\varepsilon^{(i)})^c}
        \right\}
    \right]\\
    &\le
    Cs^2d
    \left\{
        (d+\log p_{\max})\varepsilon^2
        +
        \eta\left(
            d+\log\frac{p_{\max}}{\eta}
        \right)
    \right\},
\end{align*}
which proves Equation~\eqref{app:eq:approximation-bound}.
\end{proof}

The bound in Proposition~\ref{app:prop:approximation} is deterministic:
the selector event concerns a fresh population draw inside the
expectation and is integrated out.  Accordingly, the approximation term
does not consume any probability in the final \(O_{\Prob}\) statement.

\subsection{Estimation error for the jointly optimized class}

Set
\[
    \Delta_n
    :=
    \sup_{\btheta\in\bTheta}
    |\mathcal L_n(\btheta)-\mathcal L(\btheta)|.
\]
Because \(\btheta^\dagger\in\bTheta\), empirical optimality gives
\begin{align*}
    \mathcal L(\widehat\btheta)-\mathcal L(\btheta^\dagger)
    &=
    \{\mathcal L(\widehat\btheta)-\mathcal L_n(\widehat\btheta)\}
    +\{\mathcal L_n(\widehat\btheta)-\mathcal L_n(\btheta^\dagger)\}
    +\{\mathcal L_n(\btheta^\dagger)-\mathcal L(\btheta^\dagger)\}\\
    &\le 2\Delta_n.
\end{align*}
Condition~\eqref{eq:uniform-loss-condition} therefore yields
\begin{equation}
\label{app:eq:estimation-bound}
    \mathcal L(\widehat\btheta)-\mathcal L(\btheta^\dagger)
    =
    O_{\Prob}\left(
        \sqrt{\frac{D\log n}{n}}
    \right).
\end{equation}
This is the only empirical-process step needed for the theorem.

\subsubsection{A norm-controlled sufficient condition}
\label{app:subsec:empirical-process}

For \(R_{1,n},R_{2,n}>0\), define
\[
    \bTheta_n(R_{1,n},R_{2,n})
    :=
    \left\{
        (\bV,\bW):
        \|\bV\|_F\le R_{1,n},\quad
        \max_{1\le k\le s}\|\bW_k\|_{\mathrm{op}}\le R_{2,n}
    \right\}.
\]
One concrete specialization of the high-level theorem is obtained by taking
\(\bTheta=\bTheta_n(R_{1,n},R_{2,n})\).  For this specialization, the conditions
below verify Condition~\eqref{eq:uniform-loss-condition}.  These norm
constraints are not part of the abstract statement of
Theorem~\ref{thm:erm_excess}.
For this specialization, the comparator assumption is ensured by
\begin{equation}
\label{app:eq:comparator-feasibility}
    R_{1,n}\ge\|\bB\|_F,
    \qquad
    R_{2,n}\ge
    \max_{1\le k\le s}\|\bW_{0,k}\|_{\mathrm{op}}.
\end{equation}
Equip it with
\[
    d_{\bTheta}(\btheta,\btheta')
    :=
    \|\bV-\bV'\|_F
    +
    \left\{
        \sum_{k=1}^s
        \|\bW_k-\bW_k'\|_{\mathrm{op}}^2
    \right\}^{1/2}.
\]
Because every attention column lies in the probability simplex,
\begin{equation}
\label{app:eq:predictor-envelope}
    |f_{\btheta}(\bX)|
    \le
    \sqrt{s}R_{1,n}\|\bX\|_{\mathrm{op}}.
\end{equation}
Moreover, the \(\ell_\infty\)-to-\(\ell_1\) Lipschitz property of softmax
gives, for \(\btheta,\btheta'\in\bTheta_n(R_{1,n},R_{2,n})\),
\begin{align}
\label{app:eq:predictor-lipschitz}
    |f_{\btheta}(\bX)-f_{\btheta'}(\bX)|
    &\le
    \sqrt{s}\|\bX\|_{\mathrm{op}}\|\bV-\bV'\|_F
    \nonumber\\
    &\quad+
    2R_{1,n}\|\bX\|_{\mathrm{op}}^3
    \left\{
        \sum_{k=1}^s
        \|\bW_k-\bW_k'\|_{\mathrm{op}}^2
    \right\}^{1/2}.
\end{align}
In particular, neither the predictor envelope nor the Lipschitz factor in
\eqref{app:eq:predictor-lipschitz} depends polynomially on \(R_{2,n}\).

For \(Z_i=(\bX^{(i)},y^{(i)})\), put
\begin{align*}
    H_{n,i}
    &:=
    \left(
        |y^{(i)}|
        +\sqrt{s}R_{1,n}\|\bX^{(i)}\|_{\mathrm{op}}
    \right)^2,\\
    J_{n,i}
    &:=
    2\left(
        |y^{(i)}|
        +\sqrt{s}R_{1,n}\|\bX^{(i)}\|_{\mathrm{op}}
    \right)
    \left(
        \sqrt{s}\|\bX^{(i)}\|_{\mathrm{op}}
        +2R_{1,n}\|\bX^{(i)}\|_{\mathrm{op}}^3
    \right).
\end{align*}
Here \(H_{n,i}\) is an envelope for the squared loss and \(J_{n,i}\) is a
Lipschitz envelope with respect to \(d_{\bTheta}\).

\begin{proposition}
\label{app:prop:uniform-loss-sufficient}
Suppose the observations are independent,
\(R_{1,n}=O(1)\), \(D\log n=o(n)\), and
\begin{equation}
\label{app:eq:empirical-process-envelope}
    \max_{1\le i\le n}\|H_{n,i}\|_{\psi_1}=O(1),
    \qquad
    \max_{1\le i\le n}\|J_{n,i}\|_{L_2}=O(1).
\end{equation}
If
\begin{equation}
\label{app:eq:radius-entropy-growth}
    \log\!\left\{
        2+\sqrt{sd}\,R_{2,n}
    \right\}
    =O(\log n),
\end{equation}
then, with \(\bTheta=\bTheta_n(R_{1,n},R_{2,n})\),
Condition~\eqref{eq:uniform-loss-condition} holds.
In particular, \(R_{2,n}=O(\log n)\) satisfies
\eqref{app:eq:radius-entropy-growth}.
\end{proposition}

The last statement concerns the entropy condition only; comparator
feasibility is checked separately below.
This proposition verifies the main-text condition only for the displayed
class.  If estimation is carried out over a larger \(\bTheta\),
Condition~\eqref{eq:uniform-loss-condition}, or the weaker increment condition
given below, must instead be established for that estimator.

\begin{proof}
The Euclidean covering bound, together with
\(\|\bW_k\|_F\le\sqrt d\|\bW_k\|_{\mathrm{op}}\), gives
\begin{equation}
\label{app:eq:bounded-parameter-entropy}
    \log N\{u,\bTheta_n(R_{1,n},R_{2,n}),d_{\bTheta}\}
    \le
    D
    \log\left[
        \frac{
            C\{1+R_{1,n}+\sqrt{sd}\,R_{2,n}\}
        }{u}
    \right]
\end{equation}
for \(u\) below the diameter of the class.  Write
\(\ell_{\btheta}(Z_i)=\{y^{(i)}-f_{\btheta}(\bX^{(i)})\}^2\), and, for a
possibly non-identically distributed array, use
\[
    \mathbb P_n g:=\frac1n\sum_{i=1}^n g(Z_i),
    \qquad
    \mathbb P g:=\frac1n\sum_{i=1}^n\mathbb E g(Z_i).
\]
The envelope definitions imply
\[
    |\ell_{\btheta}(Z_i)-\ell_{\btheta'}(Z_i)|
    \le
    J_{n,i}d_{\bTheta}(\btheta,\btheta'),
\]
by \eqref{app:eq:predictor-envelope}--%
\eqref{app:eq:predictor-lipschitz}, and
\(0\le\ell_{\btheta}(Z_i)\le H_{n,i}\).

Let \(\mathcal N_n\) be an \(n^{-1}\)-net of the parameter class under
\(d_{\bTheta}\).  Conditions~\eqref{app:eq:bounded-parameter-entropy}
and~\eqref{app:eq:radius-entropy-growth}, together with
\(R_{1,n}=O(1)\), give
\[
    \log|\mathcal N_n|=O(D\log n).
\]
For each net point, the centered losses are uniformly sub-exponential by
\eqref{app:eq:empirical-process-envelope}.  Bernstein's inequality for sums
of independent, non-identically distributed sub-exponential variables,
followed by a union bound over \(\mathcal N_n\), therefore gives
\[
    \max_{\bm{\vartheta}\in\mathcal N_n}
    |(\mathbb P_n-\mathbb P)\ell_{\bm{\vartheta}}|
    =
    O_{\Prob}\left[
        \sqrt{\frac{D\log n}{n}}
        +\frac{D\log n}{n}
    \right].
\]
For any \(\btheta\), choose \(\pi_n\btheta\in\mathcal N_n\) with
\(d_{\bTheta}(\btheta,\pi_n\btheta)\le n^{-1}\).  Then
\begin{align*}
    |(\mathbb P_n-\mathbb P)
        (\ell_{\btheta}-\ell_{\pi_n\btheta})|
    &\le
    \frac1n\left\{
        \frac1n\sum_{i=1}^n J_{n,i}
        +\frac1n\sum_{i=1}^n\mathbb E J_{n,i}
    \right\}\\
    &=O_{\Prob}(n^{-1}),
\end{align*}
where the last equality follows from the uniform \(L_2\) bound in
\eqref{app:eq:empirical-process-envelope}.  Since
\(D\log n=o(n)\), combining the last two displays proves
\[
    \sup_{\btheta\in\bTheta_n(R_{1,n},R_{2,n})}
    |\mathcal L_n(\btheta)-\mathcal L(\btheta)|
    =O_{\Prob}\left(\sqrt{\frac{D\log n}{n}}\right).
\]
The dependence on \(R_{2,n}\) is logarithmic because softmax keeps the
predictor and loss envelopes bounded uniformly over the query--key matrices;
\(R_{2,n}\) enters only through the size of \(\mathcal N_n\).
\end{proof}

To make the envelope requirement explicit, write
\[
    K_{X,n}
    :=
    \max_{1\le i\le n}
    \left\|\|\bX^{(i)}\|_{\mathrm{op}}\right\|_{\psi_2}.
\]
Conditional Gaussianity and the oracle regression model imply that
Condition~\eqref{app:eq:empirical-process-envelope} holds whenever
\begin{align*}
    &\left[
        \sigma
        +\sqrt{s}\{\|\bB\|_F+R_{1,n}\}K_{X,n}
    \right]^2
    =O(1),\\
    &\left[
        \sigma
        +\sqrt{s}\{\|\bB\|_F+R_{1,n}\}K_{X,n}
    \right]
    \left[
        \sqrt{s}K_{X,n}+R_{1,n}K_{X,n}^3
    \right]
    =O(1).
\end{align*}
Thus the envelope condition follows, for example, under uniformly controlled
sub-Gaussian design and noise after the design normalization needed to keep
these operator-norm scales bounded.
This full-design envelope requirement is separate from, and generally
stronger than, Assumption~\ref{app:assum:rowwise-regularity}, which is used for
the approximation term.  In particular, this rowwise assumption alone does
not make the empirical-process
constants uniform when \(p_{\max}\), \(s\), or the operator norm of the full
design grows.  Also, Condition~\eqref{app:eq:coefficient-column-bound} implies
only \(\|\bB\|_F=O(\sqrt{sd})\).  Since comparator feasibility requires
\(R_{1,n}\ge\|\bB\|_F\), Condition~\eqref{app:eq:coefficient-column-bound}
alone guarantees compatibility with \(R_{1,n}=O(1)\) only when \(sd=O(1)\).
For growing \(s\) or \(d\), an additional bound on the total coefficient
energy is needed to apply this sufficient result.  In other growth regimes,
Condition~\eqref{eq:uniform-loss-condition} remains a separate high-level
assumption unless a suitably scaled or localized empirical-process argument
is supplied.

In fact, the preceding ERM comparison only needs the weaker increment bound
\begin{equation}
\label{app:eq:centered-increment-sufficient}
    \left|
        \{\mathcal L_n(\widehat\btheta)-\mathcal L(\widehat\btheta)\}
        -
        \{\mathcal L_n(\btheta^\dagger)-\mathcal L(\btheta^\dagger)\}
    \right|
    =
    O_{\Prob}\left(
        \sqrt{\frac{D\log n}{n}}
    \right).
\end{equation}
Thus a localized empirical-process or stability argument establishing
\eqref{app:eq:centered-increment-sufficient} may replace the stronger global
condition in the main text without changing the excess-risk conclusion.

\subsection{Completion of the proof}

Let \(\varepsilon_n,\eta_n\in(0,1)\) be the sequences in
Theorem~\ref{thm:erm_excess}.  The score-gap condition
\eqref{eq:score-gap-condition-multi}, evaluated at
\((\varepsilon,\eta)=(\varepsilon_n,\eta_n)\), is exactly the hypothesis
needed to apply Theorem~\ref{thm:selection_multi_s1} at
\((\varepsilon,\eta)=(\varepsilon_n,\eta_n)\).
Proposition~\ref{app:prop:approximation} and
Equation~\eqref{app:eq:estimation-bound} therefore give
\begin{equation}
\label{app:eq:final-excess-rate}
    \mathcal L(\widehat\btheta)-\mathcal L^\ast
    =
    O_{\Prob}\left[
        \sqrt{\frac{D\log n}{n}}
        +
        s^2d
        \left\{
            (d+\log p_{\max})\varepsilon_n^2
            +
            \eta_n\left(
                d+\log\frac{p_{\max}}{\eta_n}
            \right)
        \right\}
    \right].
\end{equation}
This is exactly the conclusion~\eqref{eq:scale-explicit-excess} and completes
the proof of Theorem~\ref{thm:erm_excess}.

For completeness, we give choices that balance the two localization terms
with the estimation term.  Put
\[
    a_n:=\sqrt{\frac{sd^2\log n}{n}},
    \qquad
    \ell_n:=d+\log p_{\max},
    \qquad
    h_n:=d+\log\left(\frac{e p_{\max}s^2d}{a_n}\right).
\]
When \(a_n\to0\), the choices
\[
    \varepsilon_n
    =
    \min\left\{
        \frac12,
        \left(\frac{a_n}{s^2d\ell_n}\right)^{1/2}
    \right\},
    \qquad
    \eta_n
    =
    \min\left\{
        \frac12,
        \frac{a_n}{2s^2dh_n}
    \right\}
\]
belong to \((0,1)\), and the truncations at \(1/2\) are eventually inactive.
For the untruncated choices,
\[
    s^2d(d+\log p_{\max})\varepsilon_n^2=a_n.
\]
Moreover,
\[
    d+\log\left(\frac{p_{\max}}{\eta_n}\right)
    =h_n-1+\log(2h_n)
    \le2h_n,
\]
where the last inequality holds because \(h_n\ge1\).  It follows that
\[
    s^2d\eta_n
    \left\{
        d+\log\left(\frac{p_{\max}}{\eta_n}\right)
    \right\}
    \le a_n.
\]
Hence both localization terms in
\eqref{app:eq:final-excess-rate} are \(O(a_n)\).  If, in addition,
\(\ell_n\asymp h_n\asymp d\), then the choices simplify to
\[
    \varepsilon_n
    \asymp
    d^{-1/2}s^{-3/4}
    \left(\frac{\log n}{n}\right)^{1/4},
    \qquad
    \eta_n
    \asymp
    d^{-1}s^{-3/2}
    \left(\frac{\log n}{n}\right)^{1/2}.
\]
Since \(D=sd+sd^2\asymp sd^2\), the excess risk is then
\(O_{\Prob}(a_n)\).  Obtaining this rate additionally requires the score-gap
condition, the membership \(\btheta^\dagger\in\bTheta\), the rowwise and
coefficient-scale conditions used in Proposition~\ref{app:prop:approximation},
and Condition~\eqref{eq:uniform-loss-condition} to remain valid along the
resulting sequences.

\subsection{Discussion: structured query--key matrices}
\label{app:subsec:structured-risk}

To illustrate how structure can change the estimation term, consider
\[
    \bTheta_{q,n}
    :=
    \left\{
        \btheta=(\bV,\bW)\in\bTheta_n(R_{1,n},R_{2,n}):
        \operatorname{rank}(\bW_k)\le q
        \text{ for every }k
    \right\},
    \qquad q\ll d.
\]
A rank-\(q\), \(d\times d\) matrix admits a factorization with
\(q(2d-q)=O(qd)\) effective parameters, rather than \(d^2\).  This motivates
replacing
\[
    D=sd+sd^2
    \quad\text{by}\quad
    D_{\mathrm{low}}
    =
    sd+O(sqd).
\]
Let \(\widehat\btheta_q\) be an empirical risk minimizer over
\(\bTheta_{q,n}\).  Suppose that \(\bTheta_{q,n}\) contains a separating
comparator and that the envelope conditions in
\eqref{app:eq:empirical-process-envelope} hold.  The same covering argument,
with \(D\) replaced by \(D_{\mathrm{low}}\), verifies the corresponding uniform
loss condition whenever
\(D_{\mathrm{low}}\log n=o(n)\) and
\(\log\{2+\sqrt{sq}\,R_{2,n}\}=O(\log n)\).  Under
Assumption~\ref{app:assum:rowwise-regularity} and
Condition~\eqref{app:eq:coefficient-column-bound}, if the selection conditions hold for
\((\varepsilon_n,\eta_n)\) and
\[
    s^2d\left\{
        (d+\log p_{\max})\varepsilon_n^2
        +\eta_n\left(
            d+\log\frac{p_{\max}}{\eta_n}
        \right)
    \right\}
    =
    O\left(
        \sqrt{\frac{D_{\mathrm{low}}\log n}{n}}
    \right),
\]
then the same proof gives
\[
    \mathcal L(\widehat\btheta_q)-\mathcal L^\ast
    =
    O_{\Prob}\left(
        \sqrt{
            \frac{D_{\mathrm{low}}\log n}{n}
        }
    \right).
\]
The matrices in
\eqref{app:eq:single-informative-weight} and
\eqref{app:eq:multi-informative-weight} give such low-rank comparators
when their effective ranks do not exceed \(q\).

The rank restriction can be implemented through
\(\bW_k=\bL_k\bR_k^\top\), or, for a positive-semidefinite matrix,
\(\bW_k=\bL_k\bL_k^\top\).  Alternatively, one may take an unrestricted
gradient step and then retain the largest \(q\) singular values.  These
parameterizations remain nonconvex.  Importantly, rank restriction and
parameter counting alone do not prove the displayed empirical-process rate;
the loss-class envelope conditions must still be verified.

\subsection{Dimension-growth conditions}
\label{app:subsec:growth-conditions}

Under Condition~\eqref{eq:uniform-loss-condition}, the estimation component
tends to zero if
\begin{equation}
\label{app:eq:estimation-growth}
    D\log n=(sd+sd^2)\log n=o(n),
\end{equation}
or, equivalently at the stated order, if \(sd^2\log n=o(n)\).  Risk
consistency additionally requires
\begin{equation}
\label{app:eq:localization-growth}
    s^2d
    \left\{
        (d+\log p_{\max})\varepsilon_n^2
        +
        \eta_n\left(
            d+\log\frac{p_{\max}}{\eta_n}
        \right)
    \right\}
    =o(1).
\end{equation}

We next separate the selection, rowwise-design, coefficient-scale, and
empirical-process requirements for the covariance construction.  Allowing the temperature to
vary with \(n\), write it as \(\tau_n\), and let
\[
    d_n^\ast:=\lfloor d/s\rfloor,
    \qquad
    \lambda_{\Delta,n}
    :=\min_{1\le k\le s}\lambda_{\min}(\bDelta_k),
    \qquad
    g_{\gamma,n}:=1-\bar\gamma_{\max}^\ast.
\]
Assume first that the sub-Gaussian and spectral factors
in~\eqref{app:eq:multi-subgaussian-dimension} are uniformly bounded.  Since
the common factor \(1/\tau_n\) cancels between the constructed gap and its
fluctuation radius, the stochastic part of the score-gap condition has the
representative requirement
\begin{equation}
\label{app:eq:growth-fluctuation}
    \frac ds
    \gtrsim
    \log\left(\frac{sp_{\max}}{\eta_n}\right).
\end{equation}
The deterministic softmax part remains
\begin{equation}
\label{app:eq:growth-softmax}
    \frac{d_n^\ast g_{\gamma,n}}{\tau_n}
    \ge
    2\log\left(\frac{p_{\max}-1}{\varepsilon_n}\right).
\end{equation}
Here \(\tau_n\) occurs only through the constructed matrices
\(\bW_{0,k}\); the softmax itself is not divided by \(\tau_n\) a second time.

The covariance construction also gives
\[
    \max_{1\le k\le s}\|\bW_{0,k}\|_{\mathrm{op}}
    \le
    \frac1{\tau_n\lambda_{\Delta,n}}.
\]
This norm bound describes how the comparator scale changes with
\(\tau_n\).  The high-level theorem only assumes
\(\btheta^\dagger\in\bTheta\).  For the bounded sufficient class in
Section~\ref{app:subsec:empirical-process}, a sufficient compatibility
interval is
\begin{equation}
\label{app:eq:temperature-radius-interval}
    \frac1{R_{2,n}\lambda_{\Delta,n}}
    \le
    \tau_n
    \le
    \frac{d_n^\ast g_{\gamma,n}}
    {2\log\{(p_{\max}-1)/\varepsilon_n\}}.
\end{equation}
The interval is nonempty whenever
\begin{equation}
\label{app:eq:temperature-radius-compatibility}
    R_{2,n}\lambda_{\Delta,n}d_n^\ast g_{\gamma,n}
    \ge
    2\log\left(\frac{p_{\max}-1}{\varepsilon_n}\right).
\end{equation}
The remaining comparator requirement is
\(R_{1,n}\ge\|\bB\|_F\).  Thus taking
\(R_{2,n}\asymp\log n\) preserves the same empirical-process order and is
comparator-feasible provided
\(\tau_n\lambda_{\Delta,n}\gtrsim1/\log n\), the upper bound in
\eqref{app:eq:temperature-radius-interval} holds, and the envelope conditions
in~\eqref{app:eq:empirical-process-envelope} remain uniform.

For the clean-rate choices following~\eqref{app:eq:final-excess-rate},
combining~\eqref{app:eq:growth-fluctuation}
and~\eqref{app:eq:growth-softmax} gives, when the stochastic spectral factors
are uniform and \(\tau_n/g_{\gamma,n}=O(1)\), the convenient sufficient order
\begin{equation}
\label{app:eq:joint-selection-growth}
    \frac ds
    \gtrsim
    \max\left\{
        \log\left(
            \frac{s^3p_{\max}d h_n}{a_n}
        \right),
        \log\left(
            p_{\max}s\sqrt{\frac{d\ell_n}{a_n}}
        \right)
    \right\}.
\end{equation}
If \(\ell_n\asymp h_n\asymp d\), then
\eqref{app:eq:joint-selection-growth} becomes
\begin{equation}
\label{app:eq:joint-selection-growth-simplified}
    \frac ds
    \gtrsim
    \max\left\{
        \log\left(
            p_{\max}s^{5/2}d\sqrt{\frac{n}{\log n}}
        \right),
        \log\left(
            p_{\max}s^{3/4}d^{1/2}
            \left(\frac{n}{\log n}\right)^{1/4}
        \right)
    \right\}.
\end{equation}
The construction requires
\[
    1\le d_n^\ast\le\min_{1\le k\le s}d_k,
    \qquad d_k:=d-\operatorname{rank}(\bZ_k),
\]
where the lower bound follows from \(d\ge s\).  If also
\(s\log s=O(d)\), \(d/s\gtrsim\log n\), and
\(\ell_n\asymp h_n\asymp d\), then, for a sufficiently small constant
\(c>0\), the selection condition permits the sufficient envelope
\begin{equation}
\label{app:eq:representative-p-growth}
    p_{\max}
    =
    O\left(
        \frac{\exp(c\,d/s)}{s^{5/2}d}
        \sqrt{\frac{\log n}{n}}
    \right).
\end{equation}
This is a selection-side statement, not by itself a joint risk-consistency
regime.  Assumption~\ref{app:assum:rowwise-regularity},
Condition~\eqref{app:eq:coefficient-column-bound}, the localization-growth requirement
\eqref{app:eq:localization-growth}, comparator feasibility, and the assumed
empirical-process bound must also hold.  With the balancing choices above,
the left-hand side of \eqref{app:eq:localization-growth} is \(O(a_n)\), so it
vanishes whenever \(a_n\to0\).

Finally, the illustrative estimation-compatible upper envelope
\[
    d
    =
    O\left(
        \frac{\sqrt{n/s}}{\log n}
    \right)
\]
makes \(sd^2\log n/n=O(1/\log n)\).  It controls the estimation term but
does not replace the lower-dimensional selection requirement
in~\eqref{app:eq:joint-selection-growth}, the rowwise-design and
coefficient-scale conditions, the comparator-membership condition, the
bounded-class compatibility condition
in~\eqref{app:eq:temperature-radius-compatibility} when that sufficient
specialization is used, or the assumed empirical-process bound.  More
generally, the estimation component tends to zero whenever
\(d=o\{\sqrt{n/(s\log n)}\}\).

\section{Discussion of the spectral-overlap condition}
\label{app:sec:overlap-discussion}

The condition \(\bar\gamma_{\max}^\ast<1\) in
Theorem~\ref{prop2} restricts the relative geometry of the excess
covariance matrices.  Positive definiteness of each individual
\(\bDelta_k\) does not by itself imply this spectral-overlap condition.
We first give an analytic family that satisfies the condition, and then
use a shared-and-class-specific path to separate spectral strength from
eigenspace geometry.

\subsection{An analytic sufficient family}
\label{app:subsec:analytic-overlap}

Let
\(\mathcal U_1,\ldots,\mathcal U_s\) be mutually orthogonal subspaces of
dimension \(d^\ast\), and let \(\bP_k\) denote the orthogonal projector
onto \(\mathcal U_k\).  For constants \(a,b>0\), define
\begin{equation}
\label{app:eq:analytic-overlap-family}
    \bDelta_k
    =
    a\bP_k+b\bI_d,
    \qquad k\in[s].
\end{equation}
Every \(\bDelta_k\) is positive definite.  Suppose in addition that
\(\operatorname{col}(\bZ_k)\perp\mathcal U_k\); this holds, in
particular, when all class means are zero.  For a target class \(k\),
\[
    \bD_k
    =
    (s-1)b\bI_d
    +
    a\sum_{\ell\ne k}\bP_\ell.
\]
Before removing the whitened means, the matrix
\(\bDelta_k^{-1/2}\bD_k\bDelta_k^{-1/2}\) has eigenvalue
\[
    \gamma_{\mathrm{low}}
    =
    \frac{(s-1)b}{a+b}
\]
on \(\mathcal U_k\), eigenvalue
\(s-1+a/b\) on every \(\mathcal U_\ell\), \(\ell\ne k\), and eigenvalue
\(s-1\) on the remaining subspace.  The mean projection preserves
\(\mathcal U_k\), so the \(d^\ast\) eigenvalues on this subspace remain
equal to \(\gamma_{\mathrm{low}}\).  If
\begin{equation}
\label{app:eq:analytic-overlap-condition}
    a>(s-2)b,
\end{equation}
then \(\gamma_{\mathrm{low}}<1\), whereas every other spectral level is at least one.
Consequently, the \(d^\ast\) smallest positive eigenvalues of
\(\widetilde{\bD}_k\) all equal \(\gamma_{\mathrm{low}}\), and
\[
    \bar\gamma_k^\ast
    =
    \frac{(s-1)b}{a+b}
    <1
    \qquad\text{for every }k.
\]
This is a rotationally invariant class of strictly positive-definite
covariance contrasts satisfying the spectral-overlap condition in
Theorem~\ref{prop2}.
The strict inequality is stable under sufficiently small perturbations
of the contrasts and of the compatible mean subspaces.

\subsection{A shared-and-class-specific construction}
\label{app:subsec:block-construction}

Suppose for simplicity that \(d\) is divisible by \(s\), and put
\[
    r:=\frac{d}{s}=d^\ast.
\]
Fix a signal strength \(c_{\mathrm{sig}}>0\) and a small interference
level \(c_{\epsilon}>0\).  The latter acts as an isotropic ridge and
guarantees strict positive definiteness.  Let
\(r_{\mathrm{cs}}\in\{0,1,\ldots,r\}\) denote the class-specific
dimension, and
decompose \(\mathbb{R}^d\) orthogonally as
\begin{equation}
\label{app:eq:block-decomposition}
    \mathbb{R}^d
    =
    \mathcal U_{\mathrm{sh}}
    \oplus
    \mathcal U_1
    \oplus\cdots\oplus
    \mathcal U_s
    \oplus
    \mathcal U_{\mathrm{rem}},
\end{equation}
where
\begin{equation}
\label{app:eq:block-dimensions}
    \dim(\mathcal U_{\mathrm{sh}})=r-r_{\mathrm{cs}},
    \qquad
    \dim(\mathcal U_k)=r_{\mathrm{cs}},
    \qquad
    \dim(\mathcal U_{\mathrm{rem}})
    =(s-1)(r-r_{\mathrm{cs}}).
\end{equation}
Indeed,
\[
    (r-r_{\mathrm{cs}})
    +sr_{\mathrm{cs}}
    +(s-1)(r-r_{\mathrm{cs}})
    =sr=d.
\]
The subspace \(\mathcal U_{\mathrm{sh}}\) contains directions enhanced
for every class, whereas \(\mathcal U_k\) contains directions enhanced
only for class \(k\).  Let \(\bP_{\mathrm{sh}}\) and \(\bP_k\) be the
orthogonal projectors onto these two subspaces.

For every \(k\), choose a matrix
\(\bH_k\in\mathbb{R}^{d\times r}\) whose columns form an orthonormal basis of
\(\mathcal U_{\mathrm{sh}}\oplus\mathcal U_k\).  Thus
\[
    \bH_k^\top\bH_k=\bI_r,
    \qquad
    \bH_k\bH_k^\top
    =
    \bP_{\mathrm{sh}}+\bP_k.
\]
Define
\begin{equation}
\label{app:eq:block-delta}
    \bDelta_k
    :=
    c_{\mathrm{sig}}\bH_k\bH_k^\top+c_{\epsilon}\bI_d
    =
    c_{\mathrm{sig}}(\bP_{\mathrm{sh}}+\bP_k)
    +c_{\epsilon}\bI_d.
\end{equation}
This basis matrix is denoted by \(\bH_k\), because it defines the covariance construction and
is not the eigenvector matrix used in the selection proof.

A coordinate implementation of \eqref{app:eq:block-decomposition} may
take the first \(r-r_{\mathrm{cs}}\) standard basis vectors to span
\(\mathcal U_{\mathrm{sh}}\), the next \(sr_{\mathrm{cs}}\) coordinates
to form \(s\)
successive private blocks, and the remaining coordinates to span
\(\mathcal U_{\mathrm{rem}}\).  This choice is only a convenient basis:
applying a common orthogonal rotation to every subspace and every mean
vector leaves all spectral quantities below unchanged.

For every \(r_{\mathrm{cs}}\) and every \(k\),
\(\bP_{\mathrm{sh}}+\bP_k\) is a rank-\(r\) orthogonal projector.
Consequently, each \(\bDelta_k\) has the fixed spectrum
\begin{equation}
\label{app:eq:block-fixed-spectrum}
    c_{\mathrm{sig}}+c_{\epsilon}
    \quad\text{with multiplicity }r,
    \qquad
    c_{\epsilon}
    \quad\text{with multiplicity }d-r.
\end{equation}
Varying \(r_{\mathrm{cs}}\) therefore changes neither the rank of the spike, its
strength, the condition number, nor the spectrum of any individual
excess covariance.  It changes only the proportion of enhanced
directions that are class-specific.

\subsection{Spectrum before removing the whitened means}
\label{app:subsec:preprojection}

For target class \(k\), the aggregate excess covariance of the other
classes is
\begin{align}
\label{app:eq:block-Dk}
    \bD_k
    &=
    \sum_{\ell\in[s]\setminus\{k\}}\bDelta_\ell \nonumber\\
    &=
    (s-1)c_{\epsilon}\bI_d
    +
    c_{\mathrm{sig}}\left\{
        (s-1)\bP_{\mathrm{sh}}
        +
        \sum_{\ell\in[s]\setminus\{k\}}\bP_\ell
    \right\}.
\end{align}
Define the pre-projection whitened interference matrix
\begin{equation}
\label{app:eq:preprojection-matrix}
    \bT_k
    :=
    \bDelta_k^{-1/2}\bD_k\bDelta_k^{-1/2}.
\end{equation}
The notation \(\bT_k\) distinguishes this matrix from the selection
matrix \(\bA^{(i)}\).

All projectors in \eqref{app:eq:block-Dk} are mutually orthogonal, so
\(\bT_k\) acts as a scalar on each component of
\eqref{app:eq:block-decomposition}.  Its spectral decomposition is
summarized below.
\begin{center}
\begin{tabular}{lccc}
\toprule
subspace
& dimension
& eigenvalue of \(\bDelta_k\)
& eigenvalue of \(\bT_k\) \\
\midrule
shared \(\mathcal U_{\mathrm{sh}}\)
& \(r-r_{\mathrm{cs}}\)
& \(c_{\mathrm{sig}}+c_{\epsilon}\)
& \(s-1\) \\
target-specific \(\mathcal U_k\)
& \(r_{\mathrm{cs}}\)
& \(c_{\mathrm{sig}}+c_{\epsilon}\)
& \(\displaystyle
   \frac{(s-1)c_{\epsilon}}{c_{\mathrm{sig}}+c_{\epsilon}}\) \\
other class-specific
\(\bigoplus_{\ell\ne k}\mathcal U_\ell\)
& \((s-1)r_{\mathrm{cs}}\)
& \(c_{\epsilon}\)
& \(\displaystyle(s-1)+\frac{c_{\mathrm{sig}}}{c_{\epsilon}}\) \\
remainder \(\mathcal U_{\mathrm{rem}}\)
& \((s-1)(r-r_{\mathrm{cs}})\)
& \(c_{\epsilon}\)
& \(s-1\) \\
\bottomrule
\end{tabular}
\end{center}
For example, on a target-specific direction, none of the other classes
contributes a spike, so \(\bD_k\) has eigenvalue
\((s-1)c_{\epsilon}\), whereas \(\bDelta_k\) has eigenvalue
\(c_{\mathrm{sig}}+c_{\epsilon}\).  On a private direction of a
competing class, \(\bD_k\) contains one spike of size
\(c_{\mathrm{sig}}\), while \(\bDelta_k\) contains only the ridge
\(c_{\epsilon}\).  A shared direction
receives the same spike from every class and whitens to the neutral
level \(s-1\).

The three spectral levels satisfy
\[
    \frac{(s-1)c_{\epsilon}}
    {c_{\mathrm{sig}}+c_{\epsilon}}
    <
    s-1
    <
    (s-1)+\frac{c_{\mathrm{sig}}}{c_{\epsilon}}.
\]
The smallest \(r\) eigenvalues of \(\bT_k\) therefore consist of the
\(r_{\mathrm{cs}}\) target-specific eigenvalues and
\(r-r_{\mathrm{cs}}\) eigenvalues from the neutral level.  Their exact
average is
\begin{equation}
\label{app:eq:preprojection-average}
    \bar\gamma_{\mathrm{pre}}(r_{\mathrm{cs}})
    =
    \frac{r_{\mathrm{cs}}}{r}
    \frac{(s-1)c_{\epsilon}}
    {c_{\mathrm{sig}}+c_{\epsilon}}
    +
    \left(1-\frac{r_{\mathrm{cs}}}{r}\right)(s-1).
\end{equation}
Solving
\(\bar\gamma_{\mathrm{pre}}(r_{\mathrm{cs}})<1\) gives
\begin{equation}
\label{app:eq:class-specific-cutoff}
    \frac{r_{\mathrm{cs}}}{r}
    >
    \frac{
        (s-2)(c_{\mathrm{sig}}+c_{\epsilon})
    }{
        (s-1)c_{\mathrm{sig}}
    }.
\end{equation}
Thus low-rank structure by itself is not sufficient: a large enough
fraction of the \(r=d^\ast\) enhanced directions must be
class-specific.

\subsection{The fully shared endpoint}
\label{app:subsec:shared-endpoint}

When \(r_{\mathrm{cs}}=0\), all matrices \(\bH_k\) span the same
enhanced subspace and
all excess covariance matrices are identical.  Hence
\[
    \bD_k=(s-1)\bDelta_k,
    \qquad
    \bT_k=(s-1)\bI_d.
\]
Projection onto any orthogonal complement preserves the nonzero
eigenvalue \(s-1\), and therefore
\begin{equation}
\label{app:eq:shared-endpoint}
    \bar\gamma_k^\ast=s-1.
\end{equation}
The strict overlap condition fails for every \(s\ge2\), independently
of the mean vectors.

\subsection{Effect of removing the whitened means}
\label{app:subsec:mean-projection}

Let \(\bQ_k\) be an orthonormal basis for
\(\operatorname{col}(\bZ_k)\), and let \(\bQ_k^\perp\) be an
orthonormal basis for its orthogonal complement.  Put
\[
    q_k:=\operatorname{rank}(\bZ_k),
    \qquad
    d_k=d-q_k.
\]
Thus \(\bQ_k\) has \(q_k\) columns and \(\bQ_k^\perp\) has \(d_k\)
columns.  The nonzero eigenvalues of
\(\widetilde{\bD}_k\) are equivalently the eigenvalues of
\[
    (\bQ_k^\perp)^\top\bT_k\bQ_k^\perp.
\]
Write
\[
    \zeta_{k,1}\le\cdots\le\zeta_{k,d}
\]
for the eigenvalues of \(\bT_k\), and
\[
    \eta_{k,1}\le\cdots\le\eta_{k,d_k}
\]
for those of this compression.  Cauchy interlacing gives
\begin{equation}
\label{app:eq:projection-interlacing}
    \zeta_{k,\ell}
    \le
    \eta_{k,\ell}
    \le
    \zeta_{k,\ell+q_k}.
\end{equation}
In particular, projection cannot decrease the average of the smallest
\(r\) eigenvalues, and
\begin{equation}
\label{app:eq:projection-correction}
    0
    \le
    \bar\gamma_k^\ast
    -\bar\gamma_{\mathrm{pre}}(r_{\mathrm{cs}})
    \le
    \frac1r
    \sum_{\ell=1}^r
    \bigl(
        \zeta_{k,\ell+q_k}-\zeta_{k,\ell}
    \bigr).
\end{equation}
The endpoint \(r_{\mathrm{cs}}=0\) remains exact after projection
because all \(\zeta_{k,\ell}=s-1\).  For \(r_{\mathrm{cs}}>0\), the
finite-dimensional correction depends on the orientation of the
whitened-mean span relative to the three eigenspaces above.

\subsection{Numerical illustration}
\label{app:subsec:numerical-design}

We set
\[
    d=600,
    \qquad
    c_{\mathrm{sig}}=1,
    \qquad
    c_{\epsilon}=0.01,
    \qquad
    s\in\{3,5,10\},
\]
and consider
\[
    \frac{r_{\mathrm{cs}}}{r}
    \in
    \{0,0.25,0.50,0.60,0.70,0.75,0.80,0.85,0.90,0.95,1\}.
\]
For these constants, the right-hand side of
\eqref{app:eq:class-specific-cutoff} is \(0.5050\), \(0.7575\), and
approximately \(0.8978\) for \(s=3,5,10\), respectively.  When a grid
fraction does not make \(r_{\mathrm{cs}}\) an integer, the
implementation uses the nearest integer and records the resulting
value of \(r_{\mathrm{cs}}/r\).

For each \(s\), 30 isotropic Gaussian mean realizations are generated.
The same raw means are reused at every value of \(r_{\mathrm{cs}}\), so
comparisons along the path are paired.  The means serve only to quantify the projection correction
in \eqref{app:eq:projection-correction}.  Figure
\ref{app:fig:tail-overlap} plots the projected truncated-tail average
against \(r_{\mathrm{cs}}/r\), together with the exact pre-projection reference
\eqref{app:eq:preprojection-average}.  Table
\ref{app:tab:tail-overlap} compares the analytic cutoff with the
numerical transition and reports both endpoints.

This experiment illustrates the two endpoint regimes and the
transition between them; it does not assert that the spectral-overlap
condition in Theorem~\ref{prop2} holds generically.  Counts over the
30 mean realizations are finite-sample numerical summaries rather than
estimates of a universal probability.

\begin{figure}[t]
\centering
\includegraphics[width=\textwidth]{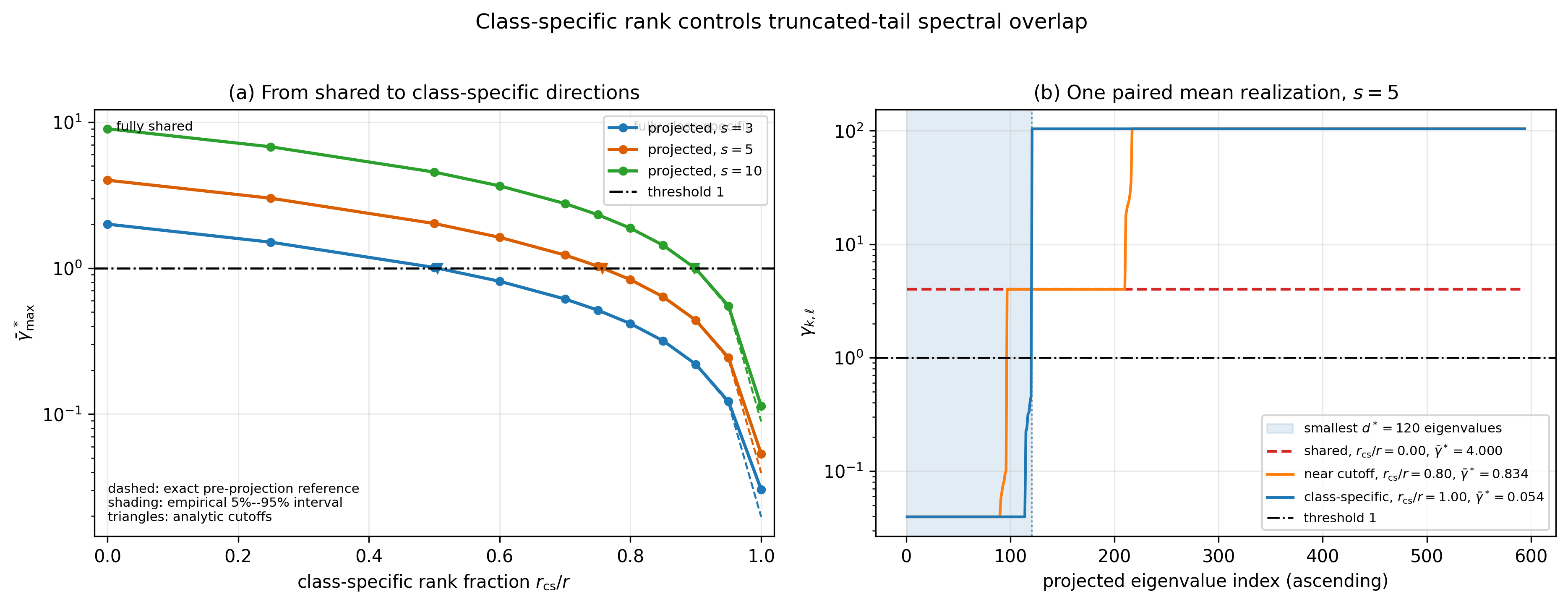}
\caption{Effect of the class-specific dimension on the truncated-tail
spectral-overlap condition.  Panel~(a) reports the median projected
\(\bar\gamma_{\max}^\ast\) over 30 paired mean realizations, with
empirical \(5\%\)--\(95\%\) bands, together with the exact
pre-projection reference and the analytic cutoff for each \(s\).
Panel~(b) shows the ordered projected spectra for one paired
realization at \(s=5\) under the fully shared
(\(r_{\mathrm{cs}}/r=0\)),
near-cutoff (\(r_{\mathrm{cs}}/r=0.8\)), and fully class-specific
(\(r_{\mathrm{cs}}/r=1\))
constructions.  The shaded region marks the smallest \(d^\ast=120\)
eigenvalues entering the truncated-tail average.}
\label{app:fig:tail-overlap}
\end{figure}

\begin{table}[t]
\centering
\small
\caption{Path from shared to class-specific enhanced directions at
\(d=600\), \(c_{\mathrm{sig}}=1\), and \(c_{\epsilon}=0.01\).
Endpoint entries report the
median and empirical \(5\%\)--\(95\%\) interval over 30 paired mean
realizations, with the number satisfying the strict condition on the
second line.
The numerical transition bracket ends at the first grid value from
which all larger grid values satisfy the condition.}
\label{app:tab:tail-overlap}
\begingroup
\setlength{\tabcolsep}{3.5pt}
\renewcommand{\arraystretch}{1.12}
\begin{tabular}{@{}ccccc@{}}
\toprule
\(s\)
& \shortstack{analytic cutoff\\for \(r_{\mathrm{cs}}/r\)}
& \shortstack{numerical\\bracket}
& \shortstack{shared\\(\(r_{\mathrm{cs}}=0\))}
& \shortstack{class-specific\\(\(r_{\mathrm{cs}}=r\))} \\
\midrule
3 & 0.505 & \((0.500,0.600]\)
& \shortstack{\(2.0000\ [2.0000,2.0000]\)\\\(0/30\)}
& \shortstack{\(0.0305\ [0.0301,0.0313]\)\\\(30/30\)} \\
5 & 0.758 & \((0.750,0.800]\)
& \shortstack{\(4.0000\ [4.0000,4.0000]\)\\\(0/30\)}
& \shortstack{\(0.0536\ [0.0530,0.0540]\)\\\(30/30\)} \\
10 & 0.898 & \((0.850,0.900]\)
& \shortstack{\(9.0000\ [9.0000,9.0000]\)\\\(0/30\)}
& \shortstack{\(0.1137\ [0.1130,0.1151]\)\\\(30/30\)} \\
\bottomrule
\end{tabular}
\endgroup
\end{table}

The numerical transition follows the analytic calculation closely.
For \(s=3\), the cutoff \(0.505\) lies between the grid points \(0.50\)
and \(0.60\): none of the 30 projected values satisfies the condition
at \(r_{\mathrm{cs}}/r=0.50\), whereas all 30 satisfy it at
\(r_{\mathrm{cs}}/r=0.60\) and every larger grid point.  Similarly,
the cutoff \(0.758\) for \(s=5\) lies
between \(0.75\) and \(0.80\).  The corresponding projected medians are
\(1.0312\) and \(0.8336\), so the two grid points lie on opposite sides
of the threshold.  For \(s=10\), the analytic cutoff \(0.8978\) lies
just below \(0.90\).  At that grid point the projected median is
\(0.9913\), with empirical \(5\%\)--\(95\%\) interval
\([0.9906,0.9921]\), and all 30 realizations remain below one.  At the
preceding grid point \(r_{\mathrm{cs}}/r=0.85\), the projected median
is \(1.4339\), and all 30 realizations violate the condition.

The projection correction has the sign predicted by
\eqref{app:eq:projection-correction}.  At the fully class-specific
endpoint, the pre-projection values from
\eqref{app:eq:preprojection-average} are \(0.0198\), \(0.0396\), and
\(0.0891\) for \(s=3,5,10\).  After removing the whitened-mean span,
the corresponding medians increase to \(0.0305\), \(0.0536\), and
\(0.1137\).  The correction is visible but does not alter the
qualitative conclusion at this endpoint.

At the fully shared endpoint, equation \eqref{app:eq:shared-endpoint}
applies without approximation after projection: the reported values
are exactly \(2\), \(4\), and \(9\) for \(s=3,5,10\).  Taken together,
the two endpoints and the intermediate path show that low rank alone
does not determine whether the spectral-overlap condition holds.  In
this construction, the decisive quantity is the fraction of enhanced
directions specific to the target class, and the required fraction
increases with the number of classes.
\section{Additional simulation results}\label{supp_sec:add_simu}

\subsection{Additional results under regression setting}
\label{supp_sec:add_simu_reg}

In this subsection, we present additional numerical results under the Gaussian regression setting. We report three sets of experiments below: (i) $n=18000$ and varying the matrix dimension; (ii) $n=9000$ and varying the correlation level, and (iii) $n=9000$ and varying the level of the distribution gap. All results are reported over 100 independent replications.

\textbf{\textsc{Varying sample size and dimension.}} Table~\ref{supp_tab:simu-reg} reports prediction and ROI localization results for $n=18000$ over different matrix dimensions under the single-signal setting $s=1$ with $\mu_1=\sigma_1^2=2$. As compared with the results in Table~\ref{tab:simu-reg} in the main text, a larger sample size leads to a clear reduction in prediction error and an improvement in ROI localization, as increasing $n$ will improve estimation accuracy of the attention and value parameters.  These results are consistent with our theoretical findings, and the remaining conclusions are similar to those in the main text.      

\begin{sidewaystable}
    \centering
    \caption{Performance of different methods across varying dimensions under the Gaussian regression setting with $n=18000$. The best and second-best results are marked by \textbf{bold} and \underline{underline}, respectively.}
    \label{supp_tab:simu-reg}
    \begin{tabular}{c|c|c|c|c|c}
    \hline
    Metric                          & Method           & $(p,d)=(100,70)$          & $(p,d)=(100,90)$          & $(p,d)=(120,70)$          & $(p,d)=(120,90)$          \\ \hline
    \multirow{7}{*}{RMSE}   & GroupLasso       & 15.085 (0.147)            & 15.804 (0.185)            & 15.123 (0.162)            & 15.790 (0.136)            \\
                                  & MDSP             & 9.750 (0.092)             & 9.277 (0.105)             & {\ul 8.957 (0.091)}       & 8.809 (0.075)             \\
                                  & Two-Step         & 15.202 (0.206)            & 15.872 (0.440)            & 15.323 (0.185)            & 15.926 (0.237)            \\
                                  & 1-head StdAtt    & 14.734 (0.140)            & 15.117 (0.167)            & 15.136 (0.157)            & 15.547 (0.165)            \\
                                  & $R$-head StdAtt  & 14.964 (0.155)            & 15.457 (0.178)            & 15.554 (0.163)            & 16.208 (0.184)            \\
                                  & 1-head DiagAtt   & \textbf{6.520 (0.185)}    & \textbf{5.101 (0.208)}    & \textbf{6.902 (0.168)}    & \textbf{5.689 (0.204)}    \\
                                  & $R$-head DiagAtt & {\ul 6.871 (0.187)}       & {\ul 5.203 (0.221)}       & {\ul 7.285 (0.179)}       & {\ul 5.844 (0.196)}       \\ \hline
    \multirow{7}{*}{CRLR} & GroupLasso       & 10.03\% (0.46\%)          & 9.98\% (0.47\%)           & 8.32\% (0.44\%)           & 8.38\% (0.42\%)           \\
                                  & MDSP             & 36.64\% (0.66\%)          & 40.11\% (0.77\%)          & 32.96\% (0.84\%)          & 33.17\% (0.61\%)          \\
                                  & Two-Step         & 24.47\% (13.06\%)         & 25.06\% (22.14\%)         & 14.34\% (9.46\%)          & 22.41\% (16.74\%)         \\
                                  & 1-head StdAtt    & 71.72\% (0.81\%)          & 80.98\% (0.70\%)          & 67.95\% (0.90\%)          & 76.37\% (0.77\%)          \\
                                  & $R$-head StdAtt  & 73.25\% (0.77\%)          & 81.74\% (0.73\%)          & 69.99\% (0.95\%)          & 78.78\% (0.75\%)          \\
                                  & 1-head DiagAtt   & {\ul 83.30\% (0.59\%)}    & {\ul 90.95\% (0.46\%)}    & {\ul 81.60\% (0.61\%)}    & {\ul 88.57\% (0.51\%)}    \\
                                  & $R$-head DiagAtt & \textbf{83.56\% (0.61\%)} & \textbf{91.28\% (0.47\%)} & \textbf{81.92\% (0.64\%)} & \textbf{88.97\% (0.49\%)} \\ \hline
    \end{tabular}
\end{sidewaystable}

\textbf{\textsc{Varying correlation level.}}
The model formulation in \eqref{eq:dist_assumption} separates active and inactive rows through row-wise distributions, but it does not require independence among rows within the same sample. We therefore examine whether DiagAtt remains effective when rows are correlated. Specifically, we generate rows of each $\bX^{(i)}$ so that, for $j\neq k$,
\(
\mathrm{Cov}\bigl(\bm{X}_{j,\cdot}^{(i)},\bm{X}_{k,\cdot}^{(i)}\bigr)
=
\rho
\bigl[\mathrm{Var}\bigl(\bm{X}_{j,\cdot}^{(i)}\bigr)\bigr]^{1/2}
\bigl[\mathrm{Var}\bigl(\bm{X}_{k,\cdot}^{(i)}\bigr)\bigr]^{1/2},
\)
where $\rho\in(0,1)$ controls the overall within-sample correlation strength. We vary $\rho \in \{0.3,0.6\}$ with $(n,p,d)=(9000,100,70)$ and report the results in Table~\ref{supp_tab:simu-correlation}.

Table~\ref{supp_tab:simu-correlation} shows that DiagAtt remains the best-performing method under correlated rows in terms of both prediction accuracy and ROI identification. Interestingly, as the correlation increases, the prediction error decreases for all the methods. Although this result may be surprising, similar observations have been reported in the high-dimensional linear regression literature
\citep{arxivsupp-reid2016study,arxivsupp-dalalyan2017prediction}.

\begin{table}[!htbp]
    \centering
    \caption{Performance of different methods across varying correlation levels under the Gaussian regression setting. The best and second-best results are marked by \textbf{bold} and \underline{underline}, respectively.}
    \label{supp_tab:simu-correlation}
    \resizebox{\linewidth}{!}{
        \begin{tabular}{c|cc|cc}
            \hline
            \multirow{2}{*}{Method} & \multicolumn{2}{c|}{$\rho=0.3$}                                         & \multicolumn{2}{c}{$\rho=0.6$}                                          \\ \cline{2-5} 
                                    & \multicolumn{1}{c|}{RMSE}                   & CRLR                      & \multicolumn{1}{c|}{RMSE}                   & CRLR                      \\ \hline
            GroupLasso              & \multicolumn{1}{c|}{12.710 (0.199)}         & 9.85\% (0.60\%)           & \multicolumn{1}{c|}{9.659 (0.153)}          & 9.83\% (0.60\%)           \\
            MDSP                    & \multicolumn{1}{c|}{8.702 (0.111)}          & 31.06\% (1.15\%)          & \multicolumn{1}{c|}{8.697 (0.124)}          & 31.29\% (1.07\%)          \\
            Two-Step                & \multicolumn{1}{c|}{13.367 (0.228)}         & 16.98\% (9.20\%)          & \multicolumn{1}{c|}{10.199 (0.175)}         & 13.96\% (8.75\%)          \\
            1-head StdAtt          & \multicolumn{1}{c|}{12.430 (0.195)}         & 55.73\% (1.73\%)          & \multicolumn{1}{c|}{8.942 (0.151)}          & 81.93\% (1.13\%)          \\
            $R$-head StdAtt          & \multicolumn{1}{c|}{12.375 (0.205)}         & 54.68\% (1.60\%)          & \multicolumn{1}{c|}{8.352 (0.149)}          & 82.68\% (0.96\%)          \\
            1-head DiagAtt         & \multicolumn{1}{c|}{\textbf{5.644 (0.218)}} & {\ul 82.63\% (0.95\%)}    & \multicolumn{1}{c|}{\textbf{3.086 (0.134)}} & {\ul 91.81\% (0.63\%)}    \\
            $R$-head DiagAtt         & \multicolumn{1}{c|}{{\ul 5.862 (0.228)}}    & \textbf{83.75\% (0.87\%)} & \multicolumn{1}{c|}{{\ul 3.151 (0.160)}}    & \textbf{92.27\% (0.64\%)} \\ \hline
        \end{tabular}
    }
\end{table}

\textbf{\textsc{Varying level of distribution gap.}} 
The theoretical analysis provides the intuition that a greater distribution discrepancy between active and background rows is expected to improve both ROI localization and prediction accuracy.  
To verify this intuition, we continue to consider the single-signal setting with $s=1$ and vary $(\mu_1,\sigma_1^2)\in\{1,3\}\times\{2, 2.5\}$ with $(n,p,d)=(9000,100,70)$. Table~\ref{supp_tab:distribution_gap} summarizes the performance of different methods across varying levels of distributional gap under the Gaussian regression setting.  

\begin{sidewaystable}
    \centering
    \caption{Performance of different methods across varying levels of distribution gap under the Gaussian regression setting. The best and second-best results are marked by \textbf{bold} and \underline{underline}, respectively.}
    \label{supp_tab:distribution_gap}
    \begin{tabular}{c|c|c|c|c|c}
    \hline
    Metric                          & Method           & $(\mu_1,\sigma_1^2)=(1,2)$ & $(\mu_1,\sigma_1^2)=(1,2.5)$ & $(\mu_1,\sigma_1^2)=(3,2)$ & $(\mu_1,\sigma_1^2)=(3,2.5)$ \\ \hline
    \multirow{7}{*}{RMSE}   & GroupLasso       & 15.079 (0.147)             & 16.834 (0.164)               & 15.097 (0.148)             & 16.850 (0.164)               \\
                                  & MDSP             & 9.573 (0.091)              & 10.275 (0.096)               & 10.031 (0.093)             & 10.672 (0.100)               \\
                                  & Two-Step         & 15.157 (0.217)             & 16.630 (0.423)               & 15.274 (0.190)             & 16.798 (0.378)               \\
                                  & 1-head StdAtt    & 14.963 (0.215)             & 15.827 (0.224)               & 14.687 (0.209)             & 15.216 (0.222)               \\
                                  & $R$-head StdAtt  & 15.508 (0.232)             & 15.750 (0.256)               & 14.631 (0.223)             & 14.487 (0.233)               \\
                                  & 1-head DiagAtt   & \textbf{7.753 (0.266)}     & \textbf{4.450 (0.262)}       & \textbf{5.883 (0.272)}     & \textbf{3.695 (0.307)}       \\
                                  & $R$-head DiagAtt & {\ul 8.366 (0.271)}        & {\ul 4.544 (0.283)}          & {\ul 6.102 (0.282)}        & {\ul 3.749 (0.302)}          \\ \hline
    \multirow{7}{*}{CRLR} & GroupLasso       & 10.03\% (0.47\%)           & 10.03\% (0.44\%)             & 10.03\% (0.45\%)           & 10.02\% (0.45\%)             \\
                                  & MDSP             & 35.57\% (0.68\%)           & 51.01\% (0.75\%)             & 38.50\% (0.80\%)           & 53.48\% (0.66\%)             \\
                                  & Two-Step         & 26.44\% (13.51\%)          & 33.56\% (13.16\%)            & 21.08\% (12.07\%)          & 28.59\% (12.96\%)            \\
                                  & 1-head StdAtt    & 66.65\% (1.06\%)           & 82.04\% (0.90\%)             & 75.17\% (1.10\%)           & 87.46\% (0.74\%)             \\
                                  & $R$-head StdAtt  & 68.11\% (1.13\%)           & 84.59\% (0.97\%)             & 76.91\% (0.99\%)           & 89.78\% (0.83\%)             \\
                                  & 1-head DiagAtt   & {\ul 76.51\% (1.05\%)}     & {\ul 93.15\% (0.58\%)}       & {\ul 86.60\% (0.82\%)}     & {\ul 95.43\% (0.50\%)}       \\
                                  & $R$-head DiagAtt & \textbf{76.88\% (0.93\%)}  & \textbf{93.43\% (0.59\%)}    & \textbf{87.05\% (0.74\%)}  & \textbf{95.64\% (0.50\%)}    \\ \hline
    \end{tabular}
\end{sidewaystable}

As expected, an increase in either $\mu_1$ or $\sigma_1^2$ enlarges the distributional discrepancy between active and background rows, thereby leading to clear improvements in both prediction accuracy and ROI localization for the single-head and multi-head DiagAtt. In contrast, the competitors exhibit negligible changes in both RMSE and CRLR. These observations further demonstrate the natural suitability of the proposed diagonalized attention mechanism for heterogeneous matrix-valued data under the separability assumption.

\subsection{Additional results under classification setting}\label{supp_sec:add_simu_cla}
This subsection provides supplementary numerical results under the Gaussian-mixture classification setting. Figure~\ref{fig:simu-mixture-supp} reports classification and ROI localization results of different methods for $n\in\{9000,18000\}$ and $(p,d)\in\{(100,90),(120,70),(120,90)\}$. Together with Figure~\ref{fig:simu-mixture-main} in the main text, the present results further confirm that the patterns associated with $(n,p,d)$ for the proposed diagonalized attention mechanism continue to hold under the Gaussian-mixture classification setting, extending beyond linear data-generating models.  

\begin{figure}[!htbp]
	\centering
	\includegraphics[width=1\textwidth]{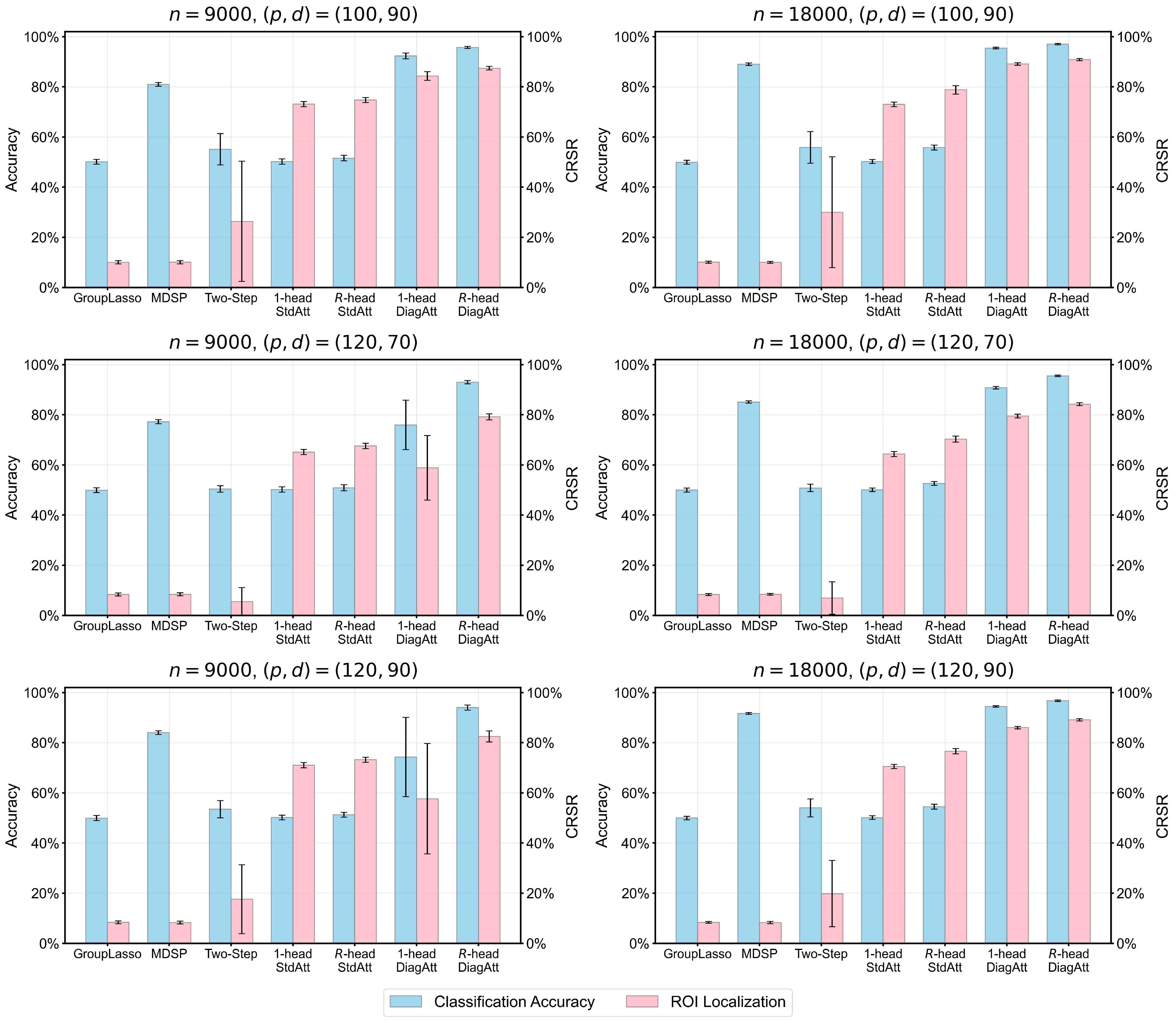}
	\caption{Performance of different methods across varying dimensions under the Gaussian-mixture classification setting with sample size $n=9000$. Each subplot displays classification accuracy (left axis, measured by Accuracy) and ROI localization (right axis, measured by CRLR), with bar heights and error bars representing the average and standard deviation of each metric, respectively.}
	\label{fig:simu-mixture-supp}
\end{figure}


\section{Additional empirical analysis}\label{supp_sec:add_emp}
In practice, ground-truth sentiment keywords are usually unavailable, which precludes a direct comparison of ROI localization across different methods for sentiment analysis. To further examine the ROI identification performance of the proposed diagonalized attention mechanism, we consider an additional sentiment analysis example from \citet{arxivsupp-supp_marion2025attention}, in which the sentiment-bearing keywords are known a priori.  

This dataset comprises a training set of 15,552 comments and a test set of 7,920 comments, both of which are balanced between positive and negative classes with a positive class proportion of 0.5. In particular, the sentiment label of each comment is exclusively determined by a pre-specified sentiment keyword. The test set is constructed using either sentiment keywords or sentence templates that differ from those in the training set, which ensures no overlap between the two sets and includes a substantial number of out-of-distribution (OOD) samples. Please refer to \citet{arxivsupp-supp_marion2025attention} for more details.

We adopt the same pretrained language model as in Section~\ref{section: real} of the main text to tokenize each comment into a sequence of tokens and represent each token as a feature vector of dimension $d=1024$. For each comment, instead of retaining the original token order, the corresponding $p^{(i)}$ feature vectors are randomly shuffled to form the embedding matrix $\bX^{(i)}\in\R^{p^{(i)}\times d}$, which further poses a greater challenge to all the methods. Additionally, a five-fold cross-validation technique is performed on the training set for hyperparameter tuning of all the methods. In this example, the competitors select up to \(\widetilde{R}=3\) tokens per comment, while DiagAtt identifies at most \(R=3\) tokens per comment.

Table~\ref{supp_tab:performance-synthetic} reports classification and ROI localization results of different methods. As discussed in Section~\ref{section: real}, DiagAtt again demonstrates superior performance in terms of both classification accuracy and ROI localization. In particular, the single-head DiagAtt performs better than its multi-head version, consistent with the fact that each comment contains only one sentiment-bearing keyword in this example. The standard attention mechanism also achieves comparable performance with a sufficiently large sample size, which supports the ability of attention-type scoring to accommodate individualized patterns in sentiment analysis. On the other hand, the inferior performance of MDSP further confirms that its underlying subgroup-based assumption is incompatible with heterogeneity inherent in natural languages. Two-Step achieves unremarkable classification accuracy and CRLR, while GroupLasso exhibits poor CRLR due to its shared coefficient estimates across all samples.   

\begin{table}[!htbp]
    \centering
    \caption{Performance of different methods on the sentiment analysis example from \citet{arxivsupp-supp_marion2025attention}. The best and second-best results are marked by \textbf{bold} and \underline{underline}, respectively.}
    \label{supp_tab:performance-synthetic}
    \begin{tabular}{c|c|c}
        \hline
        Method           & Accuracy     
        & CRLR             \\ \hline
        GroupLasso       & 85.72\%          & 41.34\%          \\
        MDSP             & 63.31\%          & 39.53\%          \\
        Two-Step         & 93.51\%          & 72.56\%          \\
        1-head StdAtt    & 98.12\%          & 86.21\%          \\
        $R$-head StdAtt  & 97.80\%          & 81.74\%          \\
        1-head DiagAtt   & \textbf{99.56\%} & \textbf{94.60\%} \\
        $R$-head DiagAtt & {\ul 99.33\%}    & {\ul 87.40\%}    \\ \hline
    \end{tabular}
\end{table}


\section{Additional implementation details}\label{supp_sec:detail}
This section provides further implementation details for the proposed diagonalized attention mechanism (DiagAtt) and the four competitors: group-lasso regression (GroupLasso), the multidirectional separation penalty (MDSP), the two-step clustering-based procedure (Two-Step), and the standard attention mechanism (StdAtt), with a focus on their model architectures, hyperparameter tuning procedures, and ROI localization strategies.  

\subsection{Model architectures and reproducibility}\label{supp_subsec:arch}
We begin with a detailed description of the model architecture of the proposed diagonalized attention mechanism used throughout all experiments.   

\textbf{\textsc{Diagonalized Attention Mechanism.}} We consider two implementations of DiagAtt, namely, the 1-head and $R$-head versions introduced in Section~\ref{sec:method_diag_attn}. Here, the number of attention heads $R$ is treated as a tunable hyperparameter, and its tuning procedure is deferred to Section~\ref{supp_subsec:tuning}. Additionally, we introduce an explicit temperature hyperparameter $\tau$ shared across all attention heads, yielding $\bkappa_r^{(i)}=\bkappa(\bX^{(i)}; \bW_r)=\softmax(\frac{1}{\tau}\cdot\bpsi(\bX^{(i)}; \bW_r))$ for $r\in[R]$. Throughout all numerical experiments, $\tau$ is fixed at $100$ for regression and $1000$ for classification.

Before describing the four competitors, we first introduce zero padding, a preprocessing technique used by all of them. Let $p=\max\{p^{(1)},\cdots,p^{(n)}\}$ denote the maximum number of rows among all covariate matrices $\bX^{(i)}\in\R^{p^{(i)}\times d}$. To ensure a shared row dimension, each $\bX^{(i)}$ is augmented with zero rows to obtain $\widetilde{\bX}^{(i)}\in\mathbb{R}^{p\times d}$. Clearly, such preprocessing inevitably introduces redundancy and additional computational cost, but is unavoidable for the competitors, which require inputs of a common matrix dimension. With this preprocessing in place, the model architectures of the competitors are presented below:

\textbf{\textsc{Group-Lasso Regression.}} The GroupLasso method adopts a generalized linear regression model with link function $g(\cdot)$, given by $g(\E(y^{(i)}))=\langle \widetilde{\bX}^{(i)}, \bC \rangle$. Here, $\bC\in\R^{p\times d}$ is an unknown coefficient matrix shared across the population. To encourage row-wise sparsity in the coefficient $\bC$, GroupLasso imposes a row-wise group-lasso penalty: $\calP_{\lambda}(\bC)=\lambda\sum_{j=1}^{p}\|\bC_{j,\cdot}\|_2$, which ignores heterogeneity in signal locations.

\textbf{\textsc{Multidirectional Separation Penalty.}} In contrast to GroupLasso, the MDSP method considers an individualized linear regression framework with link function $g(\cdot)$, given by $g(\E(y^{(i)}))=\langle\widetilde{\bX}^{(i)},\bC^{(i)}\rangle$, where $\bC^{(i)}\in\R^{p\times d}$ are sample-specific coefficients with shared matrix dimensions. To facilitate ROI localization, we introduce a generalized row-wise multidirectional separation penalty with two latent subgroups:  
\begin{equation}
	\label{supp_eq:mdsp_penalty}
	\calP_{\lambda}(\bC^{(1)},\cdots,\bC^{(n)},\bgamma_1,\cdots,\bgamma_p)=\lambda\sum_{i=1}^n\sum_{j=1}^p \min\left\{\|\bC_{j,\cdot}^{(i)}-\bgamma_j\|_1, \|\bC_{j,\cdot}^{(i)}\|_1\right\}.
\end{equation}
Here, $\{\bgamma_j\}_{j=1}^p$ are unknown but fixed parameter vectors. Although this method can exploit similarities among the sample-specific coefficients through the separation penalty, it imposes no explicit structures linking the sample-specific coefficients to the covariates. As a result, its individualized inference and prediction for a new sample require access to labeled data even at the testing stage, regardless of whether $\{\bgamma_j\}_{j=1}^p$ have been estimated.

\textbf{\textsc{Two-Step Clustering-Based Procedure.}} Given the data generating processes described in Section~\ref{section: simulation}, a seemingly straightforward approach is to adopt a two-step clustering-based procedure. Specifically, Two-Step treats all $\sum_{i=1}^n p^{(i)}$ rows across $n$ samples as independent observations and fits a two-component Gaussian mixture models (\citealp[GMM]{arxivsupp-mclachlan2000finite}). For each sample, it then identifies the target mixture component containing the majority of active rows and selects the top $\widetilde{R}$ representative rows as an estimate of $\calS^{(i)}$, with more details deferred to Section~\ref{supp_subsec:localization}. The resulting $\widetilde{R}$ rows, corresponding to the estimated active set $\widehat{\calS}^{(i)}$, are further concatenated into a feature vector of dimension $\widetilde{R} d$, which are subsequently used for downstream regression and classification tasks to mitigate noise. However, this naive two-step procedure faces two major limitations. First, it struggles with extreme class imbalance. Consider a simple setting where $s=1$ (i.e., there is only one active row per sample). Under the separability assumption~\eqref{eq:dist_assumption}, the total number of active rows across all samples is $n$, whereas the number of background rows is $-n+\sum_{i=1}^n p^{(i)}$. This creates a severe imbalance when $p^{(i)}$'s are large, which becomes even more problematic for clustering when there are multiple important rows. Second, the $p^{(i)}$ rows within a sample are often highly correlated—particularly in language or image data—whereas GMM assumes independence between observations. Violating this assumption significantly degrades GMM's performance. As detailed in Section~\ref{section: simulation}, our numerical studies empirically verify that this naive two-step method is ineffective.

\textbf{\textsc{Standard Attention Mechanism.}} To facilitate a direct comparison between the proposed diagonalized attention mechanism and the standard attention mechanism, we implement both 1-head and $R$-head versions of StdAtt, following the formulations in models~\eqref{single1} and \eqref{multi1} in Section~\ref{sec:method_attn}. In fact, StdAtt can be related to an individualized linear regression model, given by
\begin{equation}
    \label{supp_eq:stdatt}
    \begin{aligned}
        \mathcal{O}(\bX^{(i)};\{\bW_r,\bV_r\}_{r=1}^R)&=\sum_{r=1}^R\left\langle\softmax\left(\frac{\bX^{(i)}\bW_r(\bX^{(i)})^\top}{\sqrt{d}}\right)\bX^{(i)},\bV_r\right\rangle\\
        & = \sum_{r=1}^R\left\langle\bX^{(i)},\left\{\softmax\left(\frac{\bX^{(i)}\bW_r(\bX^{(i)})^\top}{\sqrt{d}}\right)\right\}^\top\bV_r\right\rangle\\
        & = \sum_{r=1}^R\left\langle\bX^{(i)},\bC_r^{(i)}\right\rangle = \left\langle\bX^{(i)},\sum_{r=1}^R\bC_r^{(i)}\right\rangle = \left\langle\bX^{(i)},\bC^{(i)}\right\rangle 
    \end{aligned}
\end{equation}
Here, $\bC^{(i)}=\sum_{r=1}^R \bC_r^{(i)}\coloneq\sum_{r=1}^R \left(\softmax_{row}\left(\frac{1}{\sqrt{d}}\widetilde{\bX}^{(i)}\bW_r(\widetilde{\bX}^{(i)})^{\top}\right)\right)^{\top}\bV_r$ are sample-specific coefficients, $\{\bW_r,\bV_r\}_{r=1}^R$ are unknown parameters, and $\softmax(\cdot)$ is applied row-wise. For fairness, we use the same number of attention heads for StdAtt and DiagAtt.  

All aforementioned models are implemented in \textit{Python} and run on a Linux-based computing cluster equipped with an Intel Xeon Platinum 8488C CPU and an NVIDIA RTX A6000 GPU. Specifically, GroupLasso is solved using the Python package \texttt{skglm}, while MDSP is optimized via a custom implementation of the proximal gradient algorithm. The Gaussian mixture model and the ridge regression in Two-Step are fitted using the Python package \texttt{sklearn}. Finally, DiagAtt and StdAtt are implemented in \textit{PyTorch} and trained using the Adam optimizer with its default settings.  

\subsection{Hyperparameter tuning procedures}\label{supp_subsec:tuning}
We tune the following hyperparameters in all numerical experiments: (1) the regularization parameter $\lambda$ in GroupLasso, MDSP, and the ridge regression used in Two-Step; (2) the number of attention heads $R$ in DiagAtt; and (3) the index of the Gaussian mixture component designated as the active component in Two-Step. Unless otherwise specified, the regularization parameter $\lambda$ and the number of attention heads $R$ are selected from $\{0.01,0.1,1\}$ and $\{1,2,3\}$, respectively.  

In the simulation study, we tune hyperparameters by randomly holding out $n/8$ training samples as a validation set. For each method, the hyperparameters are selected to minimize the validation loss, measured by the mean squared error for regression and the mean cross-entropy loss for classification. After hyperparameter selection, the training and validation samples are combined to refit the final model, which is evaluated on the test set.   

\subsection{ROI localization strategies}\label{supp_subsec:localization}
This subsection details ROI localization strategies for the proposed diagonalized attention mechanism and the four competitors described in Section~\ref{supp_subsec:arch}.
Unless otherwise noted, for all the competitors, we select $\widetilde{R}=10$ important rows per sample to serve as an estimate of $\calS^{(i)}$. 
By comparison, the number of attention heads in DiagAtt is restricted to $R\in\{1,2,3\}$, so that the number of its estimated active rows never exceeds that of the competitors. 
Intuitively, this imbalance in cardinality gives the competitors more opportunities to cover the correct active rows, which systematically favors them under the CRLR metric.  
The corresponding localization procedures are summarized below.    

\textbf{\textsc{Diagonalized Attention Mechanism.}} As discussed in the main text, the weight vector $\bkappa_r^{(i)}=\bkappa(\bX^{(i)};\bW_r)\in\R^{p^{(i)}}$ in the $r$-th attention head of DiagAtt is designed to locate the $r$-th signal row for sample $i$. As a result, the active rows $\calS^{(i)}$ for sample $i$ can be estimated as a collection of indices corresponding to the largest entry in each of the $R$ columns of the attention matrix, namely,
\begin{equation}
    \label{supp_eq:localization_diagatt}
    \hat{\calS}^{(i)} = \left\{(\hat{j}_1,\ldots, \hat{j}_R): \ \hat{j}_r=\argmax_{j\in\{1,\ldots,p^{(i)}\}}\hat{\bK}_{j,r}^{(i)},\quad r=1,\ldots,R,\quad  \hat{\bK}^{(i)}=\bK\left(\bX^{(i)};\{\hat{\bW}_r\}_{r=1}^R\right)\right\}.
\end{equation}
Here, $\{\hat{\bW}_r\}_{r=1}^R$ is obtained by solving the optimization problem~\eqref{obj} using a standard gradient-based algorithm.  

\textbf{\textsc{Group-Lasso Regression.}} Since GroupLasso imposes a group-lasso penalty that encourages row-wise sparsity, we recover the active rows $\calS^{(i)}$ by selecting the indices of the $\widetilde{R}$ rows with the largest row-wise $\ell_2$ norms of the estimated coefficient matrix $\widehat{\bC}$, namely,
\begin{equation}
    \label{supp_eq:localization_grouplasso}
    \hat{\calS}^{(i)}=\left\{(\hat{j}_1,\ldots,\hat{j}_{\widetilde{R}}):\ (\hat{j}_1,\ldots,\hat{j}_{\widetilde{R}})=\argmax_{(j_1,\ldots,j_{\widetilde{R}})\subset [p]}\sum_{r=1}^{\tilde{R}}\|\widehat{\bC}_{j_r,\cdot}\|_2\right\}.
\end{equation}
The resulting estimated active set $\hat{\calS}^{(i)}$ is identical across all samples, as $\widehat{\bC}$ is shared across the population.

\textbf{\textsc{Multidirectional Separation Penalty.}} The MDSP method uses an ROI localization strategy similar to \eqref{supp_eq:localization_grouplasso}, except that the population-level coefficient $\widehat{\bC}$ is replaced by the sample-specific coefficient $\widehat{\bC}^{(i)}$ for sample $i$. Specifically, the active rows for sample $i$ can be identified by
\begin{equation}
    \label{supp_eq:localization_mdsp}
    \hat{\calS}^{(i)}=\left\{(\hat{j}_1,\ldots,\hat{j}_{\widetilde{R}}):\ (\hat{j}_1,\ldots,\hat{j}_{\widetilde{R}})=\argmax_{(j_1,\ldots,j_{\widetilde{R}})\subset [p]}\sum_{r=1}^{\tilde{R}}\|\widehat{\bC}_{j_r,\cdot}^{(i)}\|_2\right\}.
\end{equation}
This individualized formulation allows the estimated active sets to vary across samples. 

\textbf{\textsc{Two-Step Clustering-Based Procedure.}} In Two-Step, once the active Gaussian mixture component has been determined, the fitted Gaussian mixture model is used to compute the posterior probabilities of membership in this component for the $p^{(i)}$ rows of $\bX^{(i)}$, denoted by $\calP^{(i)}=\{P_j^{(i)}\}_{j=1}^{p^{(i)}}$. We then estimate $\calS^{(i)}$ by selecting the indices corresponding to the $\widetilde{R}$ largest posterior probabilities in $\mathcal{P}^{(i)}$, namely,
\begin{equation}
    \label{supp_eq:localization_twostep}
    \hat{\calS}^{(i)}=\left\{(\hat{j}_1,\ldots,\hat{j}_{\widetilde{R}}):\ (\hat{j}_1,\ldots,\hat{j}_{\widetilde{R}})=\argmax_{(j_1,\ldots,j_{\widetilde{R}})\subset [p]}\sum_{r=1}^{\widetilde{R}} P_{j_r}^{(i)}\right\}.
\end{equation}

\textbf{\textsc{Standard Attention Mechanism.}} As shown in the formulation~\eqref{supp_eq:stdatt}, each summation term corresponds to a standard attention head, which can be rewritten as an individualized linear regression model. From the perspectives of \citet{arxivsupp-supp_vaswani2017attention} and \citet{arxivsupp-supp_yang2024attention}, $\bC_r^{(i)}=\left(\softmax_{row}\left(\frac{1}{\sqrt{d}}\widetilde{\bX}^{(i)}\bW_r(\widetilde{\bX}^{(i)})^{\top}\right)\right)^{\top}\bV_r$ can be interpreted as a sample-specific coefficient that captures the $r$-th heterogeneous regression effect for sample $i$. Motivated by this interpretation, we adopt a hybrid ROI localization strategy that combines the procedures used in MDSP and DiagAtt. 
Specifically, for sample $i$, we select the indices of the $\lfloor \widetilde{R}/R\rfloor$ rows with the largest row-wise $\ell_2$ norms of the $r$-th estimated regression effect $\widehat{\bC}_r^{(i)}$ to form $\widehat{\calS}_r^{(i)}$ as in \eqref{supp_eq:localization_mdsp}, for $r=1,\cdots,R$.  
The estimated active set is then defined as the union of these index sets across all $R$ attention heads, namely, $\widehat{\calS}^{(i)}\coloneq\bigcup_{r=1}^R \widehat{\calS}_r^{(i)}$. Note that this strategy yields an estimated active set with cardinality at most $\widetilde{R}$. As a consequence, the CRLR metric may inherently favor StdAtt over DiagAtt, whereas no analogous advantage arises for the other competitors.

\end{document}